\documentclass{article}
\usepackage{graphicx}
\usepackage[breaklinks=true,colorlinks,bookmarks=false,citecolor={cyan}]{hyperref}
\usepackage{amsmath,amssymb}
\usepackage{color}
\usepackage[skins,theorems]{tcolorbox}
\usepackage{authblk}
\usepackage[margin=1.2in]{geometry}
\usepackage{microtype}

\DeclareMathOperator{\Trace}{tr} 
\newcommand{\emc}[1]{{\textbf{\textit{\color[rgb]{0,.3,.6}#1}}}}
\newcommand{\qed}{\hfill \ensuremath{\Box}} 
\newtheorem{theorem}{Theorem}

\newtheorem{proposition}[theorem]{Proposition}
\newtheorem{theoremappendix}{Theorem}
\newtheorem{propositionappendix}[theoremappendix]{Proposition}
\newenvironment{proof}[1][Proof]{\begin{trivlist}
\item[\hskip \labelsep {\bfseries #1}]}{\end{trivlist}}

\definecolor{identifiercolor}{rgb}{.4,.6,.56}
\definecolor{stringcolor}{gray}{0.5}
\definecolor{inactivecolor}{rgb}{0.15,0.15,0.5}
\usepackage{listings}
\usepackage{setspace}

\begin{document}
\title{Self-Distillation Amplifies Regularization in Hilbert Space\footnote{This article is a more detailed version of a paper with the same title in \emc{Neural and Information Processing Systems (NeurIPS) 2020} conference.}}

\author{\vspace{0.2in}Hossein Mobahi$^\clubsuit$ \quad Mehrdad Farajtabar$^{\mathsection}$ \quad Peter L.~Bartlett$^{\clubsuit\,\ddag}$\\
\vspace{-0.1in}
{\color{magenta}\texttt{hmobahi@google.com} \quad \texttt{farajtabar@google.com} \quad \texttt{bartlett@eecs.berkeley.edu}} \\
\vspace{0.2in}
$^\clubsuit$Google Research, Mountain View, CA, USA\\
$^{\mathsection}$ DeepMind, Mountain View, CA, USA\\
$^\ddag$EECS Dept., University of California at Berkeley, Berkeley, CA, USA
}
\date{}
\maketitle

\begin{abstract}
Knowledge distillation introduced in the deep learning context is a method to transfer knowledge from one architecture to another. In particular, when the architectures are identical, this is called self-distillation. The idea is to feed in predictions of the trained model as new target values for retraining (and iterate this loop possibly a few times). It has been empirically observed that the self-distilled model often achieves higher accuracy on held out data. Why this happens, however, has been a mystery: the self-distillation dynamics does not receive any new information about the task and solely evolves by looping over training. To the best of our knowledge, there is no rigorous understanding of this phenomenon. This work provides the first theoretical analysis of self-distillation. We focus on fitting a nonlinear function to training data, where the model space is Hilbert space and fitting is subject to $\ell_2$ regularization in this function space. We show that self-distillation iterations modify regularization by progressively limiting the number of basis functions that can be used to represent the solution. This implies (as we also verify empirically) that while a few rounds of self-distillation may reduce over-fitting, further rounds may lead to under-fitting and thus worse performance.
\end{abstract}

\section{Introduction}

\paragraph{\bf Knowledge Distillation.} Knowledge distillation was introduced in the deep learning setting \cite{KnowledgeDistillation} as a method for transferring knowledge from one architecture (teacher) to another (student), with the student model often being smaller (see also \cite{compression} for earlier ideas). This is achieved by training the student model using the output probability distribution of the teacher model in addition to original labels. The student model benefits from this ``dark knowledge'' (extra information in soft predictions) and often performs better than if it was trained on the actual labels. 

\paragraph{\bf Extensions.} Various extensions of this approach have been recently proposed, where instead of output predictions, the student tries to match other statistics from the teacher model such as intermediate feature representations \cite{Romero2014FitNetsHF}, Jacobian matrices \cite{DBLP:journals/corr/abs-1803-00443}, distributions \cite{DBLP:journals/corr/HuangW17a}, Gram matrices \cite{Yim2017AGF}. Additional developments on knowledge distillation include its extensions to Bayesian settings \cite{NIPS2015_5965, vadera2020generalized}, uncertainty preservation \cite{tran2020hydra}, reinforcement learning \cite{hong2020collaborative, NIPS2017_7036, ghosh2018divideandconquer}, online distillation \cite{NIPS2018_7980}, zero-shot learning \cite{pmlr-v97-nayak19a}, multi-step knowledge distillation~\cite{takd}, tackling catastrophic forgetting \cite{Li2016LearningWF}, transfer of relational knowledge~\cite{Park2019RelationalKD}, adversarial distillation \cite{NIPS2018_7358}. Recently \cite{Phuong19} analyzed why the student model is able to mimic teacher model in knowledge distillation and \cite{Krishna2020} presented a statistical perspective on distillation.

\paragraph{\bf Self-Distillation.} The special case when the teacher and student architectures are identical is called\footnote{The term self-distillation has been used to refer a range of related ideas in the recent literature. We adopt the formulation of \cite{born-again-2018}, which is explained in our Section \ref{sec:self-distillation}. Self-distillation, which is defined in a supervised setting, can be considered an extension of self-training that is used unsupervised or semi-supervised learning.} \emc{self-distillation}. The idea is to feed in predictions of the trained model as new target values for retraining (and iterate this loop possibly a few times). It has been consistently observed that the self-distilled often achieves higher accuracy on held out data \cite{born-again-2018, YangXQY19, Ahn2019VariationalID}. Why this happens, however, has been a mystery: the self-distillation dynamics does not receive any new information about the task and solely evolves by looping over training. There have been some recent attempts to understand the mysteries around distillation. \cite{gotmare2018a} have empirically observed that the dark knowledge transferred by the teacher is localized mainly in higher layers and does not affect early (feature extraction) layers much. \cite{born-again-2018} interprets dark knowledge as importance weighting. \cite{Dong2019DistillationE} shows that early-stopping is crucial for reaching dark-knowledge of self-distillation. \cite{Abnar2020} empirically study how inductive biases are transferred through distillation. Ideas similar to self-distillation have been used in areas besides modern machine learning but with different names such diffusion and boosting in both the statistics and image processing communities \cite{Peyman13}.

\paragraph{\bf Contributions.} Despite interesting developments, why distillation can improve generalization remains elusive. To the best of our knowledge, there is no rigorous understanding of this phenomenon. This work provides a \emc{theoretical analysis of self-distillation}.  While originally observed in deep learning, we argue that this is a {\bf fundamental phenomenon} that can occur even in classical regression settings, where we fit a nonlinear function to training data with models belonging to a Hilbert space and using $\ell_2$ regularization in this function space. In this setting we show that the {\bf self-distillation iterations progressively limit the number of basis functions used to represent the solution}. This implies (as we also verify empirically) that while {\bf a few rounds of self-distillation may reduce over-fitting, further rounds may lead to under-fitting and thus worse performance}. To prove our results, we show that self-distillation leads to a non-conventional power iteration where the linear operation changes dynamically; each step depends intricately on the results of earlier linear operations via a nonlinear recurrence. While this recurrence has no closed form solution, we provide bounds that allow us to prove our sparsity guarantees. We also prove that using {\bf lower training error across distillation steps generally improves the sparsity effect}, and specifically we provide a closed form bound on the sparsity level as the training error goes to zero. Finally, we discuss how our regularization results can be translated into {\bf generalization bounds}.

\paragraph{\bf Organization.} In Section~\ref{sec:setup} we setup a variational formulation for nonlinear regression and discuss the existence of non-trivial solutions for it. Section~\ref{sec:self-distillation} presents our main technical results. It begins by formalizing the self-distillation process in our regression setup. It then shows that self-distillation iterations cannot continue indefinitely; at some point the solution collapses. After that, it provides a lower bound on the number of distillation iterations before the solution collapses. In addition, it shows that the basis functions initially used to represent the solution gradually change to a more sparse representation. It then demonstrates how operating in the near-interpolation regime throughout self-distillation ultimately helps with achieving higher sparsity. At the end of the section, we discuss how our regularization results can be translated into generalization bounds. In Section~\ref{sec:illustrative}, we walk through a toy example where we can express its solution as well as sparsity of its basis coefficients exactly and analytically over the course of self-distillation; empirically verifying the theoretical results. Section~\ref{sec:experiments} draws connections between our setting and the NTK regime of neural networks. This motivates subsequent experiments on deep neural networks in that section, where the observed behavior is consistent with the regularization viewpoint the paper provides.

\paragraph{\bf Supplemental.} To facilitate the presentation of analyses in Sections~\ref{sec:setup} and~\ref{sec:self-distillation}, we present our results in small steps as propositions and theorems. Their full {\bf proofs} are provided in the supplementary appendix. In addition, {\bf code} to generate the illustrative example of Sections~\ref{sec:illustrative} and~\ref{sec:experiments} are available in the supplementary material. 

\section{Problem Setup}
\label{sec:setup}

We first introduce some notation. We denote a set by $\mathcal{A}$, a matrix by $\boldsymbol{A}$, a vector by $\boldsymbol{a}$, and a scalar by $a$ or $A$. The $(i,j)$'th component of a matrix is denoted by $\boldsymbol{A}[i,j]$ and the $i$'th component of a vector by $\boldsymbol{a}[i]$. Also $\|\,.\,\|$ refers to $\ell_2$ norm of a vector. We use $\triangleq$ to indicate equal by definition. A linear operator $L$ applied to a function $f$ is shown by $[L f]$, and when evaluated at point $x$ by $[L f](x)$. For a positive definite matrix $\boldsymbol{A}$, we use $\kappa$ to refer to its condition number $\kappa \triangleq \frac{d_{\max}}{d_{\min}}$, where $d$'s are eigenvalues of $\boldsymbol{A}$.

Consider a finite training set $\mathcal{D} \triangleq \cup_{k=1}^K \{(\boldsymbol{x}_k, y_k)\}$, 
where $\boldsymbol{x}_k \in \mathcal{X} \subseteq \mathbb{R}^d$ and $y_k \in \mathcal{Y} \subseteq \mathbb{R}$. Consider a space of all admissible functions (as we define later in this section) $\mathcal{F}:\mathcal{X} \rightarrow \mathcal{Y}$. The goal is to use this training data to find a function $f^*: \mathcal{X} \rightarrow \mathcal{Y}$ that approximates the underlying mapping from $\mathcal{X}$ to $\mathcal{Y}$. We assume the function space $\mathcal{F}$ is rich enough to contain multiple functions that can fit finite training data. Thus, the presence of an adequate inductive bias is essential to guide the training process towards solutions that generalize. We infuse such bias in training via regularization. Specifically, we study regression problems of form\footnote{
Our choice of setting up learning as a constrained optimization rather than unconstrained form  $\frac{1}{K} \sum_k \big( f(\boldsymbol{x}_k) - y_k \big)^2 \,+\, c \, R(f)$ is motivated by the fact that we often have control over $\epsilon$ as a user-specified stopping criterion. In fact, in the era of overparameterized models, we can often fit training data to a desired $\epsilon$-optimal loss value \cite{zhang2016understanding}. However, if we adopt the unconstrained setting, it is unclear what value of $c$ would correspond to a particular stopping criterion.},
\begin{equation}
\label{eq:goal_regress}
f^* \triangleq \arg\min_{f \in \mathcal{F}} R(f) \quad \mbox{s.t.} \quad \frac{1}{K} \sum_k \big( f(\boldsymbol{x}_k) - y_k \big)^2 \leq \epsilon  \,,
\end{equation}
where $R: \mathcal{F} \rightarrow \mathbb{R}$ is a regularization functional, and $\epsilon > 0$ is a desired loss tolerance. We study regularizers with the following form,
\begin{equation}
\label{eq:regularizer}
R(f) = \int_{\mathcal{X}} \int_{\mathcal{X}} u(\boldsymbol{x},\boldsymbol{x}^\dag) f(\boldsymbol{x}) f(\boldsymbol{x}^\dag) \, d \boldsymbol{x} \, d \boldsymbol{x}^\dag\,,
\end{equation}
with $u$ being such that $\forall f \in \mathcal{F} \,;\,  R(f) \geq 0$ with {\it equality} only when $f(\boldsymbol{x})=0$. Without loss of generality\footnote{If $u$ is not symmetric, we define a new function $u^\diamond \triangleq \frac{1}{2} \big( u(\boldsymbol{x},\boldsymbol{x}^\dag)+u(\boldsymbol{x}^\dag,\boldsymbol{x}) \big)$ and work with that instead. Note that $u^\diamond$ is symmetric and satisfies $R_u(f)=R_{u^\diamond}(f)$.}, we assume $u$ is symmetric $u(\boldsymbol{x},\boldsymbol{x}^\dag)=u(\boldsymbol{x}^\dag,\boldsymbol{x})$. For a given $u$, the space $\mathcal{F}$ of admissible functions are $f$'s for which the double integral in (\ref{eq:regularizer}) is bounded.

The conditions we imposed on $R(f)$ implies that the operator $L$ defined as $[L f] \triangleq \int_{\mathcal{X}} u(\boldsymbol{x},\,.\,) f(\boldsymbol{x}) \, d \boldsymbol{x}$ has an empty null space\footnote{This a technical assumption for simplifying the exposition. If the null space is non-empty, one can still utilize it using \cite{Girosi95}.}. The symmetry and non-negativity conditions together are called \emc{Mercer's condition} and $u$ is called a kernel. Constructing $R$ via kernel $u$ can cover a wide range of regularization forms including\footnote{To see that, let's rewrite $\int_{\mathcal{X}} \sum_j w_j \big( P_j f(\boldsymbol{x}) \big)^2 \, d \boldsymbol{x}$ by a more compact form $\sum_j w_j \, \langle\, P_j f \,,\, P_j f \, \rangle$. Observe that $\sum_j w_j \, \langle\, P_j f \,,\, P_j f \, \rangle = \sum_j w_j \, \langle\, f \,,\, P_j^* P_j f \, \rangle =  \, \langle\, f \,,\, (\sum_j w_j P_j^* P_j) f \, \rangle = \, \langle\, f \,,\, U f \, \rangle$, where $P^*_j$ denotes the adjoint operator of $P_j$, and $U \triangleq \sum_j w_j P_j^* P_j$. Notice that $P_j^* P_j$ is a positive definite operator. Scaling it by the non-negative scalar $w_j$ still keeps the resulted operator positive definite. Finally, a sum of positive-definite operators is positive definite. Thus $U$ is a positive definite operator. Switching back to the integral notation, this gives exactly the requirement we had on choosing $u$,
\begin{equation}
\forall f \in \mathcal{F} \,;\, \int_{\mathcal{X}} u(\boldsymbol{x},\boldsymbol{x}^\dag) f(\boldsymbol{x}) f(\boldsymbol{x}^\dag) \, d \boldsymbol{x} \, d \boldsymbol{x}^\dag \geq 0 \,. \nonumber
\end{equation}
},
\begin{equation}
R(f) = \int_{\mathcal{X}} \sum_{j=1}^J w_j \big( [P_j f](\boldsymbol{x}) \big)^2 \, d \boldsymbol{x} \,,
\end{equation}
where $P_j$ is some linear operator (e.g. a differential operator to penalize non-smooth functions as we will see in Section \ref{sec:illustrative}), and $w_j \geq 0$ is some weight, for $j=1,\dots,J$ operators.

Plugging $R(f)$ into the objective functional leads to the following variational problem,
\begin{eqnarray}
\label{eq:var_objective}
f^* \triangleq \arg\min_{f \in \mathcal{F}} & \int_{\mathcal{X}} \int_{\mathcal{X}} u(\boldsymbol{x},\boldsymbol{x}^\dag) f(\boldsymbol{x}) f(\boldsymbol{x}^\dag)  d \boldsymbol{x}  d \boldsymbol{x}^\dag \nonumber \\
\label{eq:var_constraint}
& \mbox{s.t.} \frac{1}{K} \sum_k \big( f(\boldsymbol{x}_k) - y_k \big)^2 \leq \epsilon \,.
\end{eqnarray}
The Karush-Kuhn-Tucker (KKT) condition for this problem yields,
\begin{eqnarray}
\label{eq:var_problem_kkt_objective}
& & f^*_\lambda \triangleq \arg\min_{f \in \mathcal{F}} \, \frac{\lambda}{K} \sum_k \big( f(\boldsymbol{x}_k) - y_k \big)^2 \\
& & \,+\, \int_{\mathcal{X}} \int_{\mathcal{X}} u(\boldsymbol{x},\boldsymbol{x}^\dag) f(\boldsymbol{x}) f(\boldsymbol{x}^\dag) \, d \boldsymbol{x} \, d \boldsymbol{x}^\dag \\
\label{eq:var_problem_kkt_constraint}
\mbox{s.t.}& & \lambda \geq 0  \quad,\quad \frac{1}{K} \sum_k \big( f(\boldsymbol{x}_k) - y_k \big)^2 \leq \epsilon\\
& & \lambda \big(\frac{1}{K} \sum_k \big( f(\boldsymbol{x}_k) - y_k \big)^2-\epsilon \big)=0  \,.
\end{eqnarray}

\subsection{Existence of Non-Trivial Solutions}
\label{sec:non-trivial}
Stack training labels into a vector,
\begin{equation}
\label{eq:stacked_y}
\boldsymbol{y}_{\color{red}K \times 1} \triangleq [ \, y_1 \,|\, y_2 \,|\, \dots \,|\, y_K]\,.
\end{equation}
It is obvious that when $\frac{1}{K} \| \boldsymbol{y}\|^2 \leq \epsilon$, then $f^*$ has trivial solution $f^*(\boldsymbol{x})=0$, which we refer to this case as \emc{collapse} regime. In the sequel, we focus on the more interesting case of $\frac{1}{K} \| \boldsymbol{y}\|^2 > \epsilon$. It is also easy to verify that collapsed solution is tied to $\lambda=0$,
\begin{equation}
\label{eq:lambda_collapse}
\| \boldsymbol{y} \|^2 > K \, \epsilon \quad \Leftrightarrow \quad \lambda > 0 \,.
\end{equation}
Thus by taking any $\lambda>0$ that satisfies $\frac{1}{K} \sum_k \big( f^*_\lambda(\boldsymbol{x}_k) - y_k \big)^2-\epsilon =0 $, the following form $f^*_\lambda$ is an optimal solution to the problem (\ref{eq:var_objective}), i.e. $f^*=f^*_\lambda$.
\begin{eqnarray}
\label{eq:f_star_again}
f^*_\lambda&=& \arg\min_{f \in \mathcal{F}} \, \frac{\lambda}{K} \sum_k \big( f(\boldsymbol{x}_k) - y_k \big)^2 \\
& & \,+\, \int_{\mathcal{X}} \int_{\mathcal{X}} u(\boldsymbol{x},\boldsymbol{x}^\dag) f(\boldsymbol{x}) f(\boldsymbol{x}^\dag) \, d \boldsymbol{x} \, d \boldsymbol{x}^\dag \,.
\end{eqnarray}

\subsection{Closed Form of $f^*$}

In this section we present a closed form expression for (\ref{eq:f_star_again}). Since we are considering $\lambda>0$, without loss of generality, we can divide the objective function in (\ref{eq:f_star_again}) by $\lambda$ and let $c \triangleq 1/\lambda$; obviously $c>0$. 
\begin{eqnarray}
\label{eq:f_star_in_c}
f^* &=& \arg\min_{f \in \mathcal{F}} \, \frac{1}{K} \sum_k \big( f(\boldsymbol{x}_k) - y_k \big)^2 \nonumber \\
& &\,+\,c\, \int_{\mathcal{X}} \int_{\mathcal{X}} u(\boldsymbol{x},\boldsymbol{x}^\dag) f(\boldsymbol{x}) f(\boldsymbol{x}^\dag) \, d \boldsymbol{x} \, d \boldsymbol{x}^\dag \,.
\end{eqnarray}

The variational problem (\ref{eq:f_star_in_c}) has appeared in machine learning context extensively \cite{Girosi95}. It has a known solution form, due to representer theorem \cite{representer}, which we will present here in a proposition. However, we first need to introduce some definitions.  Let $g(\boldsymbol{x},\boldsymbol{t})$ be a function such that,
\begin{equation}
\int_{\mathcal{X}} u(\boldsymbol{x},\boldsymbol{x}^\dag) \, g(\boldsymbol{x}^\dag,\boldsymbol{t}) \, d \boldsymbol{x}^\dag = \delta(\boldsymbol{x} - \boldsymbol{t}) \,,
\end{equation}
where $\delta(\boldsymbol{x})$ is Dirac delta. Such $g$ is called the \emc{Green's function}\footnote{We assume that the Green's function exists and is continuous. Detailed treatment of existence conditions is beyond the scope of this work and can be found in text books such as \cite{duffy2001green}.} of the linear operator $L$, with $L$ being $[L f](\boldsymbol{x}) \triangleq \int_{\mathcal{X}} u(\boldsymbol{x},\boldsymbol{x}^\dag) \, f(\boldsymbol{x}^\dag) \, d \boldsymbol{x}^\dag$. Let the matrix $\boldsymbol{G}_{\color{red}K \times K}$ and the vector ${\boldsymbol{g}_{\boldsymbol{x}}}_{\color{red}K \times 1}$ be defined as,
\begin{eqnarray}
\boldsymbol{G}[j,k] &\triangleq& \frac{1}{K} \, g(\boldsymbol{x}_j, \boldsymbol{x}_k) \\
\boldsymbol{g}_{\boldsymbol{x}} [k] &\triangleq& \frac{1}{K} \, g (\boldsymbol{x}, \boldsymbol{x}_k) \,.
\end{eqnarray}

\begin{tcolorbox}[colback=green!8!white,colframe=black]
\begin{proposition}
\label{prop:f_star}
The variational problem (\ref{eq:f_star_in_c}) has a solution of the form,
\begin{equation}
\label{eq:sol_f_star}
f^*(\boldsymbol{x}) \,=\, \boldsymbol{g}_{\boldsymbol{x}}^T (c \boldsymbol{I} + \boldsymbol{G})^{-1} \boldsymbol{y} \,.
\end{equation}
\end{proposition}
\end{tcolorbox}

Notice that the matrix $\boldsymbol{G}$ is \emc{positive definite}\footnote{This property of $\boldsymbol{G}$ comes from the fact that $u$ is a positive definite kernel (definite instead of semi-definite, due to empty null space assumption on the operator $L$), thus so is its inverse (i.e. $g$). Since $g$ is a kernel, its associated Gram matrix is positive definite.}. Since by definition $c > 0$, the inverse of the matrix $c \boldsymbol{I} + \boldsymbol{G}$ is well-defined. Also, because $\boldsymbol{G}$ is positive definite, it can be diagonalized as,
\begin{equation}
\boldsymbol{G} = \boldsymbol{V}^T \boldsymbol{D} \boldsymbol{V} \,,
\end{equation} 
where the diagonal matrix $\boldsymbol{D}$ contains the eigenvalues of $\boldsymbol{G}$, denoted as $d_1,\dots,d_K$ that are strictly greater than zero, and the matrix $\boldsymbol{V}$ contains the corresponding eigenvectors.

\subsection{Bounds on Multiplier $c$}

Earlier we showed that any $c>0$ that is a root of $\frac{1}{K} \sum_k \big( f^*_c(\boldsymbol{x}_k) - y_k \big)^2-\epsilon =0$ produces an optimal solution $f^*$ via (\ref{eq:f_star_in_c}). However, in the settings that we are interested in, we do not know the underlying $c$ or $f^*$ beforehand; we have to relate them to the given training data instead. As we will see later in Proposition~\ref{prop:lower_bound_ut}, knowledge of $c$ allows us to resolve the recurrence on $\boldsymbol{y}$ created by self-distillation loop and obtain an explicit bound $\|\boldsymbol{y}\|$ at each distillation round. Unfortunately finding $c$ in closed form by seeking roots of $\frac{1}{K} \sum_k \big( f^*_c(\boldsymbol{x}_k) - y_k \big)^2-\epsilon =0$ w.r.t. $c$ is impossible, due to the nonlinear dependency of $f$ on $c$ caused by matrix inversion; see (\ref{eq:sol_f_star}). However, we can still provide bounds on the value of $c$ as shown in this section.

Throughout the analysis, it is sometimes easier to work with rotated labels $\boldsymbol{V} \boldsymbol{y}$. Thus we define,
\begin{equation}
\boldsymbol{z} \triangleq \boldsymbol{V} \boldsymbol{y} \,.
\end{equation}
Note that any result on $\boldsymbol{z}$ can be easily converted back via $\boldsymbol{y} = \boldsymbol{V}^T \boldsymbol{z}$, as $\boldsymbol{V}$ is an orthogonal matrix. Trivially $\|\boldsymbol{z}\|=\|\boldsymbol{y}\|$.
Our first step is to present a simplified form for the error term $\frac{1}{K} \sum_k \big( f^*(\boldsymbol{x}_k) - y_k \big)^2$ using the following proposition.

\begin{tcolorbox}[colback=green!8!white,colframe=black]
\begin{proposition}
\label{prop:simple_error}
The following identity holds,
\begin{equation}
\frac{1}{K} \sum_k \big( f^*(\boldsymbol{x}_k) - y_k \big)^2 \,=\, \frac{1}{K} \, \sum_k (z_k \, \frac{c}{c+d_k} )^2 \,.
\end{equation}
\end{proposition}
\end{tcolorbox}

We now proceed to bound the roots of $h(c) \triangleq \frac{1}{K} \, \sum_k (z_k \, \frac{c}{c+d_k} )^2 - \epsilon$. Since we are considering $\| \boldsymbol{y} \|^2 > K \, \epsilon$, and thus by (\ref{eq:lambda_collapse}) $c>0$, it is easy to construct the following lower and upper bounds on $h$,
\begin{eqnarray}
\underline{h}(c) &\triangleq& \frac{1}{K} \, \sum_k (z_k \, \frac{c}{c+d_{\max}} )^2 - \epsilon \\
\overline{h}(c) &\triangleq& \frac{1}{K} \, \sum_k (z_k \, \frac{c}{c+d_{\min}} )^2 - \epsilon \,.
\end{eqnarray}
The roots of $\underline{h}$ and $\overline{h}$, namely $c_1$ and $c_2$, can be easily derived,
\begin{equation}
c_1 = \frac{ d_{\max} \sqrt{K \, \epsilon}}{\| \boldsymbol{z} \| - \sqrt{K \, \epsilon}} \quad , \quad c_2 = \frac{ d_{\min} \sqrt{K \, \epsilon}}{\| \boldsymbol{z} \| - \sqrt{K \, \epsilon}} \,.
\end{equation}
Since $\underline{h}(c) \leq h(c)$, the condition $\underline{h}(c_1) = 0$ implies that $h(c_1) \geq 0$. Similarly, since $h(x) \leq \overline{h}(c)$, the condition $\overline{h}(c_2) = 0$ implies that $h(c_2) \leq 0$. By the intermediate value theorem, due to continuity of $f$ and the fact that $\|\boldsymbol{z}\|=\|\boldsymbol{y}\|> \sqrt{K \, \epsilon}$ (non-collapse condition), there is a point $c$ between $c_1$ and $c_2$ at which $h(c)=0$,
\begin{equation}
\label{eq:root_bounds}
\frac{ d_{\min} \sqrt{K \, \epsilon}}{\| \boldsymbol{z} \| - \sqrt{K \, \epsilon}}\leq c \leq \frac{ d_{\max} \sqrt{K \, \epsilon}}{\| \boldsymbol{z} \| - \sqrt{K \, \epsilon}}.
\end{equation}

\section{Self-Distillation}
\label{sec:self-distillation}

Denote the prediction vector over the training data $\boldsymbol{x}_1,\dots \boldsymbol{x}_K$ as,
\begin{eqnarray}
{\boldsymbol{f}}_{\color{red} K \times 1} &\triangleq& \big[ f^*(\boldsymbol{x}_1) \,|\, \dots \,|\, f^*(\boldsymbol{x}_K) \big]^T \\
&=& \boldsymbol{V}^T \boldsymbol{D} (c \boldsymbol{I} + \boldsymbol{D})^{-1} \boldsymbol{V} \boldsymbol{y} \,.
\end{eqnarray}

\begin{figure*}[t]
\centering
\includegraphics[width=0.9\textwidth]{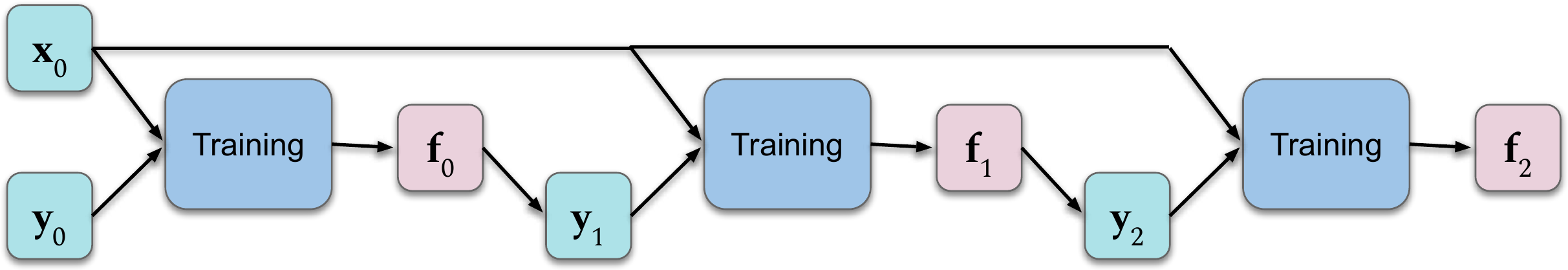}
\caption{Schematic illustration of the self-distillation process for two iterations.}
\label{fig:schematic}
\end{figure*}

Self-distillation treats this prediction as target labels for a new round of training, and repeats this process as shown in Figure~\ref{fig:schematic}. With abuse of notation, denote the $t$'th round of distillation by subscript $t$. We refer to the standard training (no self-distillation yet) by the step $t=0$. Thus the standard training step has the form,
\begin{equation}
{\boldsymbol{f}_0} = \boldsymbol{V}^T \boldsymbol{D} (c_0 \boldsymbol{I} + \boldsymbol{D})^{-1} \boldsymbol{V} \boldsymbol{y}_0 \,,
\end{equation}
where $\boldsymbol{y}_0$ is the vector of ground truth labels as defined in (\ref{eq:stacked_y}). Letting $\boldsymbol{y}_1 \triangleq \boldsymbol{f}_0$, we obtain the next model by applying the first round of self-distillation, whose prediction vector has the form,
\begin{equation}
{\boldsymbol{f}_1} = \boldsymbol{V}^T \boldsymbol{D} (c_1 \boldsymbol{I} + \boldsymbol{D})^{-1} \boldsymbol{V} \boldsymbol{y}_1 \,,
\end{equation}
In general, for any $t \geq 1$ we have the following recurrence,
\begin{equation}
\label{eq:rec_y}
\boldsymbol{y}_t = \boldsymbol{V}^T \boldsymbol{A}_{t-1} \boldsymbol{V} \boldsymbol{y}_{t-1} \,,
\end{equation}
where we define for any $t \geq 0$,
\begin{equation}
\label{eq:def_A}
{\boldsymbol{A}_t}_{\color{red}K \times K} \triangleq \boldsymbol{D} (c_t \boldsymbol{I} + \boldsymbol{D})^{-1} \,.
\end{equation}
Note that $\boldsymbol{A}_t$ is also a diagonal matrix. Furthermore, since at the $t$'th distillation step, everything is the same as the initial step except the training labels, we can use Proposition~\ref{prop:f_star} to express $f_t(\boldsymbol{x})$ as,
\begin{equation}
\label{eq:sol_f_star_t}
f^*_t(\boldsymbol{x}) \,=\, \boldsymbol{g}_{\boldsymbol{x}}^T (c_t \boldsymbol{I} + \boldsymbol{G})^{-1} \boldsymbol{y}_t =  \boldsymbol{g}_{\boldsymbol{x}}^T \boldsymbol{V}^T \boldsymbol{D}^{-1} ( \Pi_{i=0}^t \boldsymbol{A}_t ) \boldsymbol{V} \boldsymbol{y}_0 \,.
\end{equation}
Observe that the only dynamic component in the expression of $f_t^*$ is $\Pi_{i=0}^t \boldsymbol{A}_i$. 
In the following, we show how $\Pi_{i=0}^t \boldsymbol{A}_i$ evolves over time. Specifically, we show that self-distillation progressively sparsifies\footnote{Here \emc{sparsity} is in a \emc{relative} sense, meaning some elements being so smaller than others that they could be considered as negligible.} the matrix $\Pi_{i=0}^t \boldsymbol{A}_i$ at a provided rate.

Also recall from Section \ref{sec:non-trivial} that {\em in each step of self-distillation} we require $\|\boldsymbol{y}_t\| > \sqrt{ K \, \epsilon}$. If this condition breaks at any point, the solution {\em collapses} to the zero function, and subsequent rounds of self-distillation keep producing $f^*(\boldsymbol{x})=0$. In this section we present a lower bound on number of iterations $t$ that guarantees all intermediate problems satisfy $\|\boldsymbol{y}_t\| > \sqrt{ K \, \epsilon}$.
\subsection{Unfolding the Recurrence}
\label{sec:unfold_recur}

Our goal here is to understand how $\|\boldsymbol{y}_t\|$ evolves in $t$. 
By combining (\ref{eq:rec_y}) and (\ref{eq:def_A}) we obtain,
\begin{equation}
\boldsymbol{y}_t = \boldsymbol{V}^T \boldsymbol{D} (c_{t-1} \boldsymbol{I} + \boldsymbol{D})^{-1} \boldsymbol{V} \boldsymbol{y}_{t-1} \,.
\end{equation}
By multiplying both sides from the left by $\boldsymbol{V}$ we get,
\begin{eqnarray}
& & \boldsymbol{V} \boldsymbol{y}_t = \boldsymbol{V} \boldsymbol{V}^T \boldsymbol{D} (c_{t-1} \boldsymbol{I} + \boldsymbol{D})^{-1} \boldsymbol{V} \boldsymbol{y}_{t-1} \\
\label{eq:rec_z_c}
&\equiv& \boldsymbol{z}_t = \boldsymbol{D} (c_{t-1} \boldsymbol{I} + \boldsymbol{D})^{-1}  \boldsymbol{z}_{t-1} \\
\label{eq:rec_z_c_2}
&\equiv& \frac{1}{\sqrt{K \, \epsilon}}\boldsymbol{z}_t = \boldsymbol{D} (c_{t-1} \boldsymbol{I} + \boldsymbol{D})^{-1}  \, \frac{1}{\sqrt{K \, \epsilon}} \boldsymbol{z}_{t-1} \,.
\end{eqnarray}
Also we can use the bounds on $c$ from (\ref{eq:root_bounds}) at any arbitrary $t \geq 0$ and thus write,
\begin{equation}
\label{eq:root_bounds_t}
\forall \, t \geq 0 \,;\, \| \boldsymbol{z}_t\| > \sqrt{K \, \epsilon} \,\Rightarrow\, \frac{ d_{\min} \sqrt{K \, \epsilon}}{\| \boldsymbol{z}_t \| - \sqrt{K \, \epsilon}}\leq c_t \leq \frac{ d_{\max} \sqrt{K \, \epsilon}}{\| \boldsymbol{z}_t \| - \sqrt{K \, \epsilon}}
\end{equation}
By combining (\ref{eq:rec_z_c_2}) and (\ref{eq:root_bounds_t}) we obtain a recurrence solely in $\boldsymbol{z}$,
\begin{equation}
\label{eq:z_t_recur}
\boldsymbol{z}_t = \boldsymbol{D} (\frac{ \alpha_t \sqrt{K \, \epsilon}}{\| \boldsymbol{z}_{t-1} \| - \sqrt{K \, \epsilon}} \boldsymbol{I} + \boldsymbol{D})^{-1} \, \boldsymbol{z}_{t-1} \,,
\end{equation}
where,
\begin{equation}
\label{eq:range_alpha}
d_{\min} \leq \alpha_t \leq d_{\max} \,.
\end{equation}

We now establish a lower bound on the value of $\|\boldsymbol{z}_t\|$.

\begin{tcolorbox}[colback=green!8!white,colframe=black]
\begin{proposition}
\label{prop:lower_bound_ut}
For any $t \geq0$, if $\| \boldsymbol{z}_i\| > \sqrt{K \, \epsilon}$ for $i=0,\dots,t$, then,
\begin{equation}
\| \boldsymbol{z}_t\| \geq a^t(\kappa) \| \boldsymbol{z}_0\| - \sqrt{K \, \epsilon} \, b(\kappa) \frac{a^t(\kappa) - 1}{a(\kappa) - 1} \,,
\end{equation}
where,
\begin{eqnarray}
a(x) &\triangleq& \frac{(r_0 - 1)^2 + x (2 r_0 - 1) }{(r_0 -1 + x)^2} \\
b(x) &\triangleq& \frac{r_0^2 x}{(r_0 -1 + x)^2 } \\
r_0 &\triangleq& \frac{1}{\sqrt{K \, \epsilon}} \,\|\boldsymbol{z}_0\| \quad,\quad
\kappa \triangleq \frac{d_{\max}}{d_{\min}} \,.
\end{eqnarray}
\end{proposition}
\end{tcolorbox}

\subsection{Guaranteed Number of Self-Distillation Rounds}
\label{sec:guaranteed_iters}

By looking at (\ref{eq:rec_z_c}) it is not hard to see the value of $\| \boldsymbol{z}_t\|$ is \emc{decreasing} in $t$. That is because $c_t$\footnote{$c_t>0$ following from the assumption $\|\boldsymbol{z}_t\|> \sqrt{K \, \epsilon}$ and (\ref{eq:lambda_collapse}).} as well as elements of the diagonal matrix $\boldsymbol{D}$ are strictly positive. Hence $\boldsymbol{D}(c_{t-1} \boldsymbol{I} + \boldsymbol{D})^{-1}$ shrinks the magnitude of $\boldsymbol{z}_{t-1}$ in each iteration.

Thus, starting from $\|\boldsymbol{z}_0\| > \sqrt{K \, \epsilon}$, as $\|\boldsymbol{z}_t\|$ decreases, at some point it falls below the value $\sqrt{K \, \epsilon}$ and thus the solution collapses. We now want to find out after how many rounds $t$, the solution collapse happens. Finding the exact such $t$, is difficult, but by setting a lower bound of $\|\boldsymbol{z}_t\|$ to $\sqrt{ K \, \epsilon}$ and solving that in $t$, calling the solution  $\underline{t}$, we can guarantee realization of at least $\underline{t}$ rounds where the value of $\|\boldsymbol{z}_{\underline{t}}\|$ remains above $\sqrt{ K \, \epsilon}$. 

We can use the lower bound we developed in Proposition~\ref{prop:lower_bound_ut} in order to find such $\underline{t}$. This is shown in the following proposition.

\begin{tcolorbox}[colback=green!8!white,colframe=black]
\begin{proposition}
\label{prop:under_t}
Starting from $\|\boldsymbol{y}_0\| > \sqrt{K \, \epsilon}$, meaningful (non-collapsing solution) self-distillation is possible at least for $\underline{t}$ rounds,
\begin{equation}
\label{eq:t_star_final}
\underline{t} \triangleq \frac{\frac{\| \boldsymbol{y}_0\|}{\sqrt{K \, \epsilon}} -1}{\kappa} \,.
\end{equation}
\end{proposition}
\end{tcolorbox}

Note that when we are in near-interpolation regime, i.e. $\epsilon \rightarrow 0$, the form of $\underline{t}$ simplifies: $\underline{t} \approx \frac{\| \boldsymbol{y}_0\|}{\kappa \, \sqrt{K \, \epsilon}} $.

\subsection{Evolution of Basis}
\label{sec:evolution_basis}
Recall from (\ref{eq:sol_f_star_t}) that the learned function after $t$ rounds of self-distillation has the form,
\begin{equation}
f^*_t(\boldsymbol{x}) \,=\,  \boldsymbol{g}_{\boldsymbol{x}}^T \boldsymbol{V}^T \boldsymbol{D}^{-1} ( \Pi_{i=0}^t \boldsymbol{A}_t ) \boldsymbol{V} \boldsymbol{y}_0 \,.
\end{equation}
The only time-dependent part is thus the following \emc{diagonal} matrix,
\begin{equation}
\label{eq:def_B}
\boldsymbol{B}_t \triangleq \Pi_{i=0}^t \boldsymbol{A}_t \,.
\end{equation}
In this section we show how $\boldsymbol{B}_t$ evolves over time. Specifically, we claim that the matrix $\boldsymbol{B}_t$ becomes progressively sparser as $t$ increases. 

\begin{tcolorbox}[colback=green!8!white,colframe=black]
\begin{theorem}
\label{prop:expansion_rate}
Suppose $\|\boldsymbol{y}_0\| > \sqrt{K \, \epsilon}$ and $t \leq \frac{\| \boldsymbol{y}_0\|}{\kappa \, \sqrt{K \, \epsilon}} - \frac{1}{\kappa}$. Then for any pair of diagonals of $\boldsymbol{D}$, namely $d_j$ and $d_k$, with the condition that $d_k > d_j$, the following inequality holds.
\begin{equation}
\label{eq:ratio_expansion}
\frac{\boldsymbol{B}_{t-1} [k,k]}{\boldsymbol{B}_{t-1} [j,j]} \geq \left( \frac{\frac{\| \boldsymbol{y}_0 \|}{\sqrt{K \, \epsilon}} - 1 +  \frac{d_{\min}}{d_j} }{\frac{\| \boldsymbol{y}_0 \|}{\sqrt{K \, \epsilon}} - 1 +  \frac{d_{\min}}{d_k} }\right)^{t} \,.
\end{equation}
\end{theorem}
\end{tcolorbox}

The above theorem suggests that, as $t$ increases, the smaller elements of $\boldsymbol{B}_{t-1}$ shrink faster and at some point become negligible compared to larger ones. That is because in (\ref{eq:ratio_expansion}) we have assumed $d_k > d_j$, and thus the r.h.s. expression in the parentheses is strictly greater than $1$. The latter implies that $\frac{\boldsymbol{B}_{t-1} [k,k]}{\boldsymbol{B}_{t-1} [j,j]}$ is increasing in $t$.

Observe that if one was able to push $t \rightarrow \infty$, then only one entry of $\boldsymbol{B}_t$  (the one corresponding to $d_{\max}$) would remain significant relative to others. Thus, self-distillation process \emc{progressively sparsifies} $\boldsymbol{B}_t$. This sparsification affects the expressiveness of the regression solution $f^*_t(\boldsymbol{x})$. To see that, use the definition of $f^*_t(\boldsymbol{x})$ from (\ref{eq:sol_f_star_t}) to express it in the following form,
\begin{equation}
\label{eq:B_as_basis_coeffs}
f^*_t(\boldsymbol{x}) \,=\, \boldsymbol{g}_{\boldsymbol{x}}^T \boldsymbol{V}^T \boldsymbol{D}^{-1} \, \boldsymbol{B}_t \boldsymbol{V} \boldsymbol{y}_0 = \boldsymbol{p}^T_{\boldsymbol{x}} \boldsymbol{B}_t \boldsymbol{z}_0\,.
\end{equation}
where we gave a name to the rotated and scaled basis $\boldsymbol{p}_{\boldsymbol{x}} \triangleq \boldsymbol{D}^{-1} \boldsymbol{V} \boldsymbol{g}_{\boldsymbol{x}} $ and rotated vector $\boldsymbol{z}_0 \triangleq \boldsymbol{V} \boldsymbol{y}_0$.
The solution $f^*_t$ is essentially represented by a weighted sum of the basis functions (the components of $\boldsymbol{p}_{\boldsymbol{x}}$). Thus, the number of significant diagonal entries of $\boldsymbol{B}_t$ determines the  \emc{effective number of basis functions} used to represent the solution.

\subsection{Self-Distillation versus Early Stopping}
\label{sec:early-stop}

Broadly speaking, early stopping can be interpreted as any procedure that cuts convergence short of the optimal solution. Examples include reducing the number of iterations of the numerical optimizer (e.g. SGD), or increasing the loss tolerance threshold $\epsilon$. The former is not applicable to our setting, as our analysis is independent of function parametrization and its numerical optimization. We consider the second definition.

This form of early stopping also has a regularization effect; by increasing $\epsilon$ in (\ref{eq:goal_regress}) the feasible set expands and thus it is possible to find functions with lower $R(f)$. However, we show here that the induced regularization is not equivalent to that of self-distillation. In fact, one can say that early-stopping does the \emc{opposite} of sparsification, as we show below.

The learned function via loss-based early stopping in our notation can be expressed as $f_0^*$ (single training, no self-distillation) with a larger error tolerance $\epsilon$,
\begin{equation}
f^*_0(\boldsymbol{x}) \,=\,  \boldsymbol{p}^T_{\boldsymbol{x}} \boldsymbol{B}_0 \boldsymbol{z}_0 = \boldsymbol{p}^T_{\boldsymbol{x}} \boldsymbol{D}(c_0 \boldsymbol{I} + \boldsymbol{D})^{-1} \boldsymbol{z}_0 \,.
\end{equation}
The effect of larger $\epsilon$ on the value of $c_0$ is shown in (\ref{eq:root_bounds}). However, since $c_0$ is just a scalar value applied to matrices, it does not provide any lever to increase the sparsity of $\boldsymbol{D}$. We now elaborate on the latter claim a bit more. Observe that, on the one hand, when $c_0$ is large, then $\boldsymbol{D}(c_0 \boldsymbol{I} + \boldsymbol{D})^{-1}\approx \frac{1}{c_0} \boldsymbol{D}$, which essentially uses $\boldsymbol{D}$ and does not sparsify it further. On the other hand, if $c_0$ is small then $\boldsymbol{D}(c_0 \boldsymbol{I} + \boldsymbol{D})^{-1}\approx\boldsymbol{I}$, which is the densest possible diagonal matrix. Thus, at best, early stopping maintains the original sparsity pattern of $\boldsymbol{D}$ and otherwise makes it even denser.

\subsection{Advantage of Near Interpolation Regime}

As discussed in Section (\ref{sec:evolution_basis}), for {\it each pair} of $j$ and $k$ satisfying $d_k> d_j$, the ratio $\frac{\boldsymbol{B}_{t-1} [k,k]}{\boldsymbol{B}_{t-1} [j,j]}$ can be interpreted as a sparsity measure (the larger, the sparser). To obtain a sparsity notion that is easier to interpret, here we try to remove its dependency on the specific choice of $i,j$; thus reflecting the sparsity in $\boldsymbol{B}_{t-1}$ by a single quantity. We still rely on the lower bound we developed for $\frac{\boldsymbol{B}_{t-1} [k,k]}{\boldsymbol{B}_{t-1} [j,j]}$ in Theorem~\ref{prop:expansion_rate}. We denote the \emc{sparsity index} of $\boldsymbol{B}_{t-1}$ by $S_{\boldsymbol{B}_{t-1}}$ and define it as the lowest value of the bound across elements all pairs satisfying $d_k > d_j$,
\begin{equation}
S_{\boldsymbol{B}_{t-1}} \triangleq \min_{j,k} \left( \frac{\frac{\| \boldsymbol{y}_0 \|}{\sqrt{K \, \epsilon}} - 1 +  \frac{d_{\min}}{d_j} }{\frac{\| \boldsymbol{y}_0 \|}{\sqrt{K \, \epsilon}} - 1 +  \frac{d_{\min}}{d_k} }\right)^{t} \quad \mbox{s.t.} \quad d_k > d_j \,.
\end{equation}
Assuming $d$'s are ordered so that $d_1 < d_2 < \dots < d_{K}$ then the above simplifies to,
\begin{equation}
\label{eq:sparsity_index}
S_{\boldsymbol{B}_{t-1}} = \min_{k \in \{1,2,\dots,K-1\}} \left( \frac{\frac{\| \boldsymbol{y}_0 \|}{\sqrt{K \, \epsilon}} - 1 +  \frac{d_{\min}}{d_k} }{\frac{\| \boldsymbol{y}_0 \|}{\sqrt{K \, \epsilon}} - 1 +  \frac{d_{\min}}{d_{k+1}} }\right)^{t}\,.
\end{equation}
We now move on to the next interesting question: what is the highest sparsity $S$ that self-distillation can attain? Since $\| \boldsymbol{y}_0 \| > \sqrt{K \, \epsilon}$ and $d_{k+1}>d_k$, the term inside parentheses in (\ref{eq:sparsity_index}) is strictly greater than $1$ and thus $S$ increases in $t$. However, the largest $t$ we can guarantee before a solution collapse happens (provided in Proposition~\ref{prop:under_t}) is $\underline{t} = \frac{\| \boldsymbol{y}_0\|}{\kappa \, \sqrt{K \, \epsilon}} - \frac{1}{\kappa}$. By plugging this $\underline{t}$ into the definition of $S$ from (\ref{eq:sparsity_index}) we eliminate $t$ and obtain the largest sparisity index,
\begin{equation}
\label{eq:def_S_b_under_t}
S_{\boldsymbol{B}_{\underline{t}-1}} = \min_{k \in \{1,2,\dots,K-1\}} \left( \frac{\frac{\| \boldsymbol{y}_0 \|}{\sqrt{K \, \epsilon}} - 1 +  \frac{d_{\min}}{d_k} }{\frac{\| \boldsymbol{y}_0 \|}{\sqrt{K \, \epsilon}} - 1 +  \frac{d_{\min}}{d_{k+1}} }\right)^{\frac{\| \boldsymbol{y}_0\|}{\kappa \, \sqrt{K \, \epsilon}} - \frac{1}{\kappa}} \,.
\end{equation}
We now further show that $S_{\boldsymbol{B}_{\underline{t}-1}}$ always improves as $\epsilon$ gets smaller.

\begin{tcolorbox}[colback=green!8!white,colframe=black]
\begin{theorem}
Suppose $\|\boldsymbol{y}_0\| > \sqrt{K \, \epsilon}$. Then the sparsity index $S_{\boldsymbol{B}_{\underline{t}-1}}$ (where $\underline{t} = \frac{\| \boldsymbol{y}_0\|}{\kappa \, \sqrt{K \, \epsilon}} - \frac{1}{\kappa}$ is number of guaranteed self-distillation steps before solution collapse) {\color{red}decreases} in $\epsilon$, i.e. lower $\epsilon$ yields higher sparsity.

Furthermore at the limit $\epsilon \rightarrow 0$, the sparsity index has the form,
\begin{equation}
\lim_{\epsilon \rightarrow 0} S_{\boldsymbol{B}_{\underline{t}-1}} = e^{\frac{d_{\min}}{\kappa} \, \min_{k \in \{1,2,\dots,K-1\}}(\frac{1}{d_k} - \frac{1}{d_{k+1}})} \,.
\end{equation}
\end{theorem}
\end{tcolorbox}

Thus, if high sparsity is a desired goal, one should choose $\epsilon$ as small as possible. One should however note that the value of $\epsilon$ cannot be \emc{identically zero}, i.e. exact interpolation regime, because then  $\boldsymbol{f}_0=\boldsymbol{y}_0$, and since $\boldsymbol{y}_1=\boldsymbol{f}_0$, self-distillation process keeps producing the same model in each round.

\subsection{Multiclass Extension}

We can formulate multiclass classification, by regressing to a one-hot encoding. Specifically, a problem with $Q$ classes can be modeled by $Q$ output functions $f_1,\dots,f_Q$. An easy extension of our analysis to this multiclass setting is to require  the functions $f_1,\dots,f_Q$ be smooth by applying the same regularization $R$ to each and then adding up these regularization terms. This way, the optimal function for each output unit can be solved for each $q=1,\dots,Q$
\begin{equation}
\label{eq:goal_multiregress}
f_q^* \triangleq \arg\min_{f_q \in \mathcal{F}} \frac{1}{K} \sum_k \Big( f_q(\boldsymbol{x}_k) - {y_q}_k \Big)^2 \,+\, c_q \, R(f_q) \,.
\end{equation}

\subsection{Generalization Bounds}
While the goal of this work is to study the implicit regularization effect of self-distillation, our result can be easily translated into generalization guarantees. Recall from (\ref{eq:sol_f_star_t}) that the regression solution after $t$ rounds of self-distillation has the form $f^*_t(\boldsymbol{x}) \,=\, \boldsymbol{g}_{\boldsymbol{x}}^T \boldsymbol{V}^T \boldsymbol{D}^{-1} ( \Pi_{i=0}^t \boldsymbol{A}_t ) \boldsymbol{V} \boldsymbol{y}_0$. We can show that (proof in Appendix~\ref{sec:equiv_kernel}), there exists a positive definite kernel $g^\dag(\,.\,,\,.\,)$ that performing standard kernel ridge regression with it over the same training data $\cup_{k=1}^K \{(\boldsymbol{x}_k,y_k)\}$ yields the function $f^\dag$ such that $f^\dag=f^*_t$. Furthermore, we can show that the spectrum of the Gram matrix $\boldsymbol{G}^\dag[j,k] \triangleq \frac{1}{K} g^\dag(\boldsymbol{x}_j,\boldsymbol{x}_k)$ in the latter kernel regression problem relates to spectrum of $\boldsymbol{G}$ via,
\begin{equation}
\label{eq:g_g_dag}
d_k^\dag =  c_0 \frac{1}{\frac{\Pi_{i=0}^t (d_k+ c_i)}{d_k^{t+1}}-1} \,.
\end{equation}

The identity (\ref{eq:g_g_dag}) enables us to leverage existing generalization bounds for standard kernel ridge regression. These results often only need the spectrum of the Gram matrix. For example, Lemma~22 in~\cite{BartlettM02} shows the Rademacher complexity of the kernel class is proportional to $\sqrt{\Trace(\boldsymbol{G}^\dag)}=\sqrt{\sum_{k=1}^K d_k^\dag}$ and then Theorem~8 of~\cite{BartlettM02} translates that Rademacher complexity into a generalization bound. Note that $\frac{\Pi_{i=0}^t (d_k+ c_i)}{d_k^{t+1}}$ increases in $t$, which implies $d^\dag_k$ and consequently $\sqrt{\Trace(\boldsymbol{G}^\dag)}$ decreases in $t$. 


A more refined bound in terms of the tail behavior of the eigenvalues $d^\dag_k$ (to better exploit the sparsity pattern) is the Corollary~6.7 of~\cite{Bartlett2005} which provides a generalization bound that is affine in the form $\min_{k \in \{0,1,\dots,K\}} \left( \frac{k}{K}+\sqrt{\frac{1}{K} \sum_{j=k+1}^K d_j^\dag} \right)$, where the eigenvalues $d_k^\dag$ for $k=1,\dots,K$, are sorted in non-increasing order .

\section{Illustrative Example}
\label{sec:illustrative}

Let $\mathcal{F}$ be the space of twice differentiable functions that map $[0,1]$ to $\mathbb{R}$,
\begin{equation}
\mathcal{F} \triangleq \{ f \,|\, f:[0,1] \rightarrow \mathbb{R} \}\,.
\end{equation}
Define the linear operator $P: \mathcal{F} \rightarrow \mathcal{F}$ as,
\begin{equation}
\label{eq:operator}
[P f](x) \triangleq \frac{d^2}{dx^2} f(x)\,,
\end{equation}
subject to boundary conditions,
\begin{equation}
\label{eq:boundary_cond}
f(0) = f(1) = f^{\prime\prime}(0) = f^{\prime\prime}(1) \,=\, 0 \,. 
\end{equation}
The associated regularization functional becomes,
\begin{equation}
\label{eq:R_illustrative}
R(f) \triangleq \int_0^1 \big( \frac{d^2}{dx^2} f(x) \big)^2 \, dx\,.
\end{equation}
Observe that this regularizer encourages smoother $f$ by penalizing the second order derivative of the function. The Green's function of the operator associated with the kernel of $R$ subject to the listed boundary conditions is a spline with the following form \cite{Helene} (see Figure~\ref{fig:illustrative}-a),
\begin{equation}
g(x,x^\dag) = \frac{1}{6}\max\big((x-x^\dag)^3,0 \big) - \frac{1}{6}x(1-x^\dag)(x^2-2x^\dag+{x^\dag}^2) \,.
\end{equation}

\begin{figure*}
\begin{centering}
\begin{tabular}{c c c}
\includegraphics[width=2.0in]{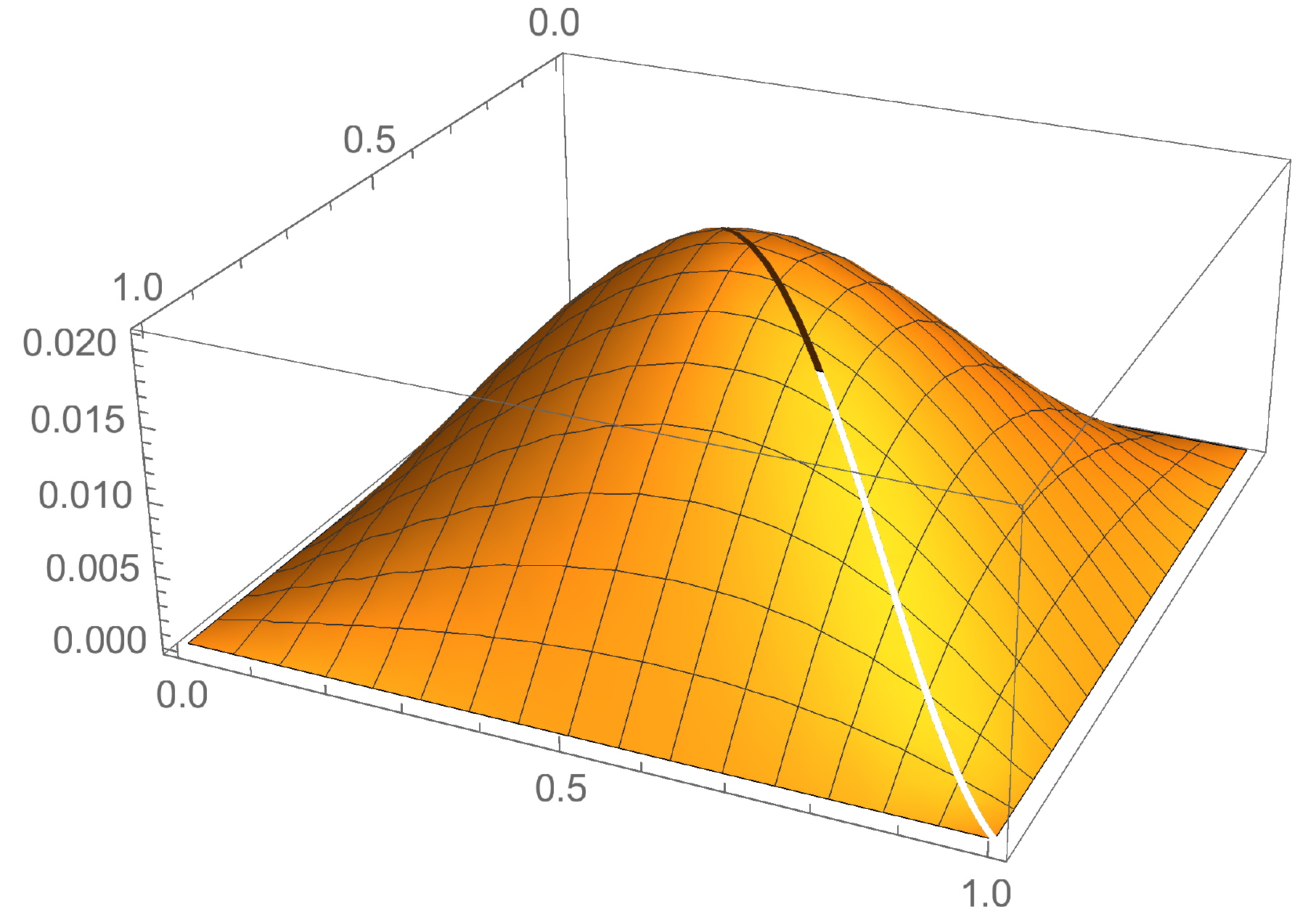} &
\includegraphics[width=2.0in]{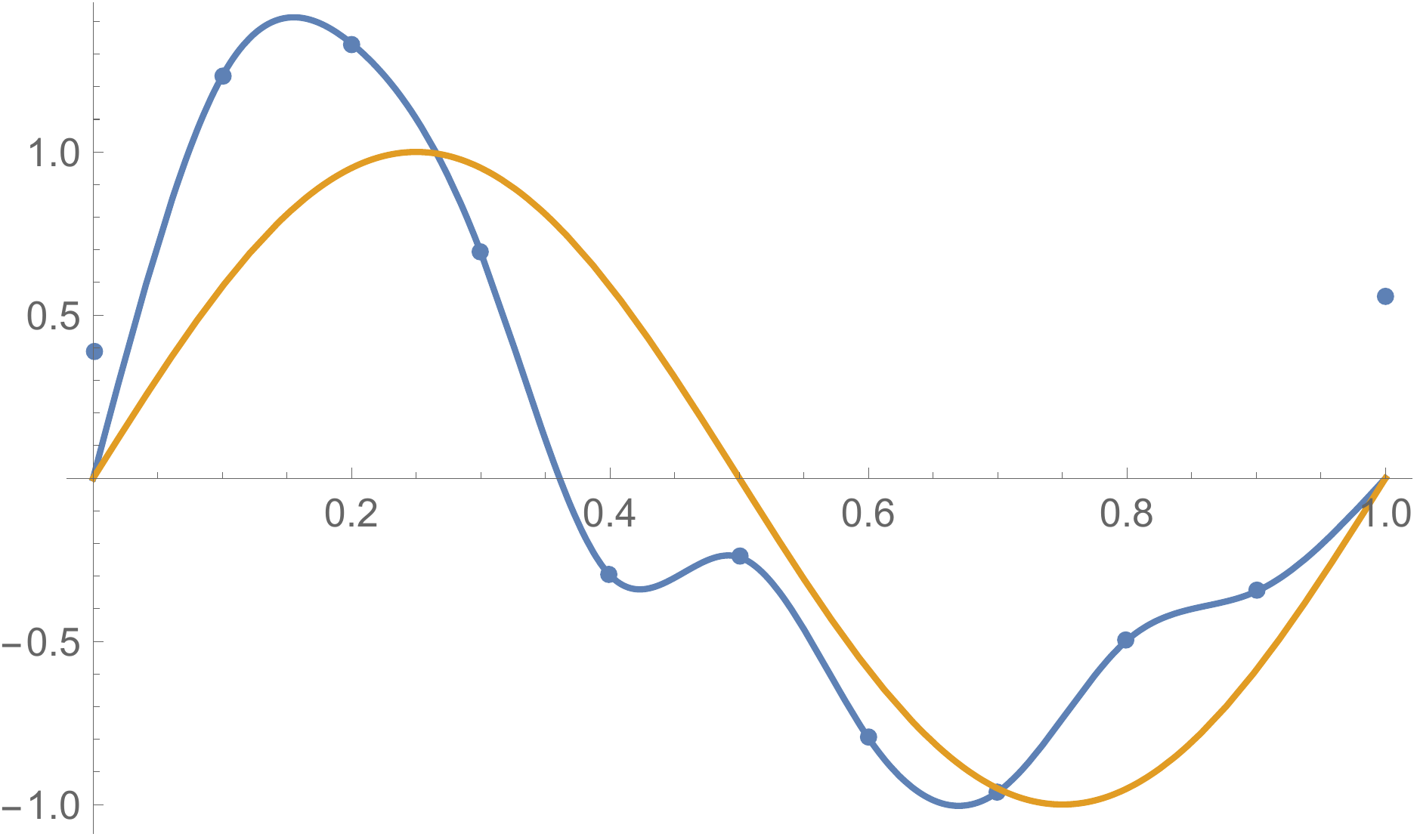} &
\includegraphics[width=2.0in]{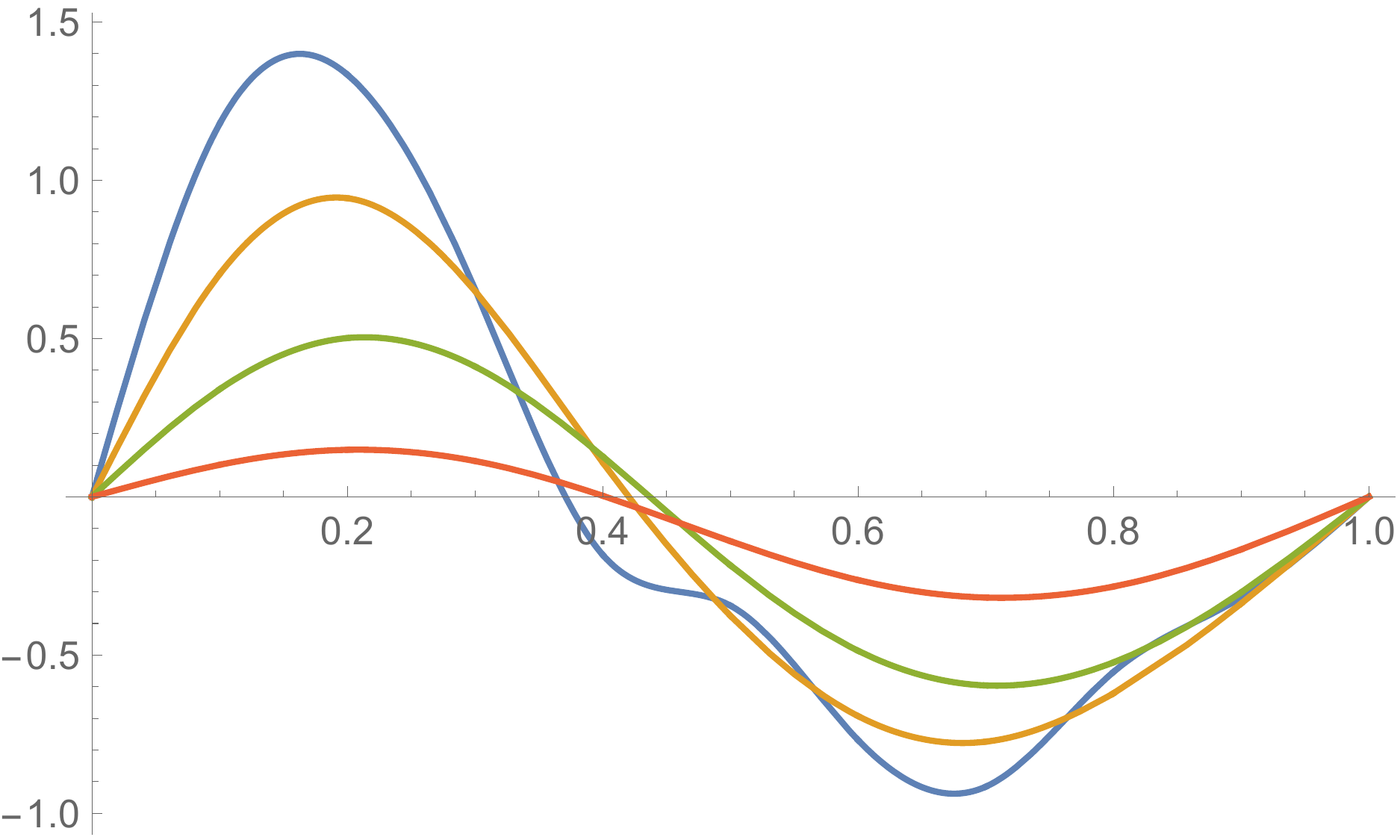} \\
(a) & (b) & (c)
\end{tabular}
\caption{Example with $R(f)(x) \triangleq \int_0^1 \big( \frac{d^2}{dx^2} f(x) \big)^2 \, dx$. 
(a) Green's function associated with the kernel of $R$. (b) Noisy training samples (blue dots) from  underlying function (orange) $y=\sin(2\pi x)$. Fitting without regularization leads to overfitting (blue curve). (c) Four rounds of self-distillation (blue, orange, green, red) with $\epsilon=0.04$.}
\label{fig:illustrative}
\end{centering}
\end{figure*}

\begin{figure*}
\begin{centering}
\includegraphics[width=6.0in]{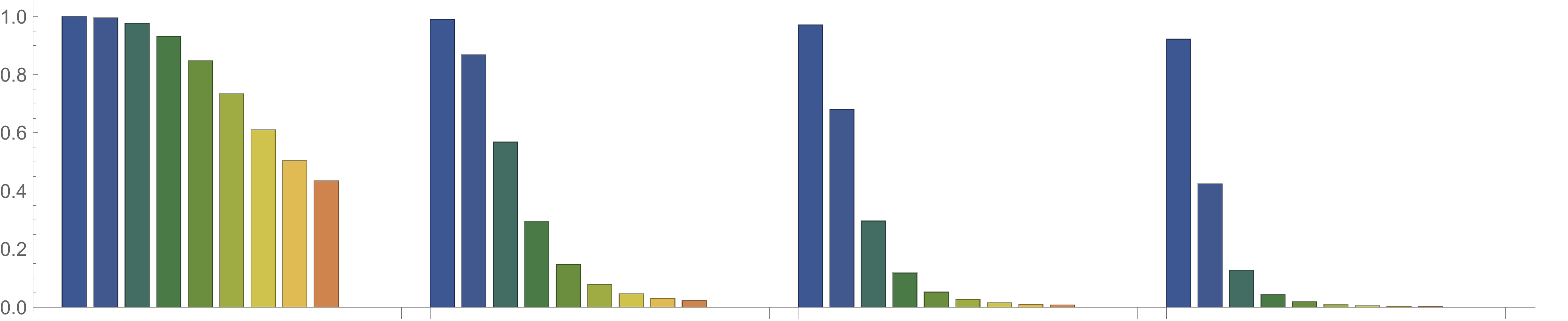}
\caption{Evolution of the diagonal entries of (the diagonal matrix) $\boldsymbol{B}_t$ from (\ref{eq:def_B}) at distillation rounds $t=0$ (left most) to $t=3$ (right most). The number of training points is $K=11$, so $\boldsymbol{B}_t$ which is $K \times K$ diagonal matrix has $11$ entries on its diagonal, each corresponding to one of the bars in the chart.}
\label{fig:charts}
\end{centering}
\end{figure*}

Now consider training points $(x_k,y_k)$ sampled from the function $y=\sin(2 \pi x)$. Let $x_k$ be evenly spaced in the interval $[0,1]$ with steps of $0.1$, and $y_k=x_k+\eta$ where $\eta$ is a zero-mean normal random variable with $\sigma=0.5$ (Figure~\ref{fig:illustrative}-b).

As shown in Figure~\ref{fig:illustrative}-c,  the regularization induced by self-distillation initially improves the quality of the fit, but after that point additional rounds of self-distillation over-regularize and lead to underfitting.

We also computed the diagonal matrix  $\boldsymbol{B}_t$ (see (\ref{eq:def_B}) for definition) at each self-distillation round $t$, for $t=0,\dots,3$ (after that, the solution collapses). Recall from (\ref{eq:B_as_basis_coeffs}) that the entries of this matrix can be thought of as the coefficients of basis functions used to represent the solution. As predicted by our analysis, self-distillation regularizes the solution by sparsifying these coefficients. This is evident in Figure~\ref{fig:charts} where smaller coefficients shrink faster.

\section{Experiments}
\label{sec:experiments}
In our experiments, we aim to empirically evaluate our theoretical analysis in the setting of deep neural networks. Although our theoretical results apply to Hilbert space rather than neural networks, recent findings show that at least very wide neural networks can be viewed as a reproducing kernel Hilbert space, which is equivalent to the setup we study here (with the kernel being the Green's function).
We adopt a clear and simple setup that is easy to reproduce (see the provided code) and also light-weight enough to run more then 10 consecutive rounds of self-distillation. Note that we are not aiming for state-of-the-art performance; readers interested in stronger baselines on studying self-distillation are referred to~\cite{born-again-2018, YangXQY19, Ahn2019VariationalID}. Note that, however, these works are limited to one or two rounds of self-distillation. The ability to run self-distillation for a larger number of rounds allows us to demonstrate the eventual decline of the test performance. To the best of our knowledge, this is the first time that the performance decline regime is observed. The initial improvement and later continuous decline is consistent with our theory, which shows rounds of self-distillation continuously amplify the regularization effect. While initially this may benefit generalization, at some point the excessive regularization leads to underfitting.

\subsection{Motivation}

Recent works on the Neural Tangent Kernel (NTK) \cite{Jacot2018} have shown that training deep models with infinite width and $\ell_2$ loss (and small step size for optimization) is equivalent to performing \emc{kernel regression} with a specific kernel. The kernel function, which is outer product of network's Jacobian matrix, encodes various biases induced by the architectural choices (convolution, pooling, nested representations, etc.) that collectively enable the deep models generalize well despite their high capacity.

The regularization form (\ref{eq:regularizer}) that we studied here also reduces to a kernel regression problem, with the kernel being the Green's function of the regularizer. In fact, regularized regression (\ref{eq:goal_regress}) and kernel ridge regression can be converted to each other \cite{Smola98} and thus, in principle, one can convert an NTK kernel\footnote{While the pure NTK theory is typically introduced as unregularized kernel regression, its practical instantiations often involve regularization (for numerical stability and other reasons) \cite{Lee2020, shankar2020neural}.} into a regularizer of form (\ref{eq:regularizer}). This implies that at least in the NTK regime of neural networks, our analysis can provide a reasonable representation of self-distillation.

Of course, real architectures have finite width and thus the NTK (and consequently our self-distillation analysis) may not always be a faithful approximation. However, the growing literature on NTK is showing scenarios where this regime is still valid under large width \cite{NIPS2019_9063}, particular choices of scaling between the weights and the inputs \cite{NIPS2019_8559}, and for fully connected networks \cite{2019arXiv190608034G}.

We hope our analysis can provide some insight into how self-distillation dynamic affects generalization. For example, the model may benefit from stronger regularizer encoded by the underlying regularizer (or equivalently kernel), and thus improve on test performance initially. However, as we discussed, excessive self-distillation can over regularize the model and thus lead to underfitting. According to this picture, the test accuracy may first go up but then will go down (instead of approaching its best value, for example). Our empirical results on deep models follow this pattern.

\subsection{Results}
\paragraph{Setup.} We use Resnet \cite{He2015DeepRL} and VGG \cite{Simonyan15} neural architectures and train them on CIFAR-10 and CIFAR-100 datasets \cite{Krizhevsky2009LearningML}. Training details and additional results are left to the appendix.
Each curve in the plots corresponds to 10 runs from randomly initialized weights, where each run is a chain of self-distillation steps indicated in the $x$-axis. In the plots, a point represents the average and the envelope around it reflects standard deviation. Any training accuracy reported here is based on assessing the model $f_t$ at the $t$'th self-distillation round on the \emc{original} training labels $\boldsymbol{y}_0$.

\paragraph{$\ell_2$ Loss on Neural Network Predictions.} Here we train the neural network using $\ell_2$ loss. The error is defined as the difference between predictions (softmax over the logits) and the target labels. These results are in concordance with a regularization viewpoint of self-distillation. The theory suggests that self-distillation progressively amplifies the underlying regularization effect. As such, we expect the training accuracy (over $\boldsymbol{y}_0$) to drop in each self-distillation round. Test accuracy may go up if training can benefit from amplified regularization. However, from the theory we expect the test accuracy to go down at some point due to over regularization and thus underfitting. Both of these phenomena are observed in Figure~\ref{fig:L2-pred-resnet50-cifar10}.

\begin{figure}[t]
    \centering
    \begin{tabular}{c c c c}
    \hspace{-4mm} 
            \includegraphics[width=0.245\textwidth]{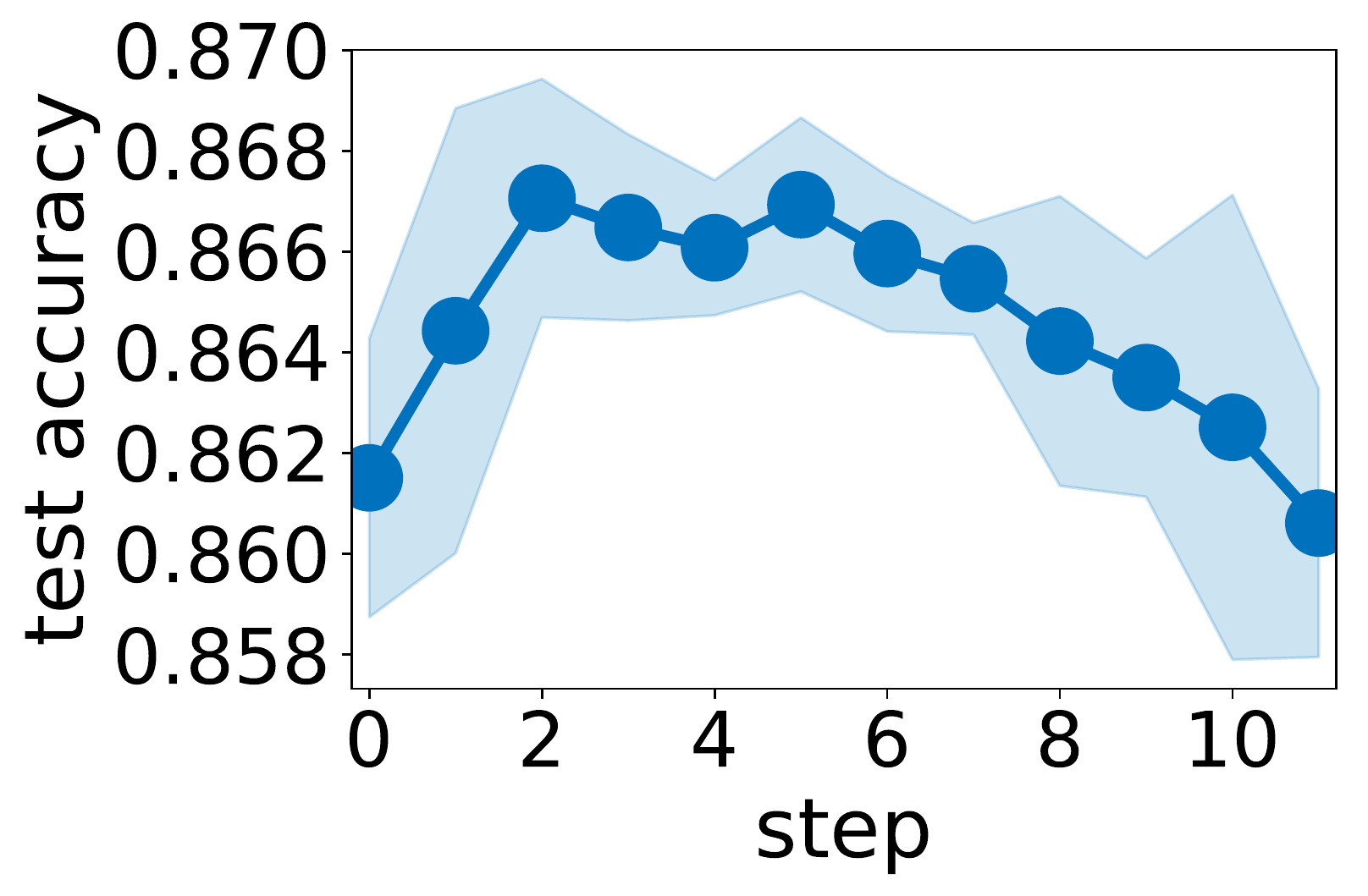} & \hspace{-5.5mm} 
            \includegraphics[width=0.245\textwidth]{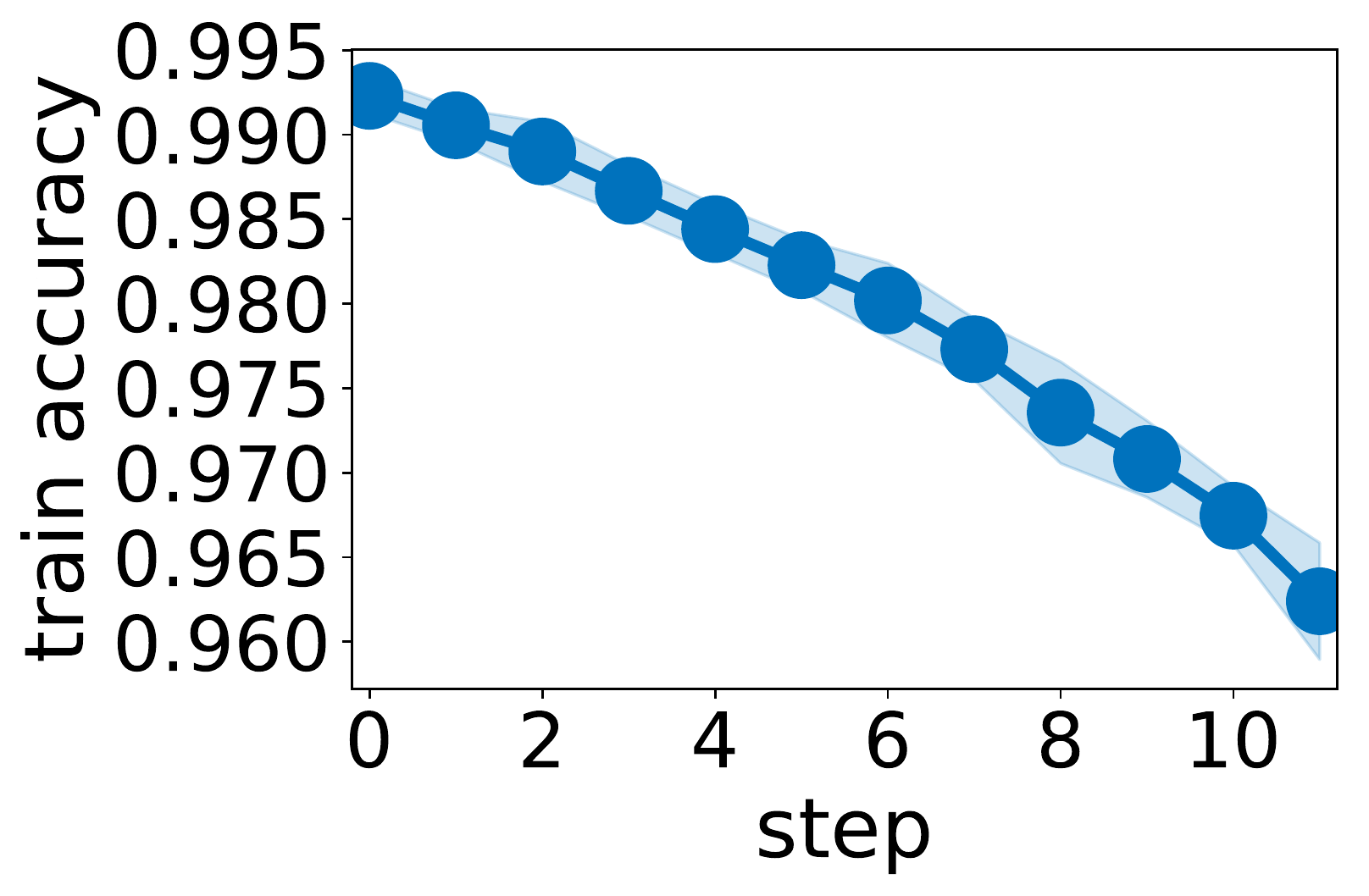} & \hspace{-2mm}
            \includegraphics[width=0.25\textwidth]{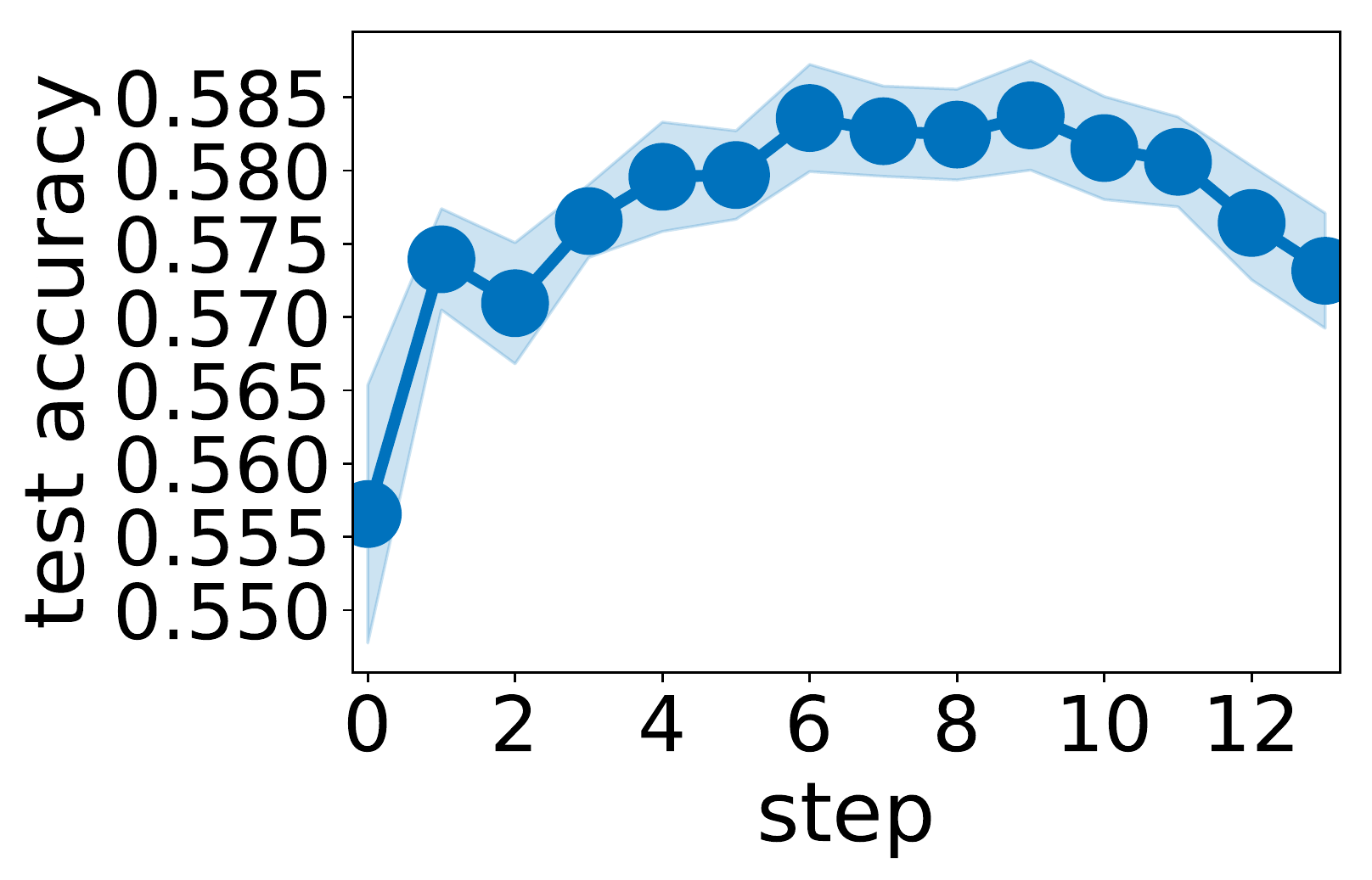} & \hspace{-5.5mm}
            \includegraphics[width=0.25\textwidth]{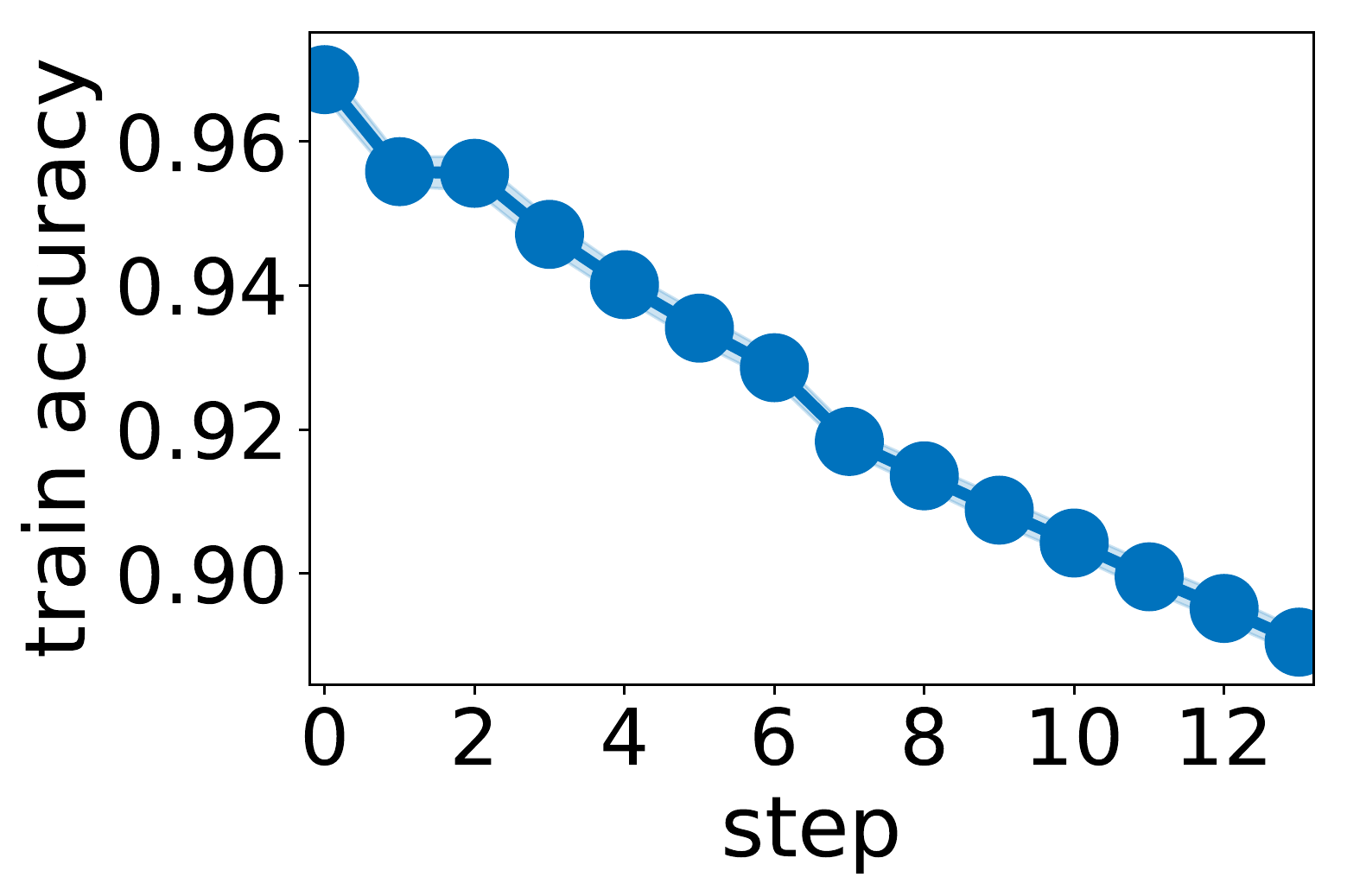} 
    \end{tabular}
    \caption{Accuracy of self-distillation steps using Resnet with $\ell_2$ loss on neural network predictions.  
    (Left 2 plots): test/train accuracy on CIFAR-10.
    (Right 2 plots): test/train accuracy on CIFAR-100.
    }
    \label{fig:L2-pred-resnet50-cifar10}
\end{figure}

\paragraph{Cross-Entropy Loss on Neural Network Predictions.} Although, our theory only applies to $\ell_2$,  loss, we empirically observed similar phenomena for cross-entropy as shown in Figure~\ref{fig:ce-pred-resnet50-cifar10}.

\paragraph{Self-Distillation versus Early Stopping.} By looking at the fall of the training accuracy over self-distillation round, one may wonder if early stopping (in the sense of choosing a larger error tolerance $\epsilon$ for training) would lead to similar test performance. However, in Section \ref{sec:early-stop} we discussed that self-distillation and early stopping have different regularization effects. Here we try to verify that.
Specifically, we record the training loss value at the end of each self-distillation round. We then train a batch of models from scratch until each batch converges to one of the recorded loss values. If the regularization induced by early stopping was the same as self-distillation, then we should have seen similar test performance between a self-distilled model that achieves a specific loss value on the original training labels, and a model that stops training as soon as it reaches the same level of error. However, the left two plots in Figure~\ref{fig:sd-vs-es-resnet50-cifar10} verifies that these two have different regularization effects. 

\paragraph{Self-Distillation on Other Networks.} The right two plots in Figure~\ref{fig:sd-vs-es-resnet50-cifar10} show the performance of $\ell_2$ distillation on CIFAR-100 using VGG network. This experiment aims to show that the theory and empirical findings are not dependent to a specific structure and apply to architectures beyond Resnet.

\begin{figure}[t]
    \centering
    \begin{tabular}{c c c c}
     \hspace{-4mm}
    \includegraphics[width=0.25\textwidth]{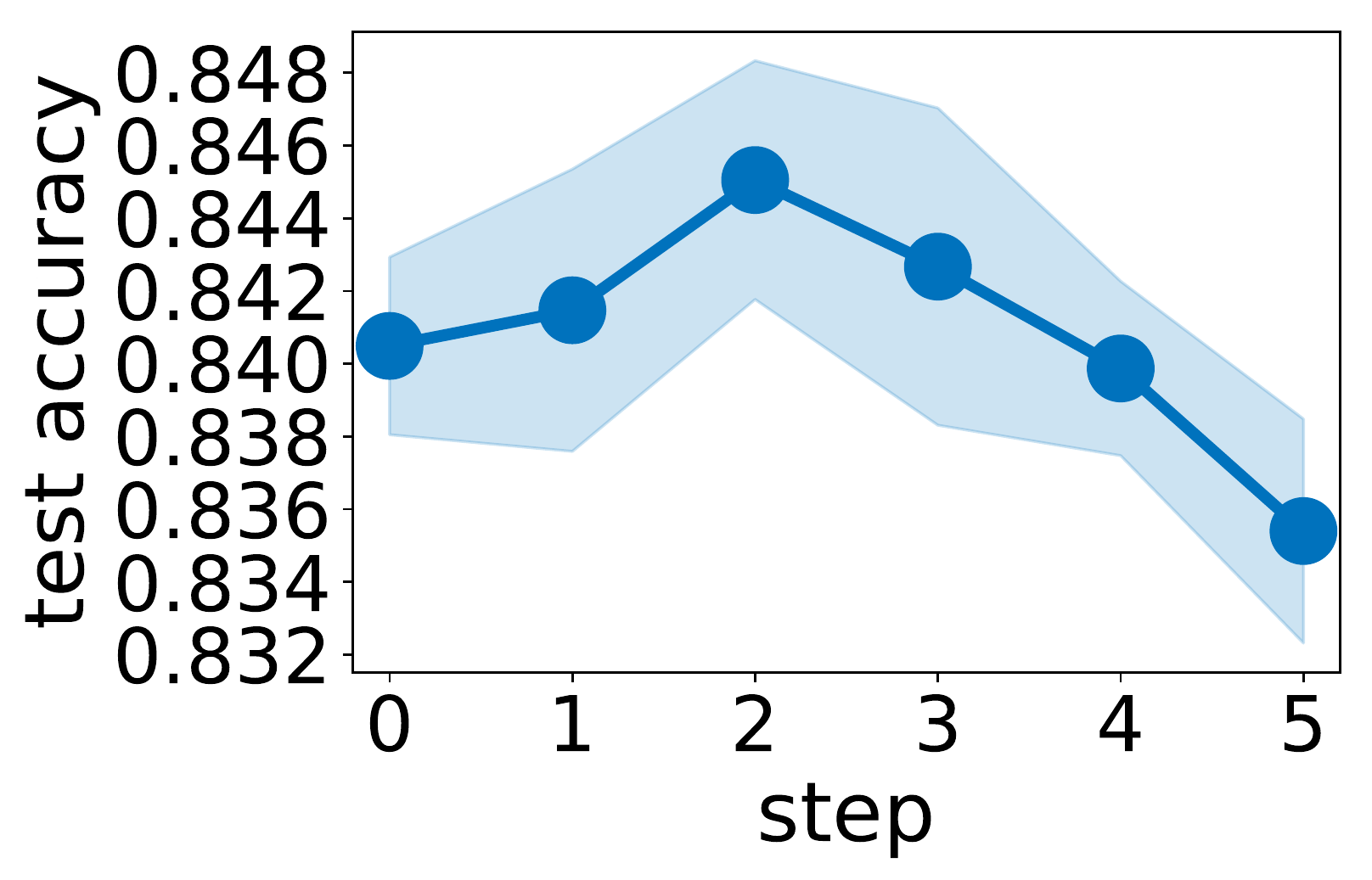} &  \hspace{-5.5mm}
    \includegraphics[width=0.245\textwidth]{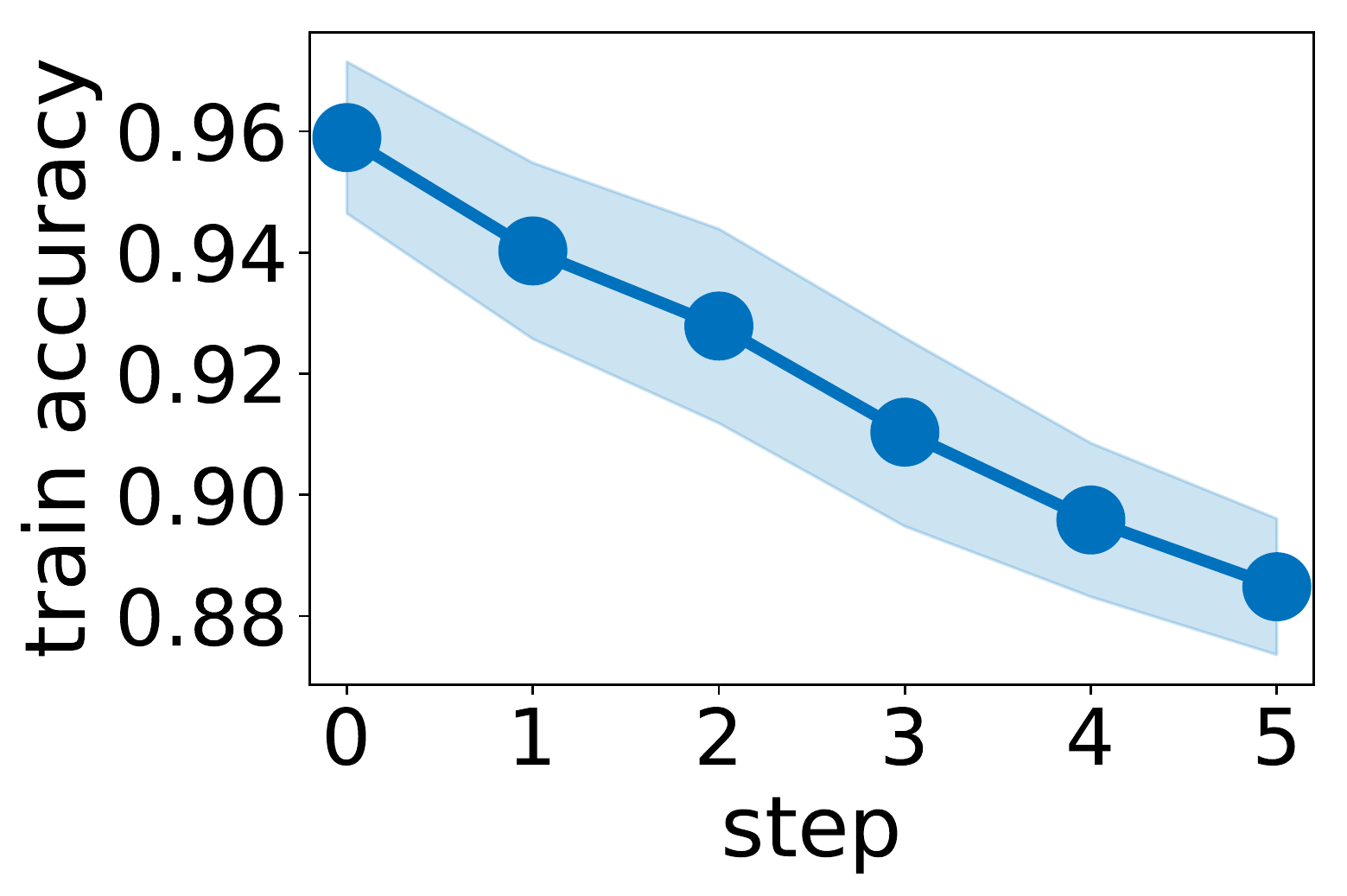} &  \hspace{-2mm}
        \includegraphics[width=0.24\textwidth]{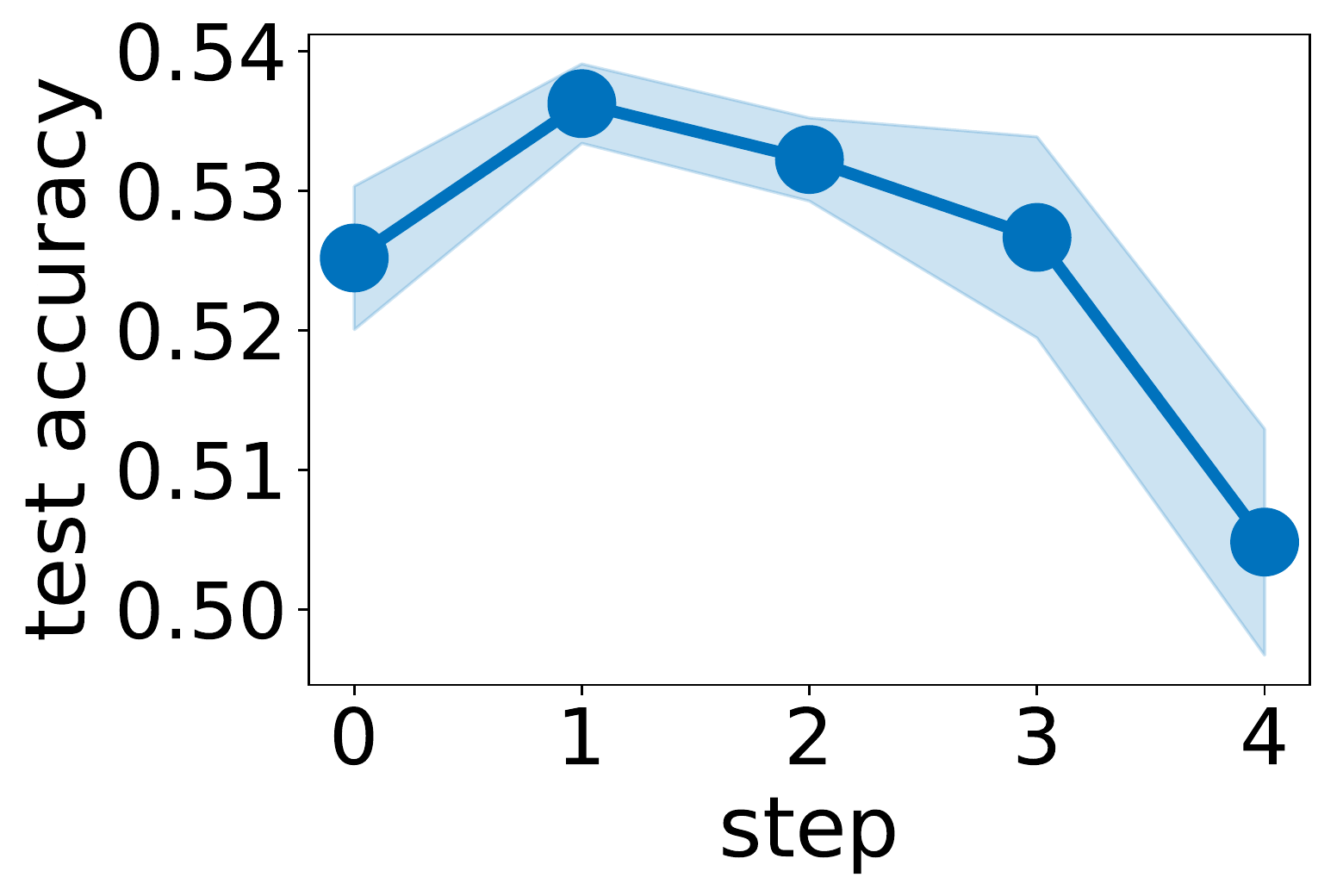} &
    \includegraphics[width=0.24\textwidth]{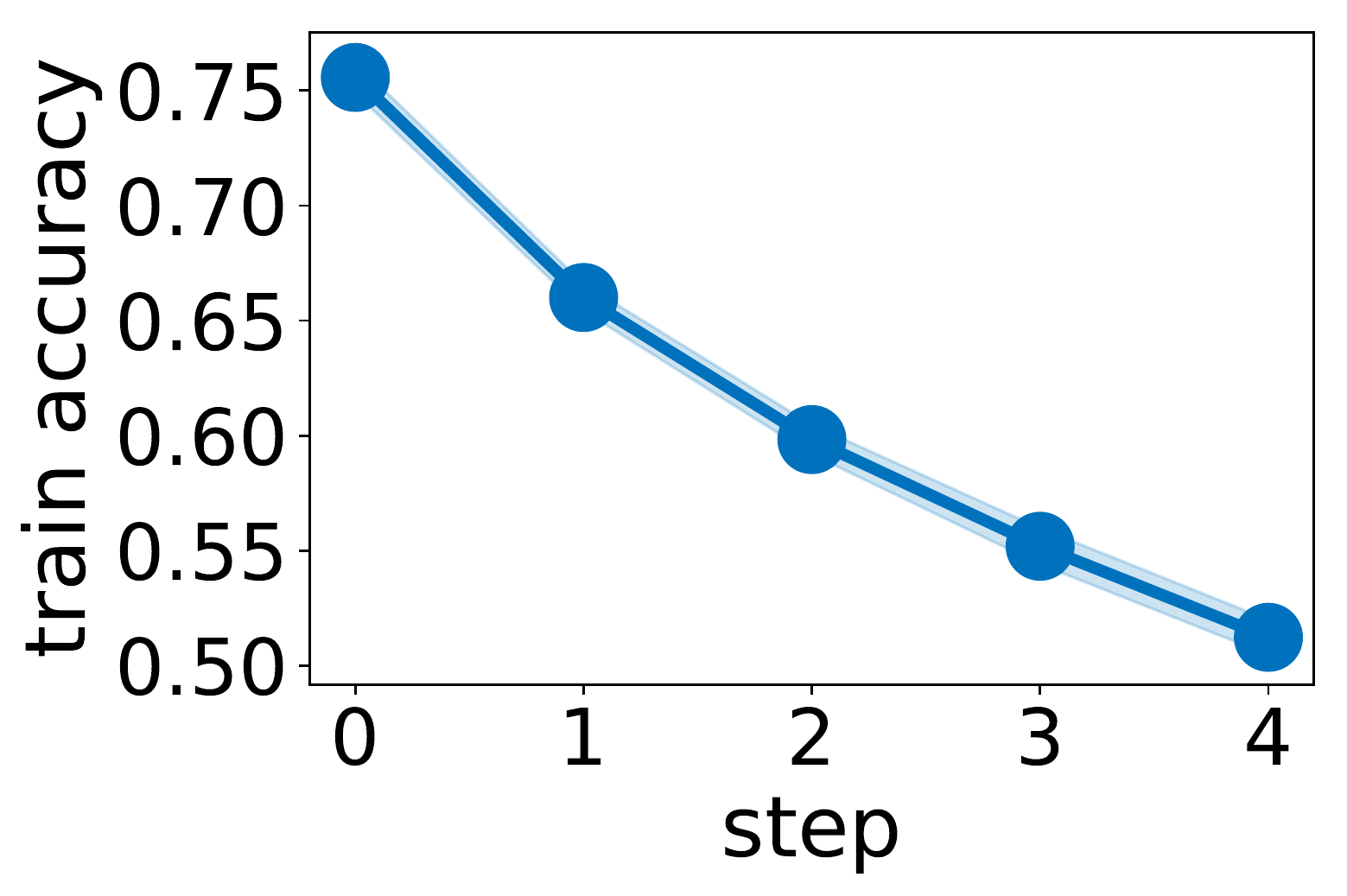}
    \end{tabular}
    \caption{
Self-distillation steps using Resnet with cross-entropy loss on neural network predictions.
    (Left 2 plots): test/train accuracy on CIFAR-10.
    (Right 2 plots): test/train accuracy on CIFAR-100.    
    }
    \label{fig:ce-pred-resnet50-cifar10}
\end{figure}

\begin{figure}[t]
    \centering
    \begin{tabular}{c c c c}
    \hspace{-4mm}
    \includegraphics[width=0.245\textwidth]{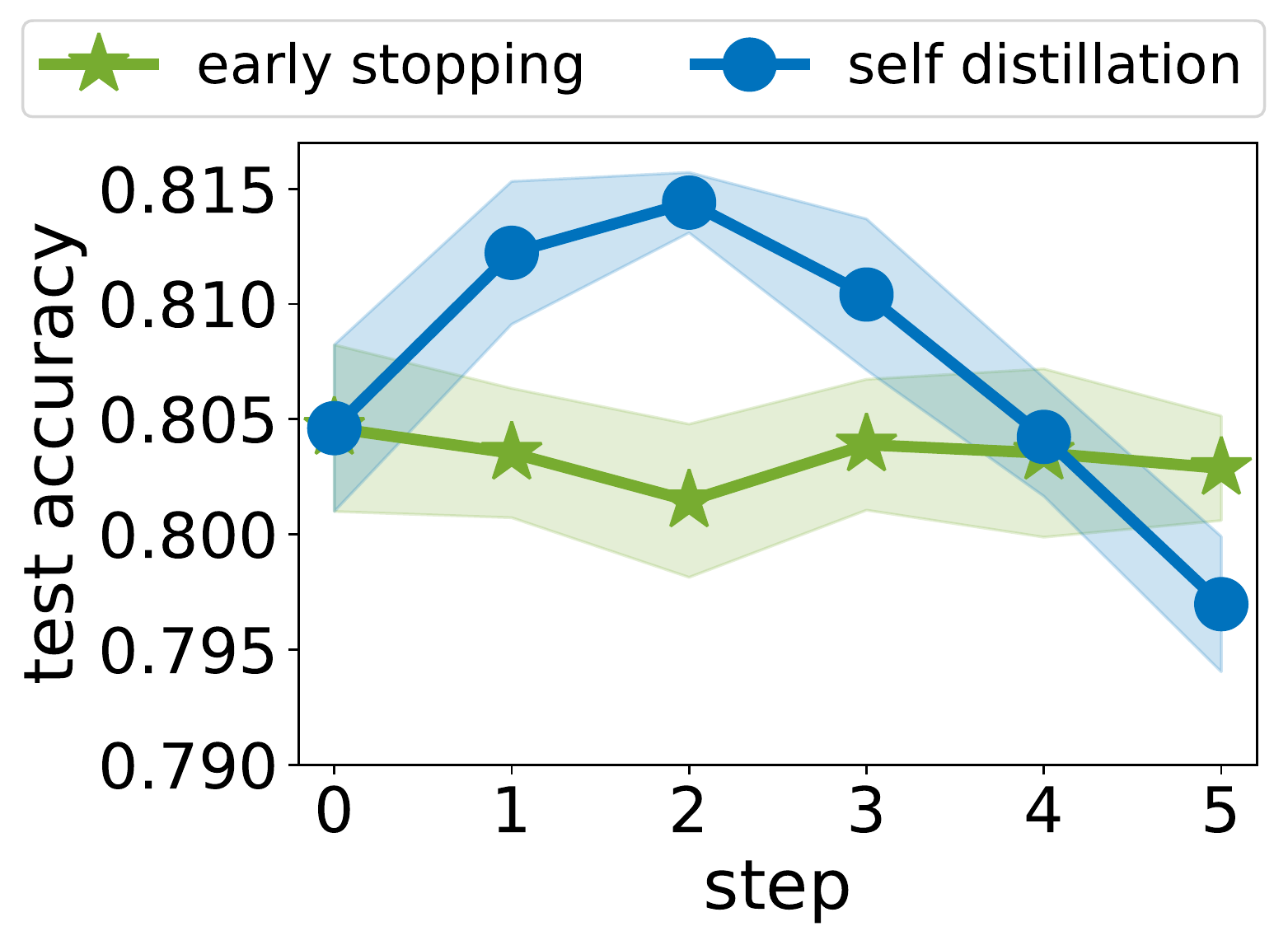} & \hspace{-5.5mm}
    \includegraphics[width=0.245\textwidth]{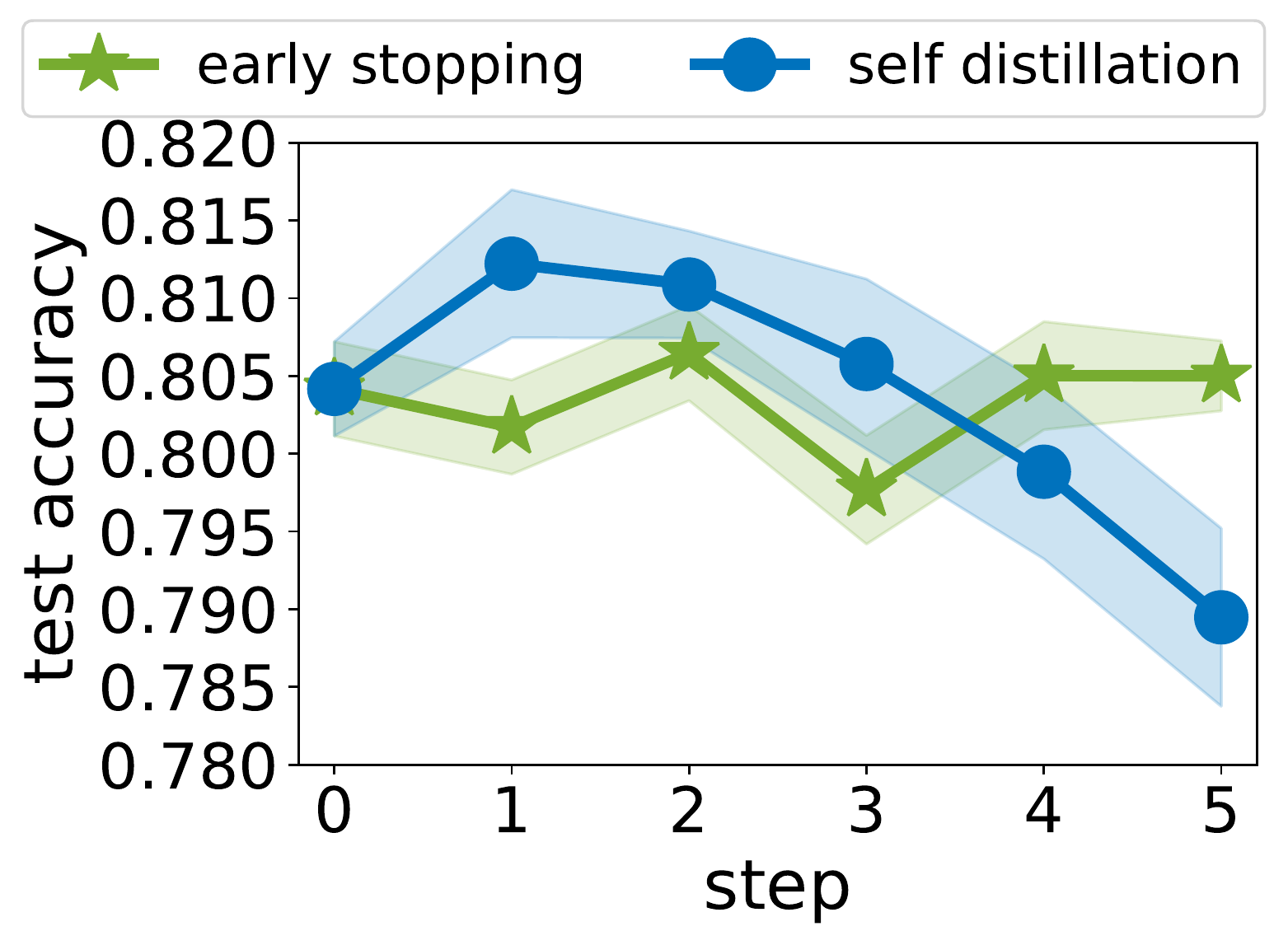} &     \includegraphics[width=0.245\textwidth]{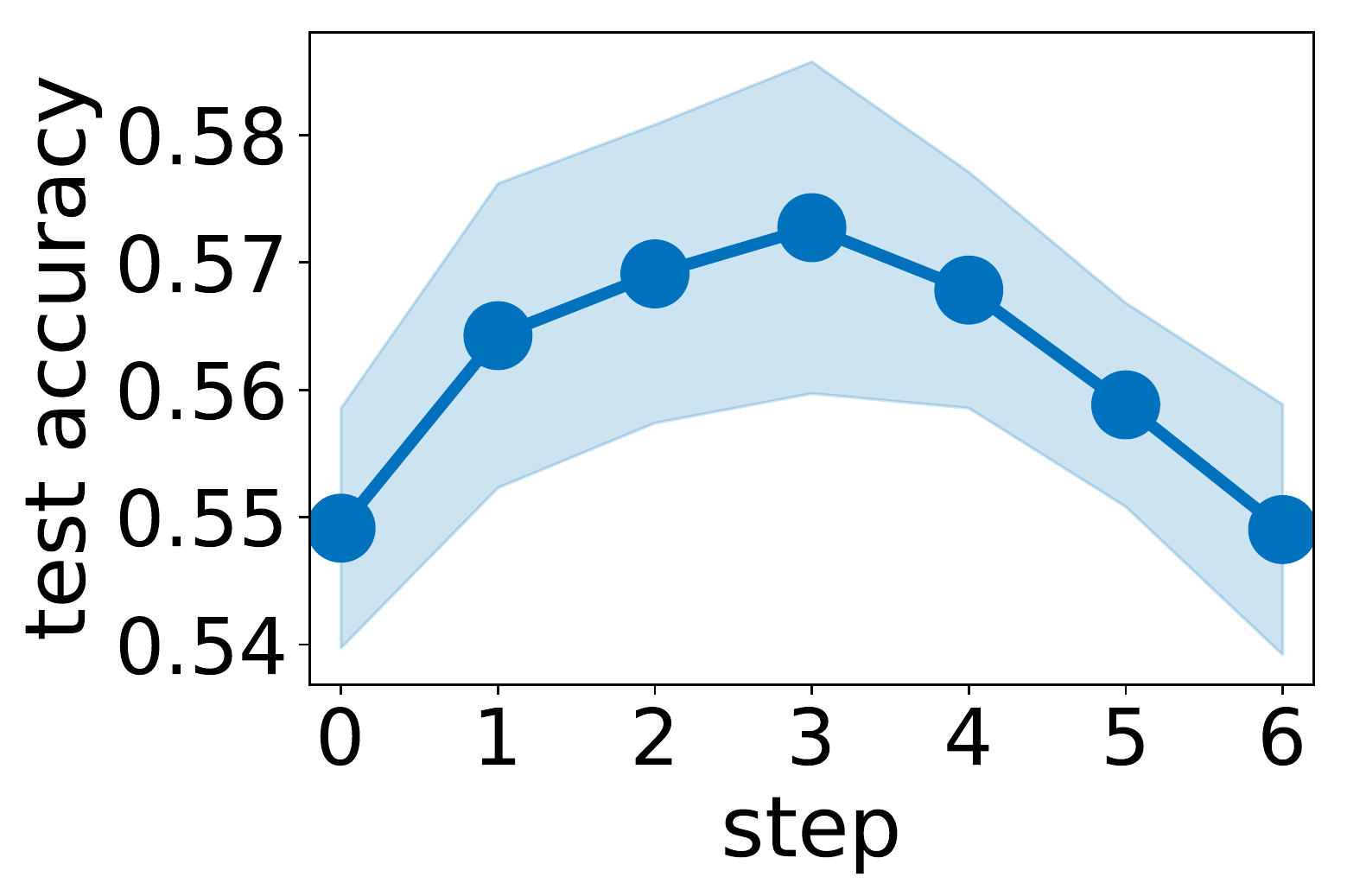} &  \hspace{-5.5mm}
    \includegraphics[width=0.245\textwidth]{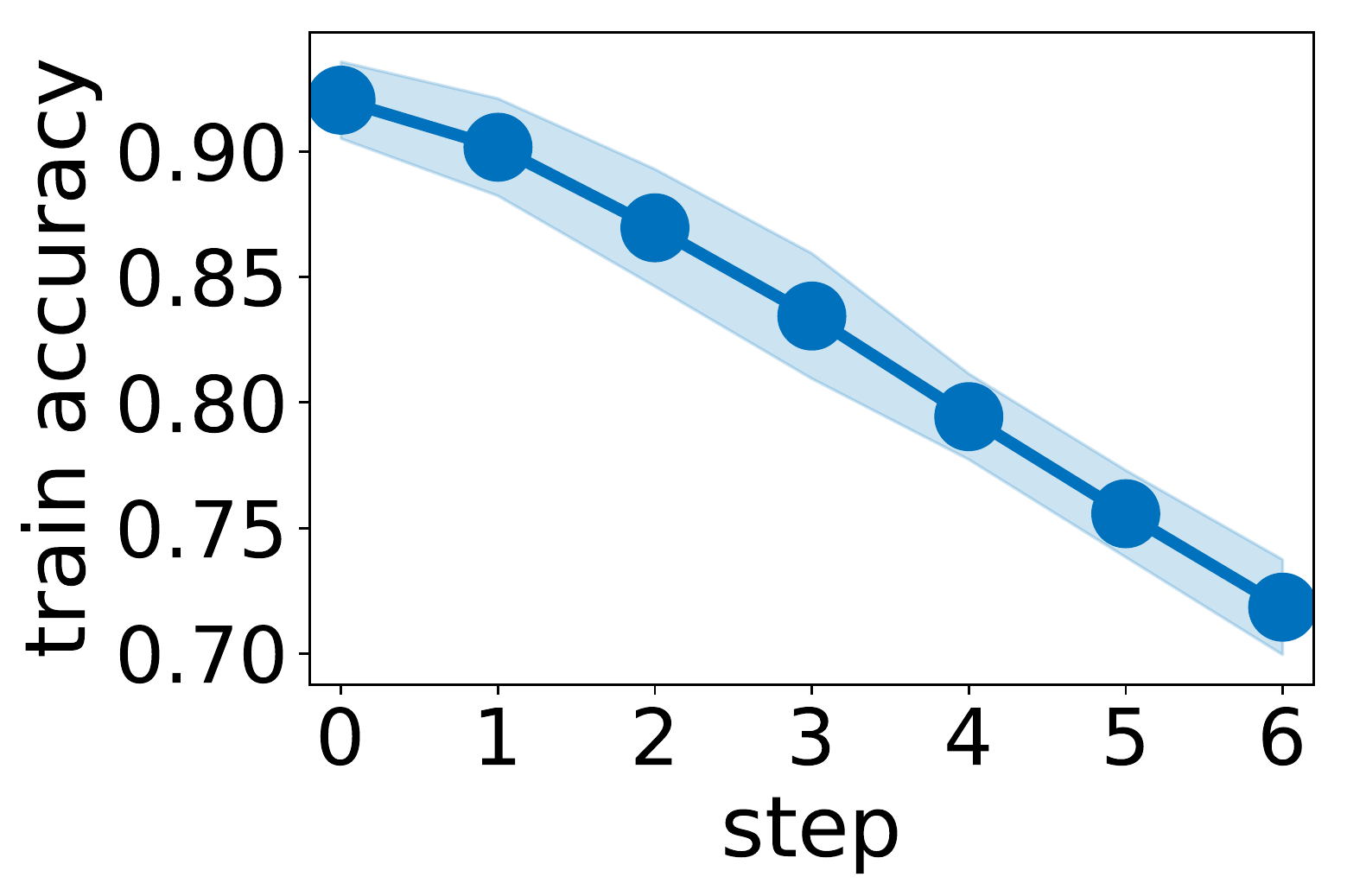} 
    \end{tabular}
    \caption{ (Left 2 plots): self-distillation compared to early stopping for Resnet50 and CIFAR-10 using $\ell_2$ and cross-entropy loss, respectively.  (Right 2 plots): self-distillation with $\ell_2$ loss using VGG16 Network on CIFAR-100.}
    \label{fig:sd-vs-es-resnet50-cifar10}
\end{figure}

\section{Conclusion}
In this work, we presented a rigorous analysis of self-distillation for regularized regression in a Hilbert space of functions. We showed that self-distillation iterations in the setting we studied cannot continue indefinitely; at some point the solution collapses to zero. We provided a lower bound on the number of meaningful (non-collapsed) distillation iterations. In addition, we proved that self-distillation acts as a regularizer that progressively employs fewer basis functions for representing the solution. We discussed the difference in regularization effect induced by self-distillation against early stopping. We also showed that operating in near-interpolation regime facilitates the regularization effect. We discussed how our regression setting resembles the NTK view of wide neural networks, and thus may provide some insight into how self-distillation works in deep learning. 

We hope that our work can be used as a stepping stone to broader settings. In particular, studying cross-entropy loss as well as other forms of regularization are interesting directions for further research.

\section*{Acknowledgement}
We would like to thank colleagues at Google Research for their feedback: Moshe Dubiner, Pierre Foret, Sergey Ioffe, Yiding Jiang, Alan MacKey, Sam Schoenholz, Matt Streeter, and Andrey Zhmoginov. We also thank {\it NeurIPS2020} conference anonymous reviewers for their comments.

\bibliography{main}
\bibliographystyle{apalike}

\newpage
\appendix

\section{Solving the Variational Problem}
\label{sec:app_variational}

In this section we derive the solution to the following variational problem,
\begin{equation}
f^* \triangleq \arg\min_{f \in \mathcal{F}} \frac{1}{K} \sum_k \Big( f(\boldsymbol{x}_k) - y_k \Big)^2 \,+\, c \,  \int_{\mathcal{X}} \int_{\mathcal{X}} u(\boldsymbol{x},\boldsymbol{x}^\dag) f(\boldsymbol{x}) f(\boldsymbol{x}^\dag) \, d \boldsymbol{x} \, d \boldsymbol{x}^\dag \,.
\end{equation}
Using Dirac delta function, we can rewrite the objective function as,
\begin{equation}
f^* = \arg\min_{f \in \mathcal{F}} \frac{1}{K} \sum_k \big( \int_{\mathcal{X}} f(\boldsymbol{x}) \delta(\boldsymbol{x}-\boldsymbol{x}_k) \, d \boldsymbol{x} - y_k \big)^2 \,+\, c \int_{\mathcal{X}} \int_{\mathcal{X}} u(\boldsymbol{x},\boldsymbol{x}^\dag) f(\boldsymbol{x}) f(\boldsymbol{x}^\dag) \, d \boldsymbol{x} \, d \boldsymbol{x}^\dag\,.
\end{equation}
For brevity, name the objective functional $J$,
\begin{equation}
J(f) \triangleq \frac{1}{K} \sum_k \big( \int_{\mathcal{X}} f(\boldsymbol{x}) \delta(\boldsymbol{x}-\boldsymbol{x}_k) \, d \boldsymbol{x} - y_k \big)^2 \,+\, c \int_{\mathcal{X}} \int_{\mathcal{X}} u(\boldsymbol{x},\boldsymbol{x}^\dag) f(\boldsymbol{x}) f(\boldsymbol{x}^\dag) \, d \boldsymbol{x} \, d \boldsymbol{x}^\dag \,.
\end{equation}
If $f^*$ minimizes the $J(f)$, it must be a stationary point of $J$. That is, $J(f+\epsilon \phi)=J(f)$, for any $\phi \in \mathcal{F}$ as $\epsilon \rightarrow 0$. More precisely, it is necessary for $f^*$ to satisfy,
\begin{equation}
\forall \phi \in \mathcal{F} \,;\, \big(\frac{d}{d \epsilon} J(f^*+\epsilon \phi) \big)_{\epsilon=0} \, = \,0 \,.
\end{equation}
We first construct $J(f^*+\epsilon \phi)$,
\begin{eqnarray}
J(f^*+\epsilon \phi) &=& \frac{1}{K} \sum_k \big( \int_{\mathcal{X}} [f^* + \epsilon \phi](\boldsymbol{x}) \delta(\boldsymbol{x}-\boldsymbol{x}_k) \, d \boldsymbol{x} - y_k \big)^2 \\
&+& c \int_{\mathcal{X}} \int_{\mathcal{X}} u(\boldsymbol{x},\boldsymbol{x}^\dag) [f^* + \epsilon \phi](\boldsymbol{x}) [f^* + \epsilon \phi](\boldsymbol{x}^\dag) \, d \boldsymbol{x} \, d \boldsymbol{x}^\dag \,,
\end{eqnarray}
or equivalently,
\begin{eqnarray}
J(f^*+\epsilon \phi) &=& \frac{1}{K} \sum_k \Big( \int_{\mathcal{X}} \big( f^*(\boldsymbol{x}) + \epsilon \phi(\boldsymbol{x}) \big) \delta(\boldsymbol{x}-\boldsymbol{x}_k) \, d \boldsymbol{x} - y_k \Big)^2 \\
&+& c \, \int_{\mathcal{X}} \int_{\mathcal{X}} u(\boldsymbol{x},\boldsymbol{x}^\dag) \big( f^*(\boldsymbol{x}) + \epsilon \phi(\boldsymbol{x}) \big) \, \big( f^*(\boldsymbol{x}^\dag) + \epsilon \phi(\boldsymbol{x}^\dag) \big) \, d \boldsymbol{x}^\dag \Bigg) \, d \boldsymbol{x} \, d \boldsymbol{x}^\dag \,.
\end{eqnarray}
Thus,
\begin{eqnarray}
\frac{d}{d \epsilon} J(f^*+\epsilon \phi) &=& \frac{1}{K} \sum_k 2 \Big( \int_{\mathcal{X}} \big( f^*(\boldsymbol{x^\diamond}) + \epsilon \phi(\boldsymbol{x}^\diamond) \big) \delta(\boldsymbol{x}^\diamond-\boldsymbol{x}_k) \, d \boldsymbol{x}^\diamond - y_k \Big) \big( \int_{\mathcal{X}} \phi(\boldsymbol{x})  \delta(\boldsymbol{x}-\boldsymbol{x}_k) \, d \boldsymbol{x} \big)\\
&+& c \, \int_{\mathcal{X}}  \int_{\mathcal{X}} u(\boldsymbol{x},\boldsymbol{x}^\dag) \Big( \, \phi(\boldsymbol{x}) \, \big( f^*(\boldsymbol{x}^\dag) + \epsilon \phi(\boldsymbol{x}^\dag) \big) \,+\, \phi(\boldsymbol{x}^\dag) \, \big( f^*(\boldsymbol{x}) + \epsilon \phi(\boldsymbol{x}) \big) \, \Big) \, d \boldsymbol{x} \, d \boldsymbol{x}^\dag \,.
\end{eqnarray}
Setting $\epsilon=0$,
\begin{eqnarray}
\big( \frac{d}{d \epsilon} J(f^*+\epsilon \phi) \big)_{\epsilon=0} &=& \frac{1}{K} \sum_k 2 \big( \int_{\mathcal{X}} f^*(\boldsymbol{x}^\diamond)  \delta(\boldsymbol{x}^\diamond-\boldsymbol{x}_k) \, d \boldsymbol{x}^\diamond - y_k \big) \big( \int_{\mathcal{X}} \phi(\boldsymbol{x})  \delta(\boldsymbol{x}-\boldsymbol{x}_k) \, d \boldsymbol{x} \big)\\
&+& c \, \int_{\mathcal{X}}  \int_{\mathcal{X}} u(\boldsymbol{x},\boldsymbol{x}^\dag) \Big( \, \phi(\boldsymbol{x}) \, f^*(\boldsymbol{x}^\dag) \,+\, \phi(\boldsymbol{x}^\dag) \, f^*(\boldsymbol{x})\, \Big) \, d \boldsymbol{x} \, d \boldsymbol{x}^\dag \,.
\end{eqnarray}
By the symmetry of $u$,
\begin{eqnarray}
\big( \frac{d}{d \epsilon} J(f^*+\epsilon \phi) \big)_{\epsilon=0} &=& \frac{1}{K} \sum_k 2 \big( \int_{\mathcal{X}} f^*(\boldsymbol{x}^\diamond)  \delta(\boldsymbol{x}^\diamond-\boldsymbol{x}_k) \, d \boldsymbol{x}^\diamond - y_k \big) \big( \int_{\mathcal{X}} \phi(\boldsymbol{x})  \delta(\boldsymbol{x}-\boldsymbol{x}_k) \, d \boldsymbol{x} \big)\\
&+& 2 c \, \int_{\mathcal{X}}  \int_{\mathcal{X}} u(\boldsymbol{x},\boldsymbol{x}^\dag) \, \phi(\boldsymbol{x}) \, f^*(\boldsymbol{x}^\dag) \, d \boldsymbol{x} \, d \boldsymbol{x}^\dag \,.
\end{eqnarray}
Factoring out $\phi$,
\begin{eqnarray}
\big( \frac{d}{d \epsilon} J(f^*+\epsilon \phi) \big)_{\epsilon=0} \,=\, \int_{\mathcal{X}} 2 \phi(\boldsymbol{x}) \Big( & & \frac{1}{K} \sum_k  \delta(\boldsymbol{x}-\boldsymbol{x}_k) \, \big( \int_{\mathcal{X}} f^*(\boldsymbol{x}^\diamond)  \delta(\boldsymbol{x}^\diamond-\boldsymbol{x}_k) \, d \boldsymbol{x}^\diamond - y_k \big) \\
&+& c \, \int_{\mathcal{X}} u(\boldsymbol{x},\boldsymbol{x}^\dag)  \, f^*(\boldsymbol{x}^\dag) \, d \boldsymbol{x}^\dag \quad \Big) \, d \boldsymbol{x} \,.
\end{eqnarray}
In order for the above to be zero for $\forall \phi \in \mathcal{F}$, it is necessary that,
\begin{equation}
\frac{1}{K} \sum_k \delta(\boldsymbol{x}-\boldsymbol{x}_k) \, \big( \int_{\mathcal{X}} f^*(\boldsymbol{x}^\diamond)  \delta(\boldsymbol{x}^\diamond-\boldsymbol{x}_k) \, d \boldsymbol{x}^\diamond - y_k \big) \,+\, c \, \int_{\mathcal{X}} u(\boldsymbol{x},\boldsymbol{x}^\dag)  \, f^*(\boldsymbol{x}^\dag) \, d \boldsymbol{x}^\dag = 0 \,,
\end{equation}
which further simplifies to,
\begin{equation}
\label{eq:var_condition}
\frac{1}{K} \sum_k \delta(\boldsymbol{x}-\boldsymbol{x}_k) \, \big( f^*(\boldsymbol{x}_k)- y_k \big) \,+\, c \, \int_{\mathcal{X}} u(\boldsymbol{x},\boldsymbol{x}^\dag)  \, f^*(\boldsymbol{x}^\dag) \, d \boldsymbol{x}^\dag = 0 \,.
\end{equation}

We can equivalently express (\ref{eq:var_condition}) by the following system of equations,
\begin{equation}
\label{eq:eq_system}
\begin{cases}
\frac{1}{K} \sum_k \delta(\boldsymbol{x}-\boldsymbol{x}_k) \, r_k \,+\, c \, \int_{\mathcal{X}} u(\boldsymbol{x},\boldsymbol{x}^\dag)  \, f^*(\boldsymbol{x}^\dag) \, d \boldsymbol{x}^\dag = 0 \\
r_1 = f^*(\boldsymbol{x}_1)- y_1 \\
\vdots \\
r_K = f^*(\boldsymbol{x}_K)- y_K
\end{cases}\,.
\end{equation}
We first focus on solving the first equation in $f^*$,
\begin{equation}
\label{eq:f_star}
\frac{1}{K} \sum_k \delta(\boldsymbol{x}-\boldsymbol{x}_k) \, r_k \,+\, c \, \int_{\mathcal{X}} u(\boldsymbol{x},\boldsymbol{x}^\dag)  \, f^*(\boldsymbol{x}^\dag) \, d \boldsymbol{x}^\dag = 0 \,;
\end{equation}
later we can replace the resulted $f^*$ in other equations to obtain $r_k$'s. Let $g(\boldsymbol{x},\boldsymbol{t})$ be a function such that,
\begin{equation}
\label{eq:green_function}
\int_{\mathcal{X}} u(\boldsymbol{x},\boldsymbol{x}^\dag) \, g(\boldsymbol{x}^\dag,\boldsymbol{t}) \, d \boldsymbol{x}^\dag = \delta(\boldsymbol{x} - \boldsymbol{t}) \,.
\end{equation}
Such $g$ is called the \emc{Green's function} of the linear operator $L$ satisfying $[L f](\boldsymbol{x}) = \int_{\mathcal{X}} u(\boldsymbol{x},\boldsymbol{x}^\dag) \, f(\boldsymbol{x}^\dag) \, d \boldsymbol{x}^\dag$.
If we multiply both sides of (\ref{eq:green_function}) by $\frac{1}{K} \sum_k \delta(\boldsymbol{t} - \boldsymbol{x}_k) r_k $ and then integrate w.r.t. $\boldsymbol{t}$, we obtain,
\begin{eqnarray}
& & \int_{\mathcal{X}}\Big( \frac{1}{K} \sum_k r_k \delta(\boldsymbol{t} - \boldsymbol{x}_k)  \int_{\mathcal{X}} u(\boldsymbol{x},\boldsymbol{x}^\dag) \, g(\boldsymbol{x}^\dag,\boldsymbol{t}) \, d \boldsymbol{x}^\dag \Big) \, d \boldsymbol{t} \\
&=& \int_{\mathcal{X}}\Big( \frac{1}{K} \sum_k r_k \delta(\boldsymbol{t} - \boldsymbol{x}_k)  \delta(\boldsymbol{x} - \boldsymbol{t}) \Big) \, d \boldsymbol{t} \,.
\end{eqnarray}
Rearranging the left hand side leads to,
\begin{eqnarray}
& & \int_{\mathcal{X}} u(\boldsymbol{x},\boldsymbol{x}^\dag) \Big( \frac{1}{K} \sum_k \int_{\mathcal{X}} r_k \delta(\boldsymbol{t} - \boldsymbol{x}_k)  g(\boldsymbol{x}^\dag,\boldsymbol{t}) \, d \boldsymbol{t} \Big) \, d \boldsymbol{x}^\dag \\
&=& \int_{\mathcal{X}}\Big( \frac{1}{K} \sum_k r_k \delta(\boldsymbol{t} - \boldsymbol{x}_k)  \delta(\boldsymbol{x} - \boldsymbol{t}) \Big) \, d \boldsymbol{t} \,.
\end{eqnarray}
Using the sifting property of the delta function this simplifies to,
\begin{equation}
\int_{\mathcal{X}} u(\boldsymbol{x},\boldsymbol{x}^\dag) \big( \frac{1}{K} \sum_k r_k  g(\boldsymbol{x}^\dag,\boldsymbol{x}_k)\big) \, d \boldsymbol{x}^\dag = \frac{1}{K} \sum_k r_k \delta(\boldsymbol{x} - \boldsymbol{x}_k)  \,.
\end{equation}
We can now use the above identity to eliminate $\frac{1}{K} \sum_k r_k \delta(\boldsymbol{x} - \boldsymbol{x}_k)$ in (\ref{eq:f_star}) and thus obtain,
\begin{equation}
\int_{\mathcal{X}} u(\boldsymbol{x},\boldsymbol{x}^\dag) \big( \frac{1}{K} \sum_k r_k  g(\boldsymbol{x}^\dag,\boldsymbol{x}_k)\big) \, d \boldsymbol{x}^\dag \,+\, c \, \int_{\mathcal{X}} u(\boldsymbol{x},\boldsymbol{x}^\dag)  \, f^*(\boldsymbol{x}^\dag) \, d \boldsymbol{x}^\dag = 0 \,,
\end{equation}
or equivalently
\begin{equation}
\int_{\mathcal{X}} u(\boldsymbol{x},\boldsymbol{x}^\dag) \big( \frac{1}{K} \sum_k r_k  g(\boldsymbol{x}^\dag,\boldsymbol{x}_k) \,+\, c \,   \, f^*(\boldsymbol{x}^\dag) \big) \, d \boldsymbol{x}^\dag = 0 \,.
\end{equation}
A sufficient (and also necessary, as $u$ is assumed to have empty null space) for the above to hold is that,
\begin{equation}
\label{eq:f_star_suff_cond}
f^*(\boldsymbol{x}) = -\frac{1}{c K} \sum_k r_k  g(\boldsymbol{x},\boldsymbol{x}_k) \,.
\end{equation}
We can now eliminate $f^*$ in the system of equations (\ref{eq:eq_system}) and obtain a system that only depends on $r_k$'s,
\begin{equation}
\begin{cases}
r_1 = -\frac{1}{c K} \sum_k r_k  g(\boldsymbol{x}_1,\boldsymbol{x}_k) - y_1 \\
\vdots \\
r_K = -\frac{1}{c K} \sum_k r_k  g(\boldsymbol{x}_K,\boldsymbol{x}_k) - y_K
\end{cases} \,.
\end{equation}
This is a linear system in $r_k$ and can be expressed in vector/matrix form,
\begin{equation}
(c \boldsymbol{I} + \boldsymbol{G}) \boldsymbol{r} = - c \, \boldsymbol{y} \,. 
\end{equation}
Thus,
\begin{equation}
\boldsymbol{r} = - c \, (c \boldsymbol{I} + \boldsymbol{G})^{-1} \boldsymbol{y} \,,
\end{equation}
and finally using the definition of $f^*$ in (\ref{eq:f_star_suff_cond}) we obtain,
\begin{equation}
f^*(\boldsymbol{x}) \,=\, -\frac{1}{c}\, \boldsymbol{g}^T_{\boldsymbol{x}}\, \boldsymbol{r}  \,=\, \boldsymbol{g}^T_{\boldsymbol{x}} \, (c \boldsymbol{I} + \boldsymbol{G})^{-1} \boldsymbol{y}  \,.
\end{equation}

\newpage
\section{Equivalent Kernel Regression Problem}
\label{sec:equiv_kernel}
Given a positive definite kernel function $g(\,.\,,\,.\,)$. Recall that the solution of regularized kernel regression after $t$ rounds of self-distillation has the form,
\begin{equation}
\label{eq:opt_f_t}
f^*_t(\boldsymbol{x})=\boldsymbol{g}^T_{\boldsymbol{x}} \boldsymbol{G}^t \Pi_{i=0}^t (\boldsymbol{G}+ c_i \boldsymbol{I})^{-1} \boldsymbol{y}_0 \,.
\end{equation}
On the other hand, the solution to a standard kernel ridge regression on the same training data with a positive definite kernel $g^\dag$ has the form,
\begin{equation}
\label{eq:f_dag}
f^\dag(\boldsymbol{x})={\boldsymbol{g}^\dag_{\boldsymbol{x}}}^T (\boldsymbol{G}^\dag + c_0 \boldsymbol{I})^{-1} \boldsymbol{y}_0 \,,
\end{equation}
for which there are standard generalization bounds. We claim $f_t^*$ can be equivalently written in this standard form by a proper choice of $g^\dag$ (as a function of $g$). As a result of that, we show the spectrum of the Gram matrix $\boldsymbol{G}^\dag$ relates to that of $\boldsymbol{G}$ via,
\begin{equation}
\lambda_k^\dag =  c_0 \frac{1}{\frac{\Pi_{i=0}^t (\lambda_k+ c_i)}{\lambda_k^{t+1}}-1} \,.
\end{equation}

Our strategy for tackling this problem is inspired by the proof technique in Corollary~6.7 of~\cite{Bartlett2005}. Let $P$ be the data-dependent linear operator defined as,
\begin{equation}
[P h](\boldsymbol{x}) \triangleq \frac{1}{K} \sum_{k=1}^K h(\boldsymbol{x}_k) g(\boldsymbol{x},\boldsymbol{x}_k) \,.
\end{equation}
Let $\mathcal{H}$ denote the Reproducing Kernel Hilbert Space associated with $g$ and $\langle \,.\, , \,.\, \rangle_{\mathcal{H}}$ be the dot product in $\mathcal{H}$.
It is easy to verify that $P$ is a \emc{positive definite operator} in this space, i.e. it satisfies $\langle h \,,\, P h \rangle > 0$ for any $h \in \mathcal{H}$ due to,
\begin{eqnarray}
\langle h \,,\, P h \rangle_{\mathcal{H}} &=& \langle h \,,\, \frac{1}{K} \sum_{k=1}^K h(\boldsymbol{x}_k) g(.,\boldsymbol{x}_k) \rangle \\
&=& \frac{1}{K} \sum_{k=1}^K h(\boldsymbol{x}_k) \underbrace{\langle h \,,\, g(.,\boldsymbol{x}_k) \rangle}_{h(\boldsymbol{x}_k)} \\
&=& \frac{1}{K} \sum_{k=1}^K h^2(\boldsymbol{x}_k)  >0 \,,
\end{eqnarray}
where we used $\langle h \,,\, g(.,\boldsymbol{x}) \rangle = h(\boldsymbol{x})$ due to the reproducing property of $\mathcal{H}$.
Since $P$ is positive definite, there exist eigenfunctions $\phi_j$ and eigenvalues $\lambda_j \geq 0$ that satisfy $[P \phi_j](\boldsymbol{x}) = \lambda_j \phi_j(\boldsymbol{x})$. Plugging the definition of $P$ into this identity yields,
\begin{equation}
\label{eq:eq_0}
\frac{1}{K} \sum_{k=1}^K \phi_j(\boldsymbol{x}_k) g(\boldsymbol{x},\boldsymbol{x}_k) = \lambda_j \phi_j(\boldsymbol{x}) \,.
\end{equation}
In particular, evaluating the latter identity at the points $\boldsymbol{x} \in \cup_{p=1}^K \{ \boldsymbol{x}_p \}$ gives $\frac{1}{K} \sum_{k=1}^K \phi_j(\boldsymbol{x}_k) g(\boldsymbol{x}_p,\boldsymbol{x}_k) = \lambda_j \phi_j(\boldsymbol{x}_p)$ for $p=1,\dots,K$. Recalling that $\boldsymbol{G}$ is evaluation of $\frac{1}{K} g(\,.\,,\,.\,)$ at pairs of points across $\cup_{k=1}^K \{\boldsymbol{x}_k\}$, this identity be expressed equivalently as,
\begin{equation}
\label{eq:eq_1}
\boldsymbol{G} \boldsymbol{\phi}_j=\lambda_j \boldsymbol{\phi}_j \,.
\end{equation}
This implies $\boldsymbol{\phi}_j$ is an eigenvector of $\boldsymbol{G}$ with corresponding eigenvalue of $\lambda_j$ for any $j$ that $\boldsymbol{G} \boldsymbol{\phi}_j\neq 0$. Thus, by sorting $\boldsymbol{\phi}_j$ in non-increasing order of $\lambda_j$, and placing them for $j=1,\dots,K$ into the matrix $\boldsymbol{\Phi}$ and the diagonal matrix $\boldsymbol{\Lambda}$ respectively, we obtain,
\begin{equation}
\boldsymbol{\Phi}=\boldsymbol{V} \quad,\quad \boldsymbol{\Lambda}=\boldsymbol{D}\,.
\end{equation}
Since the eigenvectors of $\boldsymbol{G}^t \Pi_{i=0}^t (\boldsymbol{G}+ c_i \boldsymbol{I})^{-1}$ are the same as those of $\boldsymbol{G}$ (adding a multiple of $\boldsymbol{I}$ or applying matrix inversion do not change eigenvectors), and the eigenvectors of $\boldsymbol{G}$ as showed in (\ref{eq:eq_1}) are $\boldsymbol{\Phi}$, we can write,
\begin{equation}
\label{eq:ker_equiv}
\boldsymbol{G}^t \Pi_{i=0}^t (\boldsymbol{G}+ c_i \boldsymbol{I})^{-1} = \boldsymbol{\Phi}^T \boldsymbol{\Lambda}^t \Pi_{i=0}^t (\boldsymbol{\Lambda}+ c_i \boldsymbol{I})^{-1} \boldsymbol{\Phi} \,. 
\end{equation}
On the other hand, using the same vector notation and recalling that $\boldsymbol{g}$ is the evaluation of $\frac{1}{K} g(\,.\,,\,\boldsymbol{x}_k)$ at $k=1,\dots,K$, we can express (\ref{eq:eq_0}) as $\boldsymbol{\phi}_j^T \boldsymbol{g}_{\boldsymbol{x}}  = \lambda_j \phi_j(\boldsymbol{x})$. Expressing this simultaneously for $j=1,\dots,K$ yields $\boldsymbol{\Phi} \boldsymbol{g}_{\boldsymbol{x}}  = \boldsymbol{\Lambda} \boldsymbol{\phi}_{\boldsymbol{x}}$, or equivalently
\begin{equation}
\label{eq:vec_equiv}
\boldsymbol{g}_{\boldsymbol{x}}  = \boldsymbol{\Phi} ^T \boldsymbol{\Lambda} \boldsymbol{\phi}_{\boldsymbol{x}} \,,
\end{equation}
where $\boldsymbol{\phi}_{\boldsymbol{x}} \triangleq[\phi_1(\boldsymbol{x}) ,\dots, \phi_K(\boldsymbol{x})]$.
Plugging (\ref{eq:ker_equiv}) and (\ref{eq:vec_equiv}) with into (\ref{eq:opt_f_t}) gives,
\begin{eqnarray}
f^*_t(\boldsymbol{x})&=&\boldsymbol{g}^T_{\boldsymbol{x}} \boldsymbol{G}^t \Pi_{i=0}^t (\boldsymbol{G}+ c_i \boldsymbol{I})^{-1} \boldsymbol{y}_0 \\
&=& \boldsymbol{\phi}^T_{\boldsymbol{x}}\boldsymbol{\Lambda} \boldsymbol{\Phi} \boldsymbol{\Phi}^T \boldsymbol{\Lambda}^t \Pi_{i=0}^t (\boldsymbol{\Lambda}+ c_i \boldsymbol{I})^{-1} \boldsymbol{\Phi} \boldsymbol{y}_0 \\
&=& \boldsymbol{\phi}^T_{\boldsymbol{x}}  \boldsymbol{\Lambda}^{t+1} \Pi_{i=0}^t (\boldsymbol{\Lambda}+ c_i \boldsymbol{I})^{-1} \boldsymbol{\Phi} \boldsymbol{y}_0 \,.
\end{eqnarray}

Suppose $g^\dag$ is a positive definite kernel and let $[P^\dag h](x) \triangleq \frac{1}{K} \sum_{k=1}^K h(\boldsymbol{x}_k) g^\dag(\boldsymbol{x},\boldsymbol{x}_i)$. We assume the operator $P^\dag$ shares the same eigenfunction as those of $P$, but varies in its eigenvalues $\lambda_j^\dag \geq 0$, i.e. $[P^\dag \phi_j](\boldsymbol{x}) = \lambda_j^\dag \phi_j(\boldsymbol{x})$. Thus, by a similar argument, the solution of (\ref{eq:f_dag}) can be written as,
\begin{equation}
f^\dag(\boldsymbol{x})=\boldsymbol{\phi}^T_{\boldsymbol{x}}  \boldsymbol{\Lambda}^\dag (\boldsymbol{\Lambda}^\dag+ c_0 \boldsymbol{I})^{-1} \boldsymbol{\Phi} \boldsymbol{y}_0 \,,
\end{equation}
Thus in order to have $f^\dag=f^*_t$, it is sufficient to have,
\begin{equation}
\boldsymbol{\Lambda}^{t+1} \Pi_{i=0}^t (\boldsymbol{\Lambda}+ c_i \boldsymbol{I})^{-1} = \boldsymbol{\Lambda}^\dag (\boldsymbol{\Lambda}^\dag+ c_0 \boldsymbol{I})^{-1} \,.
\end{equation}
Since the matrices above are all diagonal, this can be expressed equivalently as,
\begin{equation}
\frac{\lambda_k^{t+1}}{\Pi_{i=0}^t (\lambda_k+ c_i)} = \frac{\lambda_k^\dag}{\lambda_k^\dag+ c_0} \,.
\end{equation}
Solving in $\lambda_k^\dag$ yields,
\begin{equation}
\lambda_k^\dag =  c_0 \frac{1}{\frac{\Pi_{i=0}^t (\lambda_k+ c_i)}{\lambda_k^{t+1}}-1} \,. 
\end{equation}
Note that this is a valid solution for $\lambda_k^\dag$, i.e. it satisfies the requirement $\lambda_k^\dag \geq 0$. This is because $\omega_k \triangleq \frac{\lambda_k^{t+1}}{\Pi_{i=0}^t (\lambda_k+ c_i)}$ always satisfies\footnote{This is due to the conditions $\lambda_k > 0$ (recall we assume $\boldsymbol{G}$ is full-rank) and $c_i>0$.} $0 < \omega_k <1$ and that the function $\lambda^\dag_k(\omega_k) \triangleq c_0 \frac{1}{\frac{1}{\omega_k}-1}$ is well-defined ($\omega_k \neq 0$) and is increasing when $0 < \omega_k <1$.

\newpage
\section{Proofs}
\begin{propositionappendix}
The variational problem (\ref{eq:f_star_in_c}) has a solution of the form,
\begin{equation}
f^*(\boldsymbol{x}) \,=\, \boldsymbol{g}_{\boldsymbol{x}}^T (c \boldsymbol{I} + \boldsymbol{G})^{-1} \boldsymbol{y} \,.
\end{equation}
\end{propositionappendix}

See Appendix \ref{sec:app_variational} for a proof.

\begin{propositionappendix}
The following identity holds,
\begin{equation}
\frac{1}{K} \sum_k \big( f^*(\boldsymbol{x}_k) - y_k \big)^2 \,=\, \frac{1}{K} \, \sum_k (z_k \, \frac{c}{c+d_k} )^2 \,.
\end{equation}
\end{propositionappendix}

\begin{proof}

\begin{eqnarray}
& & \frac{1}{K} \, (f^*(\boldsymbol{x}_k)-y_k)^2 \\
&=& \frac{1}{K} \, \big( \boldsymbol{g}_{\boldsymbol{x}_k}^T (c \boldsymbol{I} + \boldsymbol{G})^{-1} \boldsymbol{y} -y_k\big)^2 \\
&=& \frac{1}{K} \, \big\| \boldsymbol{G} (c \boldsymbol{I} + \boldsymbol{G})^{-1} \boldsymbol{y} -\boldsymbol{y} \big\|^2 \\
&=& \frac{1}{K} \, \big\| \boldsymbol{V}^T \boldsymbol{D} (c \boldsymbol{I} + \boldsymbol{D})^{-1} \boldsymbol{V} \boldsymbol{y} -\boldsymbol{y} \big\|^2 \,,
\end{eqnarray}
which after exploiting rotation invariance property of $\|.\|$ and the fact that the matrix of eigenvectors $\boldsymbol{V}$ is a rotation matrix, can be expressed as,
\begin{eqnarray}
& & \frac{1}{K} \, (f^*(\boldsymbol{x}_k)-y_k)^2 \\
&=& \frac{1}{K} \, \big\| \boldsymbol{V}^T \boldsymbol{D} (c \boldsymbol{I} + \boldsymbol{D})^{-1} \boldsymbol{V} \boldsymbol{y} -\boldsymbol{y} \big\|^2 \\
&=& \frac{1}{K} \, \big\| \boldsymbol{V} \boldsymbol{V}^T \boldsymbol{D} (c \boldsymbol{I} + \boldsymbol{D})^{-1} \boldsymbol{V} \boldsymbol{y} - \boldsymbol{V} \boldsymbol{y} \big\|^2 \\
&=& \frac{1}{K} \, \big\| \boldsymbol{D} (c \boldsymbol{I} + \boldsymbol{D})^{-1} \boldsymbol{z} -\boldsymbol{z} \big\|^2 \\
&=& \frac{1}{K} \, \Big\| \big( \boldsymbol{D} (c \boldsymbol{I} + \boldsymbol{D})^{-1}  - \boldsymbol{I} \big) \boldsymbol{z} \Big\|^2 \\
&=& \frac{1}{K} \, \sum_k (\frac{d_k}{c+d_k} - 1)^2 z_k^2 \\
&=& \frac{1}{K} \, \sum_k (z_k \, \frac{c}{c+d_k} )^2 \,,
\end{eqnarray}

\qed

\end{proof}

\begin{propositionappendix}
For any $t \geq0$, if $\| \boldsymbol{z}_i\| > \sqrt{K \, \epsilon}$ for $i=0,\dots,t$, then,
\begin{equation}
\| \boldsymbol{z}_t\| \geq a^t(\kappa) \| \boldsymbol{z}_0\| - \sqrt{K \, \epsilon} \, b(\kappa) \frac{a^t(\kappa) - 1}{a(\kappa) - 1} \,,
\end{equation}
where,
\begin{eqnarray}
a(x) &\triangleq& \frac{(r_0 - 1)^2 + x (2 r_0 - 1) }{(r_0 -1 + x)^2} \\
b(x) &\triangleq& \frac{r_0^2 x}{(r_0 -1 + x)^2 } \\
r_0 &\triangleq& \frac{1}{\sqrt{K \, \epsilon}} \,\|\boldsymbol{z}_0\| \quad,\quad
\kappa \triangleq \frac{d_{\max}}{d_{\min}} \,.
\end{eqnarray}
\end{propositionappendix}

\begin{proof}

We start from the identity we obtained in (\ref{eq:z_t_recur}). By diving both sides of it by $\sqrt{K\, \epsilon}$ we obtain,
\begin{equation}
\label{eq:z_t_recur_appendix}
\frac{1}{\sqrt{K\, \epsilon}} \, \boldsymbol{z}_t = \boldsymbol{D} (\frac{ \alpha_t \sqrt{K \, \epsilon}}{\| \boldsymbol{z}_{t-1} \| - \sqrt{K \, \epsilon}} \boldsymbol{I} + \boldsymbol{D})^{-1} \, \frac{1}{\sqrt{K\, \epsilon}} \,\boldsymbol{z}_{t-1} \,,
\end{equation}
where,
\begin{equation}
d_{\min} \leq \alpha_t \leq d_{\max} \,.
\end{equation}

Note that the matrix $\boldsymbol{D} (\frac{ \alpha_t \sqrt{K \, \epsilon}}{\| \boldsymbol{z}_{t-1} \| - \sqrt{K \, \epsilon}} \boldsymbol{I} + \boldsymbol{D})^{-1}$ in the above identitiy is {\em diagonal} and its $k$'th entry can be expressed as,
\begin{equation}
\big(\boldsymbol{D} (\frac{ \alpha_t \sqrt{K \, \epsilon}}{\| \boldsymbol{z}_{t-1} \| - \sqrt{K \, \epsilon}} \boldsymbol{I} + \boldsymbol{D})^{-1} \big) [k,k] = \frac{d_k}{\frac{ \alpha_t \sqrt{K \, \epsilon}}{\| \boldsymbol{z}_{t-1} \| - \sqrt{K \, \epsilon}} + d_k} = \frac{1}{\frac{ \frac{\alpha_t}{d_k} }{\frac{\| \boldsymbol{z}_{t-1} \|}{\sqrt{K \, \epsilon}} - 1} + 1} \,.
\end{equation}
Thus, as long as $\| \boldsymbol{z}_{t-1}\| > \sqrt{K \epsilon}$ we can get the following upper and lower bounds,
\begin{equation}
\frac{1}{\frac{ \frac{d_{\max}}{d_{\min}} }{\frac{\| \boldsymbol{z}_{t-1} \|}{\sqrt{K \, \epsilon}} - 1} + 1} \leq \big(\boldsymbol{D} (\frac{ \alpha_t \sqrt{K \, \epsilon}}{\| \boldsymbol{z}_{t-1} \| - \sqrt{K \, \epsilon}} \boldsymbol{I} + \boldsymbol{D})^{-1} \big) [k,k] \leq \frac{1}{\frac{ \frac{d_{\min}}{d_{\max}} }{\frac{\| \boldsymbol{z}_{t-1} \|}{\sqrt{K \, \epsilon}} - 1} + 1} \,.
\end{equation}
Putting the above fact beside recurrence relation of $\boldsymbol{z}_t$ in (\ref{eq:z_t_recur_appendix}), we can bound $\frac{1}{\sqrt{K \, \epsilon}}\|\boldsymbol{z}_t\|$ as,
\begin{equation}
\frac{1}{\frac{ \kappa }{r_{t-1} - 1} + 1} r_{t-1} \leq r_t \leq \frac{1}{\frac{ \frac{1}{\kappa} }{r_{t-1} - 1} + 1} r_{t-1} \,,
\end{equation}
where we used short hand notation,
\begin{eqnarray}
\kappa &\triangleq& \frac{d_{\max}}{d_{\min}} \\
r_t &\triangleq& \frac{1}{\sqrt{K \, \epsilon}} \| \boldsymbol{z}_t\| \,.
\end{eqnarray}
Note that $\kappa$ is the \emc{condition number} of the matrix $\boldsymbol{G}$ and by definition satisfies $\kappa \geq 1$. To further simplify the bounds, we use the inequality\footnote{This follows from concavity of $\frac{x}{\frac{ \frac{1}{\kappa} }{x - 1} + 1}$ in $x$ as long as $x - 1 \geq 0$ (can be verified by observing that the second derivative of the function is negative when $x - 1 \geq 0$ because $\kappa >1$ by definition). For any function $f(x)$ that is concave on the interval $[\underline{x}, \overline{x}]$, any line tangent to $f$ forms an {\em upper} bound on $f(x)$ over $[\underline{x}, \overline{x}]$. In particular, we use the tangent at the end point $\overline{x}$ to construct our bound. In our setting, this point which happens to be $r_0$. The latter is because $r_t$ is a decreasing sequence (see beginning of Section \ref{sec:guaranteed_iters}) and thus its largest values is at $t=0$.},
\begin{eqnarray}
& & \frac{1}{\frac{ \frac{1}{\kappa} }{r_{t-1} - 1} + 1} r_{t-1} \leq r_{t-1} \,  \frac{(r_0 - 1)^2 + \frac{1}{\kappa} (2 r_0 - 1) }{(r_0 -1 + \frac{1}{\kappa})^2} \,-\, \frac{r_0^2 \frac{1}{\kappa} }{ (r_0 -1 + \frac{1}{\kappa})^2} \,,
\end{eqnarray}
and\footnote{Similar to the earlier footnote, this follows from convexity of $\frac{x}{\frac{\kappa }{x - 1} + 1}$ in $x$ as long as $x - 1 \geq 0$ since $\kappa >1$ by definition. For any function $f(x)$ that is convex on the interval $[\underline{x}, \overline{x}]$, any line tangent to $f$ forms an {\em lower} bound on $f(x)$ over $[\underline{x}, \overline{x}]$. In particular, we use the tangent at the end point $\overline{x}$ to construct our bound, which as the earlier footnote, translate into $r_0$.},
\begin{eqnarray}
& & \frac{1}{\frac{ \kappa }{r_{t-1} - 1} + 1} r_{t-1} \geq r_{t-1} \, \frac{(r_0 - 1)^2 + \kappa (2 r_0 - 1) }{(r_0 -1 + \kappa)^2} \,-\, \frac{r_0^2 \kappa}{(r_0 -1 + \kappa)^2} \,.
\end{eqnarray}

For brevity, we introduce,
\begin{eqnarray}
a(x) &\triangleq& \frac{(r_0 - 1)^2 + x (2 r_0 - 1) }{(r_0 -1 + x)^2} \\
b(x) &\triangleq& \frac{r_0^2 x}{(r_0 -1 + x)^2 } \,.
\end{eqnarray}
Therefore, the bounds can be expressed more concisely as,
\begin{equation}
a(\kappa) \, r_{t-1} - b(\kappa) \quad\leq\quad r_t \quad\leq\quad a(\frac{1}{\kappa}) \, r_{t-1} - b(\frac{1}{\kappa}) \,.
\end{equation}
Now since both $r_{t-1} \triangleq \frac{1}{\sqrt{K \, \epsilon}}\|\boldsymbol{z}_{t-1}\|$ and $a(\kappa)$ or $a(\frac{1}{\kappa})$ are non-negative, we can solve the recurrence\footnote{More compactly, the problem can be stated as $\alpha^\dag r_{t-1} - b \leq r_t \leq \alpha r_{t-1} - b$, where $\alpha>0$ and $\alpha^\dag>0$. Let's focus on $r_t \leq \alpha r_{t-1} - b$, as the other case follows by similar argument. Start from the base case $r_1 \leq \alpha r_0 - b$. Since $\alpha>0$, we can multiply both sides by that and then add $-b$ to both sides: $\alpha r_1 -b \leq \alpha^2 r_0 - b( \alpha +1)$. On the other hand, looking at the recurrence $r_t \leq \alpha r_{t-1} - b$ at $t=2$ yields $r_2 \leq \alpha r_{1} - b$. Combining the two inequalities gives $r_2 \leq \alpha^2 r_0 - b( \alpha +1)$. By repeating this argument we obtain the general case $r_t \leq \alpha^t r_0 - b( \sum_{j=0}^{t-1} \alpha^j)$.} and obtain,
\begin{eqnarray}
\label{eq:r_t_bounds}
a^t(\kappa) r_0 - b(\kappa) \frac{a^t(\kappa) - 1}{a(\kappa) - 1} \quad\leq\quad r_t \quad\leq\quad a^t(\frac{1}{\kappa}) r_0 - b(\frac{1}{\kappa}) \frac{a^t(\frac{1}{\kappa}) - 1}{a(\frac{1}{\kappa}) - 1} \,.
\end{eqnarray}

\qed

\end{proof}

\begin{propositionappendix}
Starting from $\|\boldsymbol{y}_0\| > \sqrt{K \, \epsilon}$, meaningful (non-collapsing solution) self-distillation is possible at least for $\underline{t}$ rounds,
\begin{equation}
\underline{t} \triangleq \frac{\frac{\| \boldsymbol{y}_0\|}{\sqrt{K \, \epsilon}} -1}{\kappa} \,.
\end{equation}
\end{propositionappendix}

\begin{proof}

Recall that the assumption $\| \boldsymbol{z}_t \| > \sqrt{K \, \epsilon}$ translates into $r_t > 1$. We now obtain a sufficient condition for $r_t > 1$ by requiring a lower bound on $r_t$ to be greater than one. For that purpose, we utilize the lower bound we established in (\ref{eq:r_t_bounds}),
\begin{equation}
\underline{r_t} \triangleq a^t(\kappa) r_0 - b(\kappa) \frac{a^t(\kappa) - 1}{a(\kappa) - 1} \,.
\end{equation}

Setting the above to value $1$ implies,
\begin{equation}
\underline{r_t} = 1 \quad\Rightarrow\quad t = \frac{\log\big( \frac{1 - a(\kappa) + b(\kappa)}{b(\kappa) + r_0 (1 - a(\kappa) )} \big)}{\log\big(a(\kappa)\big)} =  \frac{\log\big( \frac{1 + \frac{\kappa - 1}{r_0^2}}{ 1 + \frac{\kappa - 1}{r_0}} \big)}{\log\big(1 - \frac{ (\frac{\kappa-1}{r_0}+\frac{1}{r_0}) (\frac{\kappa-1}{r_0} )}{(1 + \frac{\kappa-1}{r_0} )^2}\big)} \,.
\end{equation}
Observe that,
\begin{equation}
\frac{\log\big( \frac{1 + \frac{\kappa - 1}{r_0^2}}{ 1 + \frac{\kappa - 1}{r_0}} \big)}{\log\big(1 - \frac{ (\frac{\kappa-1}{r_0}+\frac{1}{r_0}) (\frac{\kappa-1}{r_0} )}{(1 + \frac{\kappa-1}{r_0} )^2}\big)} \geq \frac{r_0-1}{\kappa} \,,
\end{equation}
Thus,
\begin{equation}
t \geq \frac{r_0-1}{\kappa} = \frac{\frac{\| \boldsymbol{z}_0\|}{\sqrt{K \, \epsilon}} -1}{\kappa} = \frac{\frac{\| \boldsymbol{z}_0\|}{\sqrt{K \, \epsilon}} -1}{\kappa} =  \frac{\frac{\| \boldsymbol{y}_0\|}{\sqrt{K \, \epsilon}} -1}{\kappa}\,.
\end{equation}

\qed

\end{proof}

\begin{theoremappendix}
Suppose $\|\boldsymbol{y}_0\| > \sqrt{K \, \epsilon}$ and $t \leq \frac{\| \boldsymbol{y}_0\|}{\kappa \, \sqrt{K \, \epsilon}} - \frac{1}{\kappa}$. Then for any pair of diagonals of $\boldsymbol{D}$, namely $d_j$ and $d_k$, with the condition that $d_k > d_j$, the following inequality holds.
\begin{equation}
\frac{\boldsymbol{B}_{t-1} [k,k]}{\boldsymbol{B}_{t-1} [j,j]} \geq \left( \frac{\frac{\| \boldsymbol{y}_0 \|}{\sqrt{K \, \epsilon}} - 1 +  \frac{d_{\min}}{d_j} }{\frac{\| \boldsymbol{y}_0 \|}{\sqrt{K \, \epsilon}} - 1 +  \frac{d_{\min}}{d_k} }\right)^{t} \,.
\end{equation}
\end{theoremappendix}

\begin{proof}
We start with the definition of $\boldsymbol{A}_t$ from (\ref{eq:def_A}) and proceed as,
\begin{equation}
\frac{\boldsymbol{A}_t [k,k]}{\boldsymbol{A}_t [j,j]} = \frac{1+\frac{c_t}{d_j}}{1+\frac{c_t}{d_k}} \,.
\end{equation}

Since the derivative of the r.h.s. above w.r.t. $c_t$ is non-negative as long as $d_k \geq d_j$, it is non-decreasing in $c_t$. Therefore, we can get a lower bound on r.h.s. using a lower bound on $c_t$ (denoted by $\underline{c_t}$),
\begin{equation}
\label{eq:c_t_LB}
\frac{\boldsymbol{A}_t [k,k]}{\boldsymbol{A}_t [j,j]} \geq \frac{1+\frac{\underline{c_t}}{d_j}}{1+\frac{\underline{c_t}}{d_k}} \,.
\end{equation}
Also, since the assumption $t \leq \frac{\| \boldsymbol{y}_0\|}{\kappa \, \sqrt{K \, \epsilon}} - \frac{1}{\kappa}$ guarantees non-collapse conditions $c_t > 0$ and $\| \boldsymbol{z}_t\| > \sqrt{ K \, \epsilon}$, we can apply (\ref{eq:root_bounds_t}) and have the following lower bound on $c_t$
\begin{equation}
\label{eq:c_t_bound}
c_t \geq \frac{ d_{\min} \sqrt{K \, \epsilon}}{\| \boldsymbol{z}_t \| - \sqrt{K \, \epsilon}} \,.
\end{equation}
Since the r.h.s. (\ref{eq:c_t_bound}) is decreasing in $ \|\boldsymbol{z}_t\|$, the smallest value for the r.h.s. is attained by the largest value of $ \|\boldsymbol{z}_t\|$. However, as $ \|\boldsymbol{z}_t\|$ is decreasing in $t$ (see beginning of Section \ref{sec:guaranteed_iters}), its largest value is attained at $t=0$. Putting these together we obtain,
\begin{equation}
\label{eq:c_t_z_0}
c_t \geq \frac{ d_{\min} \sqrt{K \, \epsilon}}{\| \boldsymbol{z}_0 \| - \sqrt{K \, \epsilon}} \,.
\end{equation}
Using the r.h.s. of the above as $\underline{c_t}$ and applying it to (\ref{eq:c_t_LB}) yields,
\begin{equation}
\frac{\boldsymbol{A}_t [k,k]}{\boldsymbol{A}_t [j,j]} \geq \frac{\frac{\| \boldsymbol{z}_0 \|}{\sqrt{K \, \epsilon}} - 1 +  \frac{d_{\min}}{d_j} }{\frac{\| \boldsymbol{z}_0 \|}{\sqrt{K \, \epsilon}} - 1 +  \frac{d_{\min}}{d_k} } \,.
\end{equation}
Notice that both sides of the inequality are positive; $\boldsymbol{A}_t$ based on its definition in (\ref{eq:def_A}) and r.h.s. by the fact that $\|\boldsymbol{z}_0\| \geq \sqrt{K \, \epsilon}$. Therefore, we can instantiate the above inequality at each distillation step $i$, for $i=0,\dots,t-1$, and multiply them to obtain,
\begin{equation}
\Pi_{i=0}^{t-1} \frac{\boldsymbol{A}_i [k,k]}{\boldsymbol{A}_i [j,j]} \geq \big( \frac{\frac{\| \boldsymbol{z}_0 \|}{\sqrt{K \, \epsilon}} - 1 +  \frac{d_{\min}}{d_j} }{\frac{\| \boldsymbol{z}_0 \|}{\sqrt{K \, \epsilon}} - 1 +  \frac{d_{\min}}{d_k} }\big)^{t} \,.
\end{equation}
or equivalently,
\begin{equation}
\frac{\boldsymbol{B}_{t-1} [k,k]}{\boldsymbol{B}_{t-1} [j,j]} \geq \big( \frac{\frac{\| \boldsymbol{z}_0 \|}{\sqrt{K \, \epsilon}} - 1 +  \frac{d_{\min}}{d_j} }{\frac{\| \boldsymbol{z}_0 \|}{\sqrt{K \, \epsilon}} - 1 +  \frac{d_{\min}}{d_k} }\big)^{t} \,.
\end{equation}

\qed

\end{proof}

\begin{theoremappendix}
Suppose $\|\boldsymbol{y}_0\| > \sqrt{K \, \epsilon}$. Then the sparsity index $S_{\boldsymbol{B}_{\underline{t}-1}}$ (where $\underline{t} = \frac{\| \boldsymbol{y}_0\|}{\kappa \, \sqrt{K \, \epsilon}} - \frac{1}{\kappa}$ is number of guaranteed self-distillation steps before solution collapse) ``decreases'' in $\epsilon$, i.e. lower $\epsilon$ yields higher sparsity.

Furthermore at the limit $\epsilon \rightarrow 0$, the sparsity index has the form,
\begin{equation}
\lim_{\epsilon \rightarrow 0} S_{\boldsymbol{B}_{\underline{t}-1}} = e^{\frac{d_{\min}}{\kappa} \, \min_{k \in \{1,2,\dots,K-1\}}(\frac{1}{d_k} - \frac{1}{d_{k+1}})} \,.
\end{equation}
\end{theoremappendix}

\begin{proof}
We first show that the sparsity index is decreasing in $\epsilon$. We start from the definition of the sparsity index $S_{\boldsymbol{B}_{\underline{t}-1}}$ in (\ref{eq:def_S_b_under_t}) which we repeat below,
\begin{equation}
S_{\boldsymbol{B}_{\underline{t}-1}} = \min_{k \in \{1,2,\dots,K-1\}} \left( \frac{\frac{\| \boldsymbol{y}_0 \|}{\sqrt{K \, \epsilon}} - 1 +  \frac{d_{\min}}{d_k} }{\frac{\| \boldsymbol{y}_0 \|}{\sqrt{K \, \epsilon}} - 1 +  \frac{d_{\min}}{d_{k+1}} }\right)^{\frac{\| \boldsymbol{y}_0\|}{\kappa \, \sqrt{K \, \epsilon}} - \frac{1}{\kappa}} \,.
\end{equation}
For brevity, we define base and exponent as,
\begin{eqnarray}
b &\triangleq& \frac{m +  \frac{d_{\min}}{d_k} }{m +  \frac{d_{\min}}{d_{k+1}} } \\
p &\triangleq& \frac{m}{\kappa} \\
m &\triangleq& \frac{\| \boldsymbol{y}_0 \|}{\sqrt{K \, \epsilon}} -1 \,,
\end{eqnarray}
so that,
\begin{equation}
S_{\boldsymbol{B}_{\underline{t}-1}}(\epsilon) \, = \, b^p \,.
\end{equation}
The derivative is thus,
\begin{eqnarray}
& & \frac{d}{d \epsilon} S_{\boldsymbol{B}_{\underline{t}-1}} \\ &=& \frac{d \, S_{\boldsymbol{B}_{\underline{t}-1}}}{d m} \, \frac{d m} {d \epsilon} \\
&=& \Big( b^p \big( \frac{p \, b_m}{b} + p_m \, \log(b) \big) \Big)  \,\, \Big( \frac{d m} {d \epsilon} \Big)\\
&=& b^p \big( \frac{p \, b_m}{b} + p_m \, \log(b) \big) \,\, \big( - \frac{1} {2 \epsilon} (m+1)\big)\\
&=& b^p \big( \frac{p}{m +  \frac{d_{\min}}{d_k} } \, - \frac{p}{m + \frac{d_{\min}}{a_{k+1}}}\, + \frac{1}{\kappa} \, \log(b) \big)  \,\, \big( - \frac{1} {2 \epsilon} (m+1)\big) \\
&=& \frac{b^p}{\kappa} \, \big( \frac{m}{m +  \frac{d_{\min}}{d_k} } \, - \frac{m}{m + \frac{d_{\min}}{a_{k+1}}}\, +  \, \log(b) \big)  \,\, \big( - \frac{1} {2 \epsilon} (m+1)\big) \\
&=& \frac{b^p}{\kappa} \, \big( \frac{1}{1 +  \frac{d_{\min}}{m \, d_k} } \, - \frac{1}{1 + \frac{d_{\min}}{m \, a_{k+1}}}\, +  \, \log(b) \big)  \,\, \big( - \frac{1} {2 \epsilon} (m+1)\big) \\
&=& \frac{b^p}{\kappa} \, \big( \frac{1}{1 +  \frac{d_{\min}}{m \, d_k} } \, - \frac{1}{1 + \frac{d_{\min}}{m \, a_{k+1}}}\, +  \, \log(\frac{1 +  \frac{d_{\min}}{m \, d_k} }{1 +  \frac{d_{\min}}{m \, d_{k+1}} }) \big)  \,\, \big( - \frac{1} {2 \epsilon} (m+1)\big) \\
&=& \frac{b^p}{\kappa} \, \big( \frac{1}{1 +  \frac{d_{\min}}{m \, d_k} } + \log(1 +  \frac{d_{\min}}{m \, d_k} ) \, - \,\frac{1}{1 + \frac{d_{\min}}{m \, a_{k+1}}} - \log(1 +  \frac{d_{\min}}{m \, d_{k+1}} ) \big)  \,\, \big( - \frac{1} {2 \epsilon} (m+1)\big)  \,.
\end{eqnarray}
We now focus on the first parentheses. Define the function $e(x) \triangleq \frac{1}{x} + \log(x)$. Thus we can write the contents in the first parentheses more compactly,
\begin{eqnarray}
& & \frac{1}{1 +  \frac{d_{\min}}{m \, d_k} } + \log(1 +  \frac{d_{\min}}{m \, d_k} ) \, - \,\frac{1}{1 + \frac{d_{\min}}{m \, a_{k+1}}} - \log(1 +  \frac{d_{\min}}{m \, d_{k+1}} )\\
\label{eq:e_args}
&=& e(1 +  \frac{d_{\min}}{m \, d_k}) - e(1 +  \frac{d_{\min}}{m \, d_{k+1}}) \,.
\end{eqnarray}
However, $e^\prime(x)=\frac{x-1}{x^2}$, thus when $x>1$ the function $e^\prime(x)$ is positive. Consequently, when $x>1$ $e(x)$ is increasing. In fact, since both $\frac{d_{\min}}{m \, d_k}$ and $\frac{d_{\min}}{m \, d_k}$ are positive, the arguments of $e$ satsify the condition of being greater that $1$ and thus $e$ is increasing. On the other hand, since $d_{k+1} > d_k$ it follows that $1 + \frac{d_{\min}}{m \, d_k} > 1 +  \frac{d_{\min}}{m \, d_{k+1}}$, and thus by leveraging the fact that $e$ is increasing we obtain $e(1+ \frac{d_{\min}}{m \, d_k}) > e(1 +  \frac{d_{\min}}{m \, d_{k+1}})$. Finally by plugging the definition of $e$ we obtain,
\begin{equation}
\frac{1}{1 +  \frac{d_{\min}}{m \, d_k} } + \log(1 +  \frac{d_{\min}}{m \, d_k} ) \, > \,\frac{1}{1 + \frac{d_{\min}}{m \, a_{k+1}}} + \log(1 +  \frac{d_{\min}}{m \, d_{k+1}} ) \,.
\end{equation}
It is now easy to determine the sign of $\frac{d}{d \epsilon} S$ as shown below,
\begin{eqnarray}
& & \frac{d}{d \epsilon} S_{\boldsymbol{B}_{\underline{t}-1}}\\
&=& \underbrace{\frac{b^p}{\kappa}}_{\color{red}\textrm{positive}} \, \big(\underbrace{ \frac{1}{1 +  \frac{d_{\min}}{m \, d_k} } + \log(1 +  \frac{d_{\min}}{m \, d_k} ) \, - \,\frac{1}{1 + \frac{d_{\min}}{m \, a_{k+1}}} - \log(1 +  \frac{d_{\min}}{m \, d_{k+1}} ) }_{\color{red}\textrm{positive}} \big)  \,\, \big( \underbrace{ - \frac{1} {2 \epsilon} (m+1)}_{\color{red}\textrm{negative}} \big)  \,.
\end{eqnarray}
By showing that $\frac{d}{d \epsilon} S_{\boldsymbol{B}_{\underline{t}-1}} < 0$ we just proved $ S_{\boldsymbol{B}_{\underline{t}-1}}$ is decreasing in $\epsilon$.

We now focus on the limit case $\epsilon \rightarrow 0$. First note due to the identity $m=\frac{\| \boldsymbol{y}_0 \|}{\sqrt{K \, \epsilon}} -1$ we have the following identity,
\begin{eqnarray}
& & \lim_{\epsilon \rightarrow 0 } \min_{k \in \{1,2,\dots,K-1\}} \left( \frac{\frac{\| \boldsymbol{y}_0 \|}{\sqrt{K \, \epsilon}} - 1 +  \frac{d_{\min}}{d_k} }{\frac{\| \boldsymbol{y}_0 \|}{\sqrt{K \, \epsilon}} - 1 +  \frac{d_{\min}}{d_{k+1}} }\right)^{\frac{\| \boldsymbol{y}_0\|}{\kappa \, \sqrt{K \, \epsilon}} - \frac{1}{\kappa}} \\
&=& \lim_{m \rightarrow \infty} \min_{k \in \{1,2,\dots,K-1\}} \left( \frac{m +  \frac{d_{\min}}{d_k} }{m +  \frac{d_{\min}}{d_{k+1}} }\right)^{\frac{1}{\kappa} m} \,.
\end{eqnarray}
Further, since pointwise minimum of continuous functions is also a continuous function, we can move the limit inside the minimum,
\begin{eqnarray}
& & \lim_{m \rightarrow \infty} \min_{k \in \{1,2,\dots,K-1\}} \left( \frac{m +  \frac{d_{\min}}{d_k} }{m +  \frac{d_{\min}}{d_{k+1}} }\right)^{\frac{1}{\kappa} m} \\
&=& \min_{k \in \{1,2,\dots,K-1\}} \lim_{m \rightarrow \infty} \left( \frac{m + \frac{d_{\min}}{d_k} }{m +  \frac{d_{\min}}{d_{k+1}} }\right)^{\frac{1}{\kappa} m} \\
\label{eq:limit_exp}
&=& \min_{k \in \{1,2,\dots,K-1\}} e^{\frac{\frac{d_{\min}}{d_k} - \frac{d_{\min}}{d_{k+1}}}{\kappa}} \\
&=& \min_{k \in \{1,2,\dots,K-1\}} e^{\frac{d_{\min}}{\kappa} (\frac{1}{d_k} - \frac{1}{d_{k+1}})} \\
\label{eq:e_monotone}
&=& e^{\frac{d_{\min}}{\kappa} \, \min_{k \in \{1,2,\dots,K-1\}}(\frac{1}{d_k} - \frac{1}{d_{k+1}})} \,,
\end{eqnarray}
where in (\ref{eq:limit_exp}) we used the identity $\lim_{x \rightarrow \infty} {f(x)}^{g(x)} = e^{\lim_{x \rightarrow \infty} \big(f(x)-1\big)\,\big(g(x)\big) }$ and in (\ref{eq:e_monotone}) we used the fact that $e^{ \frac{d_{\min}}{\kappa} x}$ is monotonically increasing in $x$ (because $\frac{d_{\min}}{\kappa} > 0$).

\qed

\end{proof}

\pagebreak
\section{More on Experiments}
\label{sec:app-experiments}
\subsection{Setup Details}
We used Adam optimizer with learning rates of $0.001$ and $0.0001$ for CIFAR-10 and CIFAR-100, respectively. They are trained up to 64000 steps with batch size equal to 16 and 64 for CIFAR-10 and CIFAR-100, respectively. 
In all the experiments, we slightly regularize the training by weight decay regularization added to the fitting loss with its coefficient set to $0.0001$ and $0.00005$ for CIFAR-10 and CIFAR-100, respectively. 
Training and test is performed on the standard (50000 train-10000 test) split of the CIFAR dataset.
Most of the experiments are conducted using Resnet-50~\cite{He2015DeepRL} and CIFAR-10 and CIFAR-100 datasets~\cite{Krizhevsky2009LearningML}. However, we briefly validate our results on VGG-16~\cite{Simonyan15} too.

\subsection{$\ell_2$ Loss on Neural Network Predictions}
Figure \ref{fig:app-L2-pred-resnet50-CIFAR-10} shows the full results on CIFAR-10 and Resnet-50. The train and test accuracies have already been discussed in the main paper and are copied here to facilitate comparison.
However, in this subsection, we demonstrated the loss of the trained model at all steps with respect to the original ground truth data too. This may help establish an intuition on how self-distillation is regularizing the training on the original data.
Looking at the train loss we can see it first drops as the regularization is amplified and then increases while the model under-fits. This, again, suggests that the mechanism that self-distillation employs for regularization is different from early stopping. For CIFAR-100 the results in Figure~\ref{fig:app-L2-pred-resnet50-CIFAR-100} show a similar trend.

\begin{figure}[h]
    \centering
    \begin{tabular}{c c}
            \includegraphics[width=0.23\textwidth]{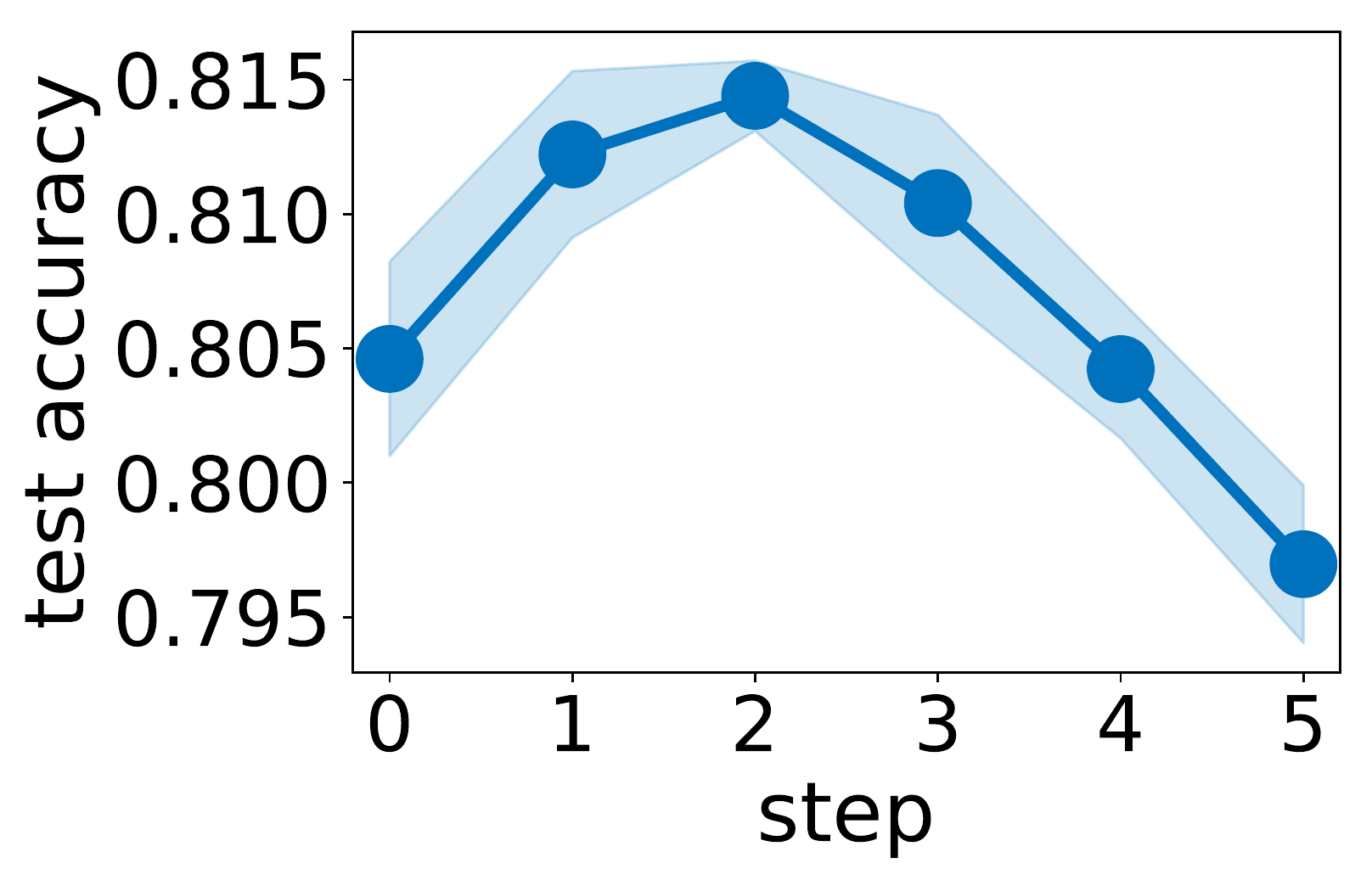} & 
            \includegraphics[width=0.23\textwidth]{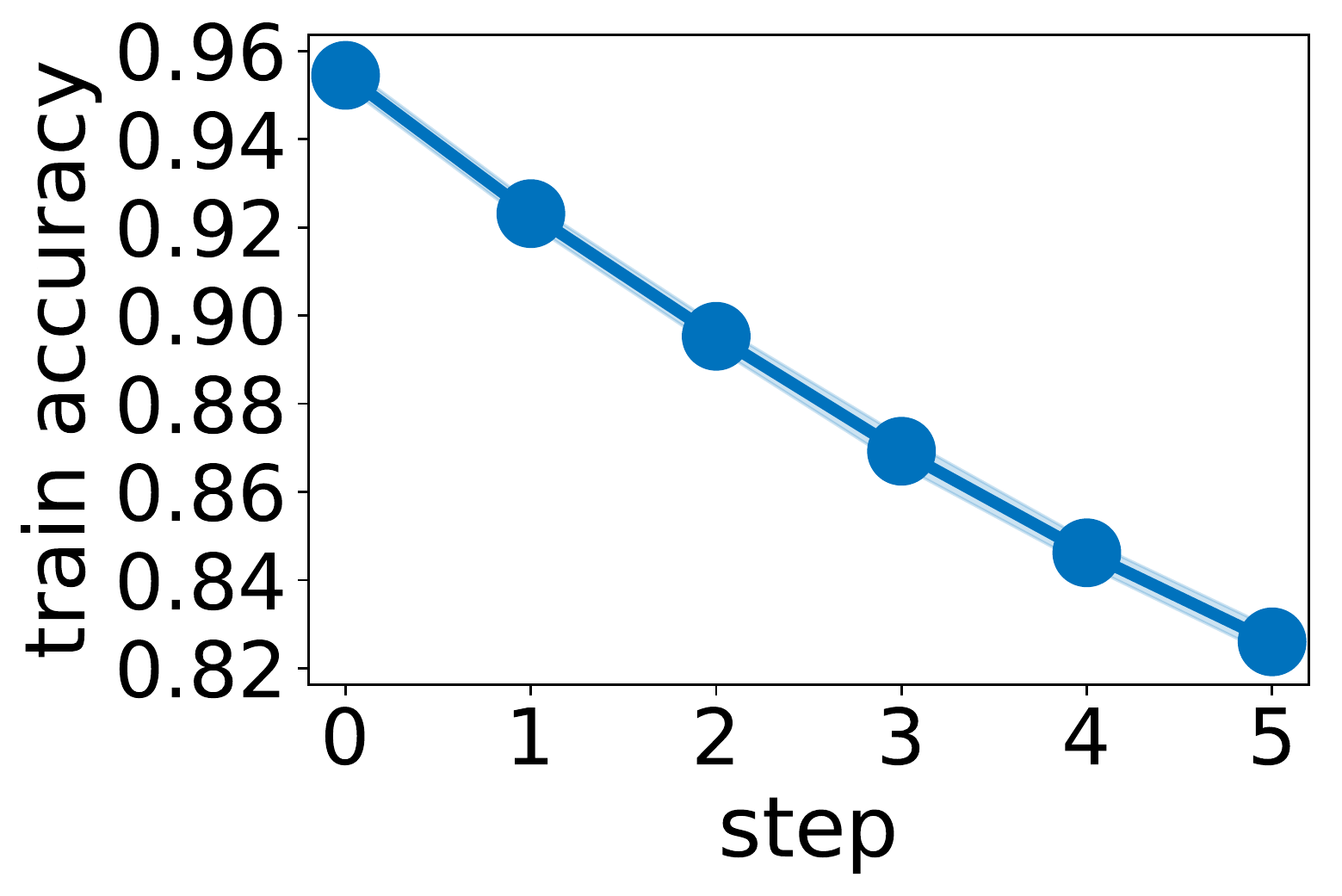} \\
            \includegraphics[width=0.23\textwidth]{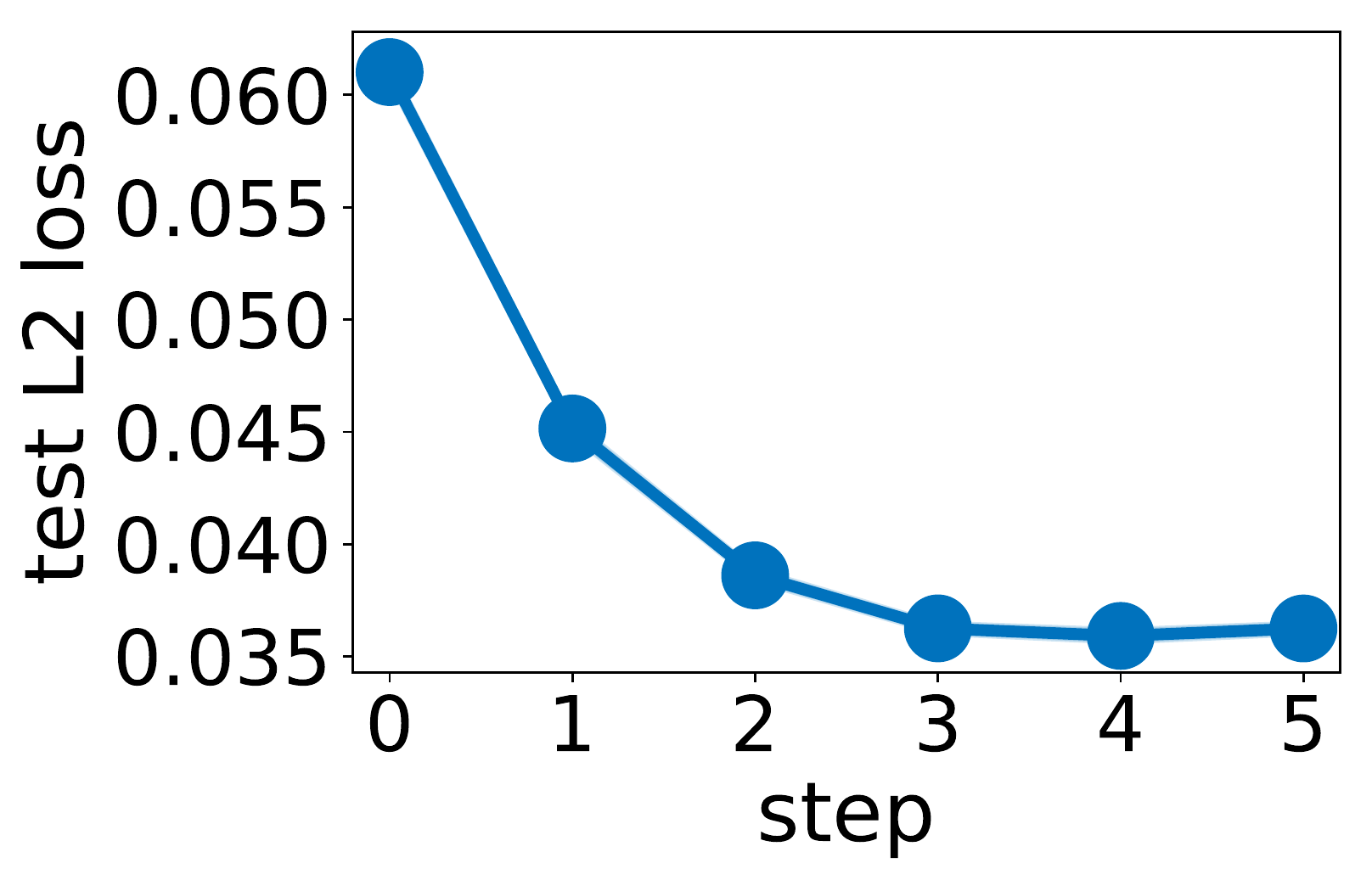} &
               \includegraphics[width=0.23\textwidth]{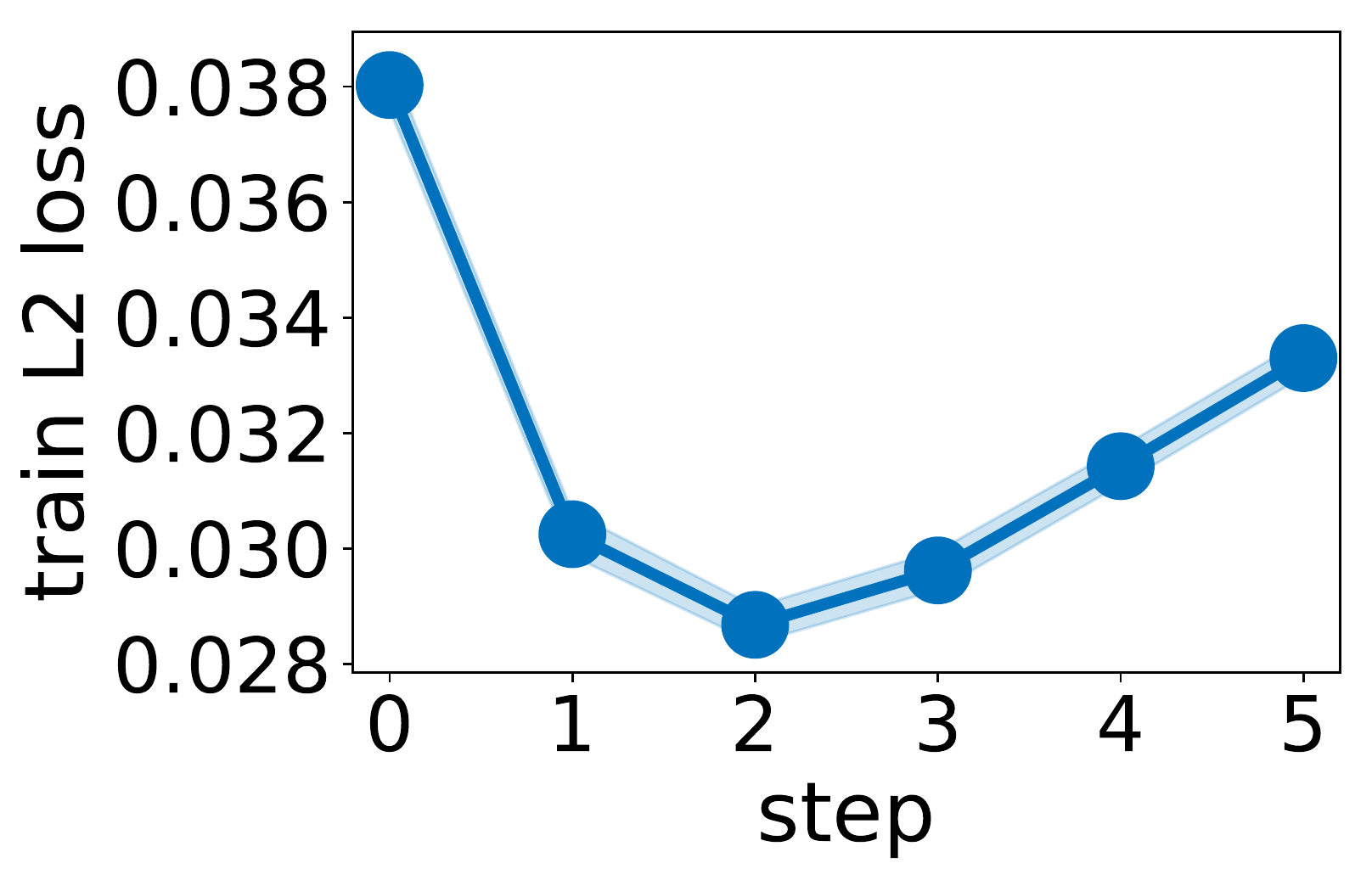}\\
    \end{tabular}
    \caption{Self-Distillation results with $\ell_2$ loss of neural network predictions for Resnet-50 and CIFAR-10 }
    \label{fig:app-L2-pred-resnet50-CIFAR-10}
\end{figure}

\begin{figure}[h]
    \centering
    \begin{tabular}{c c}
            \includegraphics[width=0.3\textwidth]{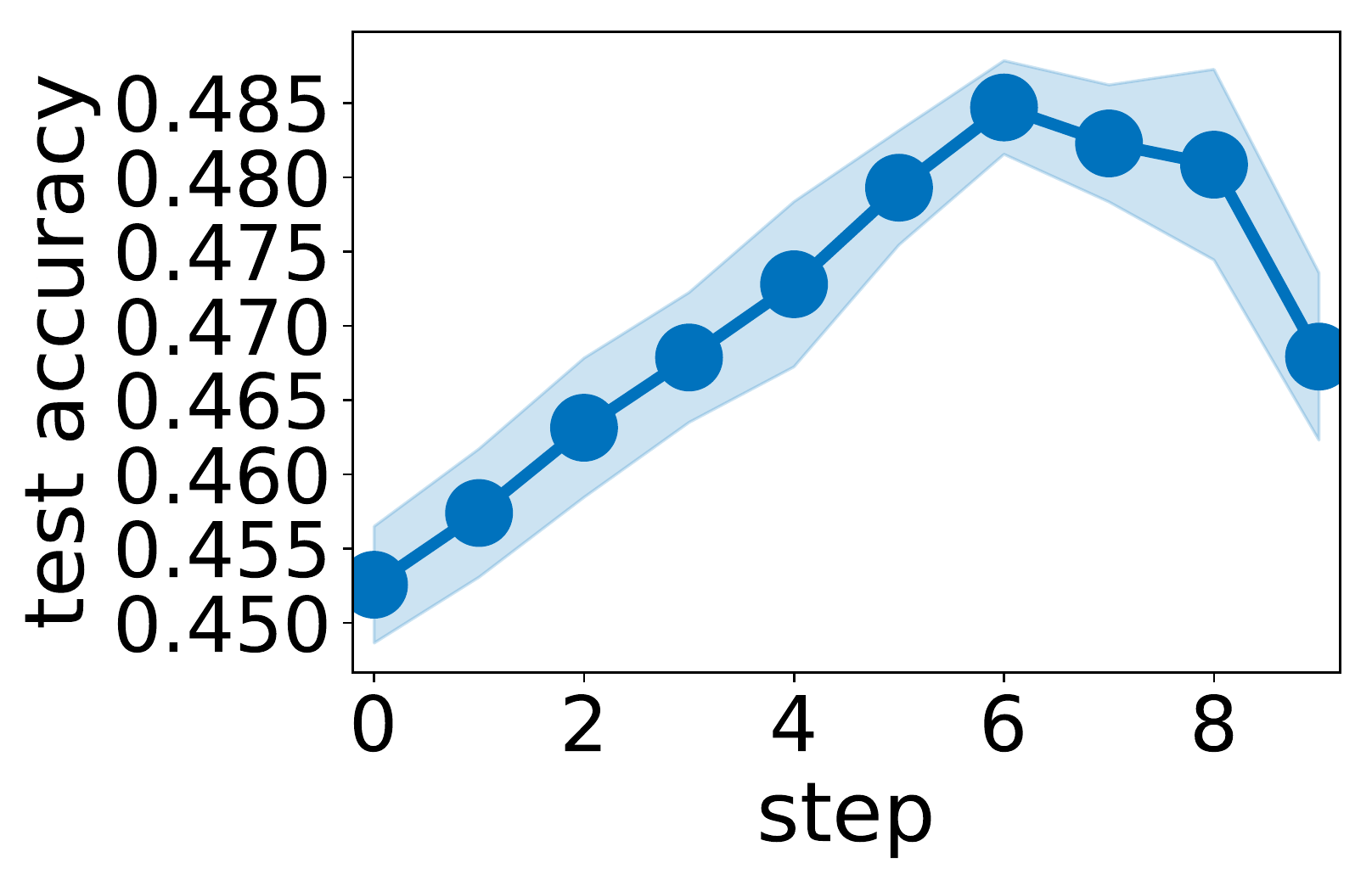} & 
            \includegraphics[width=0.3\textwidth]{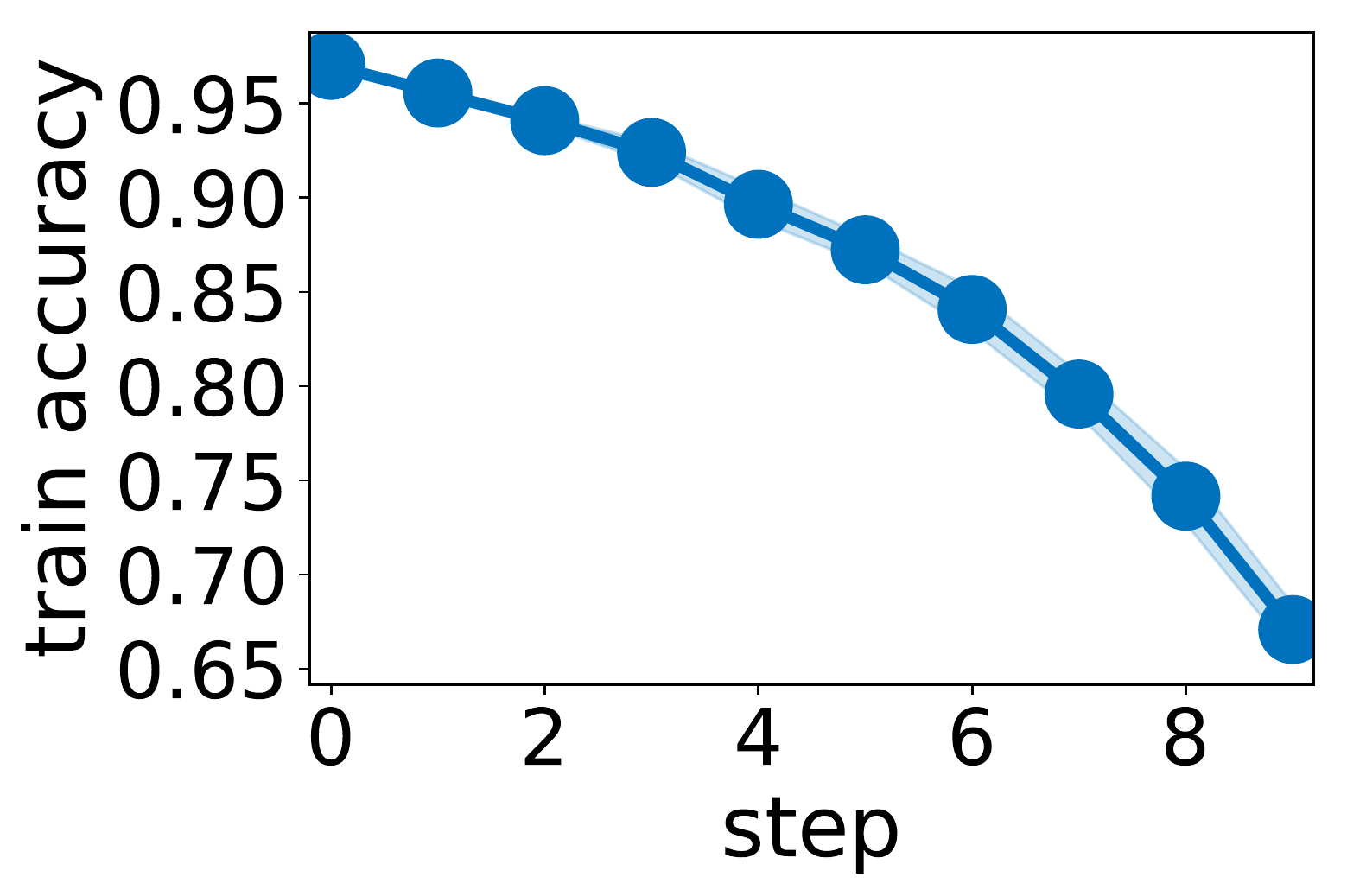}\\
            \includegraphics[width=0.3\textwidth]{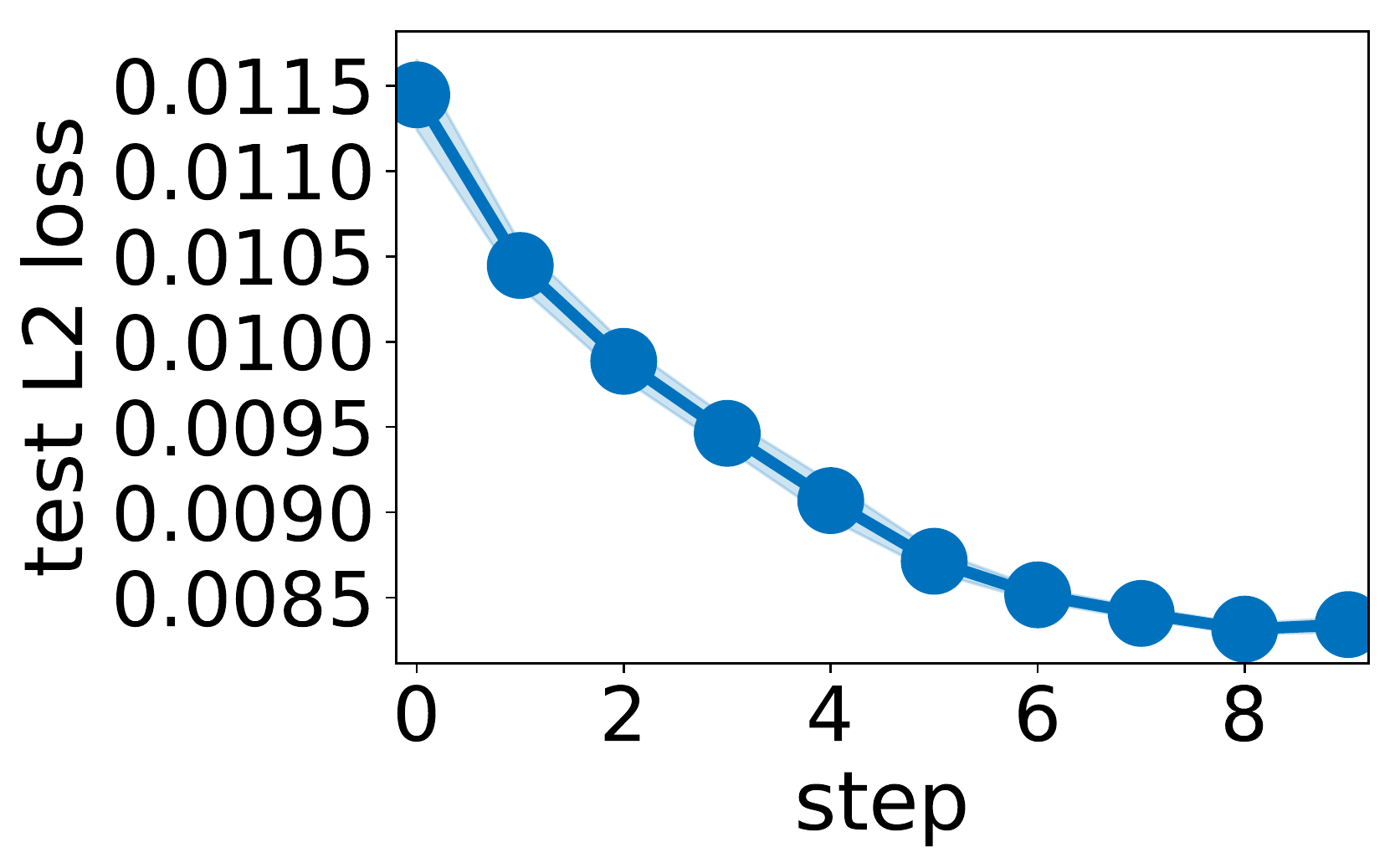}&
               \includegraphics[width=0.3\textwidth]{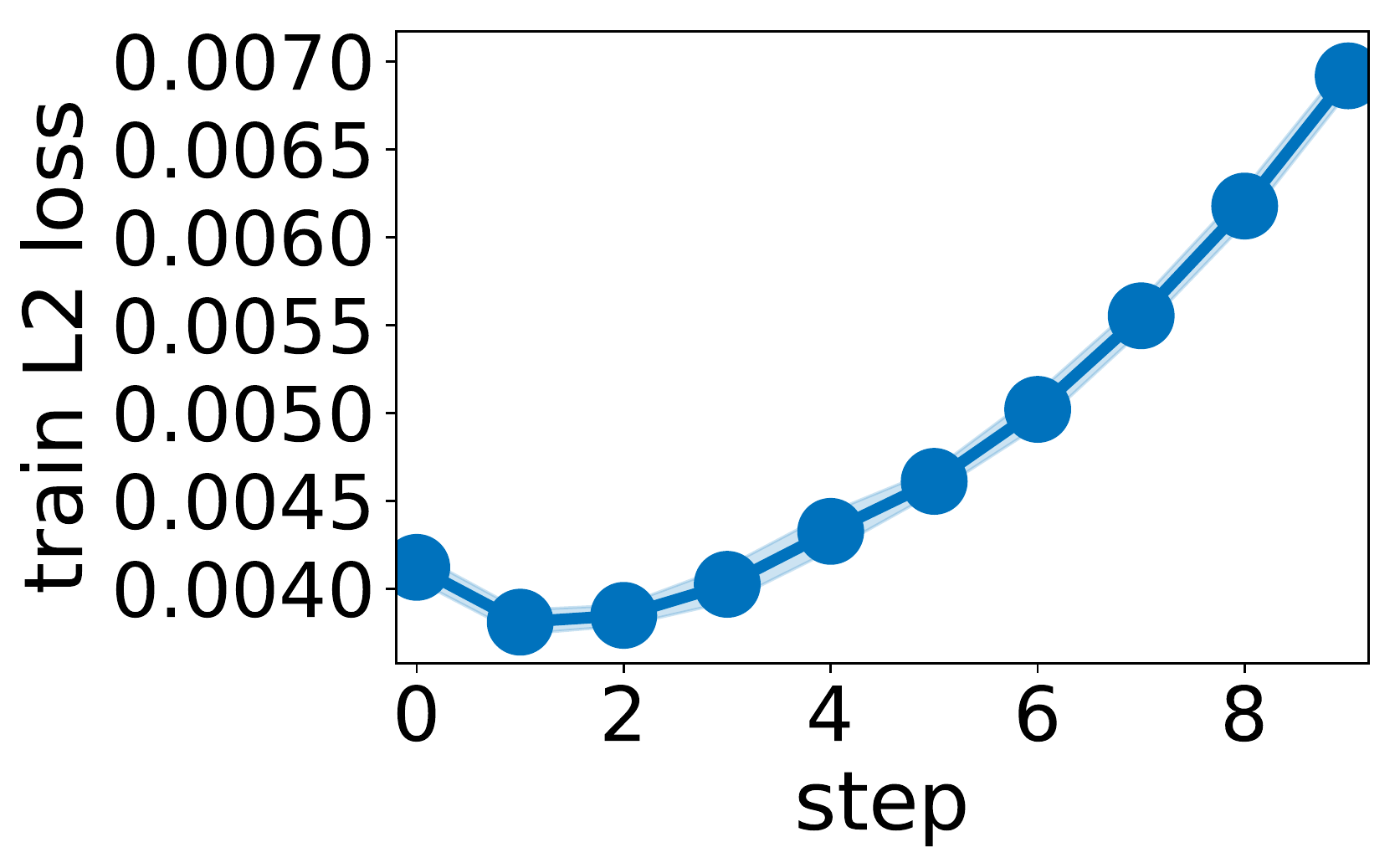}
    \end{tabular}
    \caption{Self-Distillation results with $\ell_2$ loss of neural network predictions for Resnet-50 and CIFAR-100 }
    \label{fig:app-L2-pred-resnet50-CIFAR-100}
\end{figure}

\subsection{Self-distillation on Hard Labels}
One might wonder how self-distillation would perform if we replace the neural network (soft) predictions with hard labels. In other words, the teacher's predictions are turned into one-hot-vector via \verb+argmax+ and they are treated like a dataset with augmented labels. Of course, since the model is already over-parameterized and trained close to interpolation regime only a small fraction of labels will change.
Figures~\ref{fig:app-ce-label-resnet50-CIFAR-10} and \ref{fig:app-ce-label-resnet50-CIFAR-100} show the results of self-distillation using cross-entropy loss on labels predicted by the teacher model. Surprisingly, self-distillation improves the performance here too. This observation may be related to learning under noisy dataset and calls for more future work on this interesting case.
\begin{figure}[ht]
    \centering
    \begin{tabular}{c c}
    \includegraphics[width=0.3\textwidth]{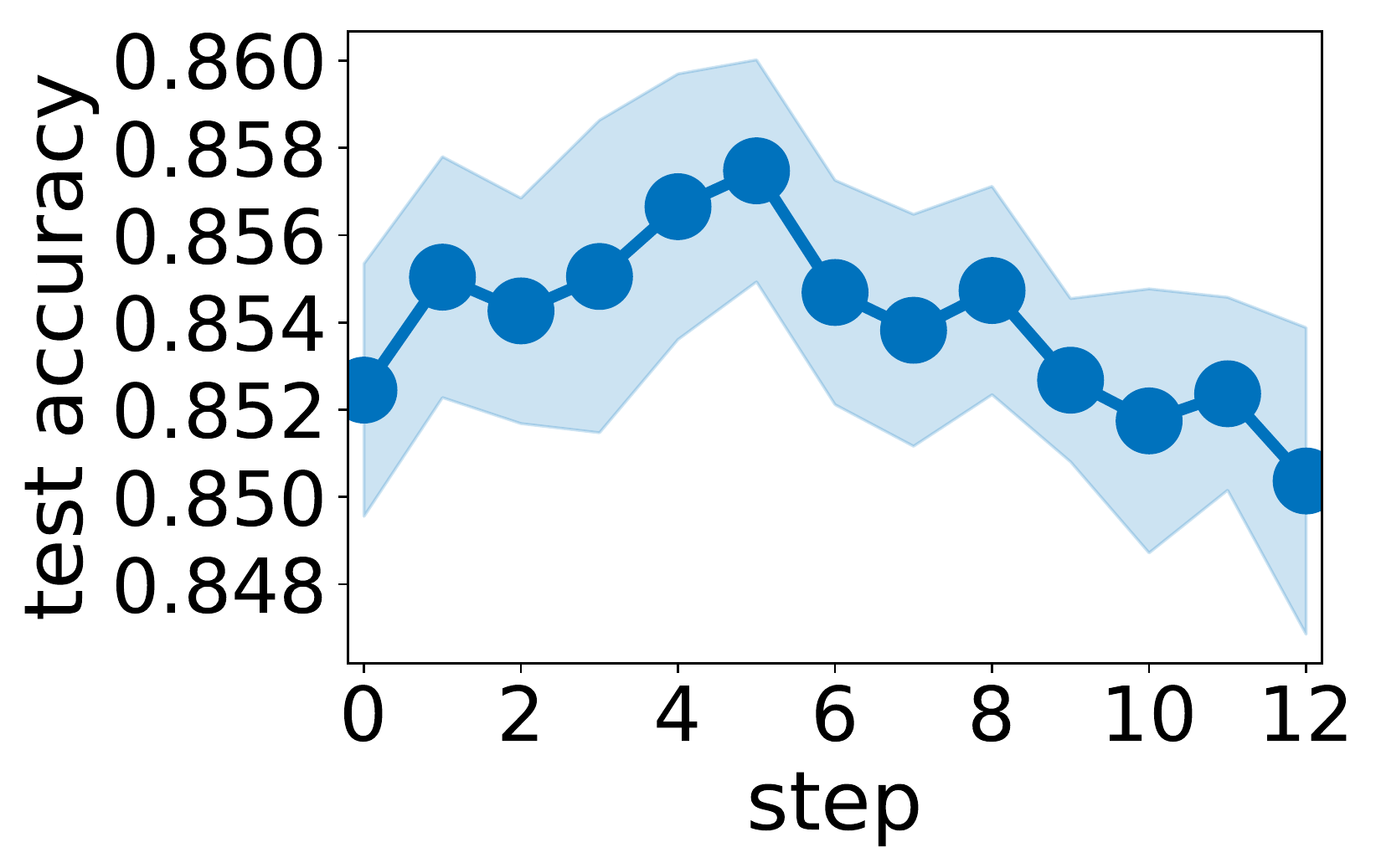} &
    \includegraphics[width=0.3\textwidth]{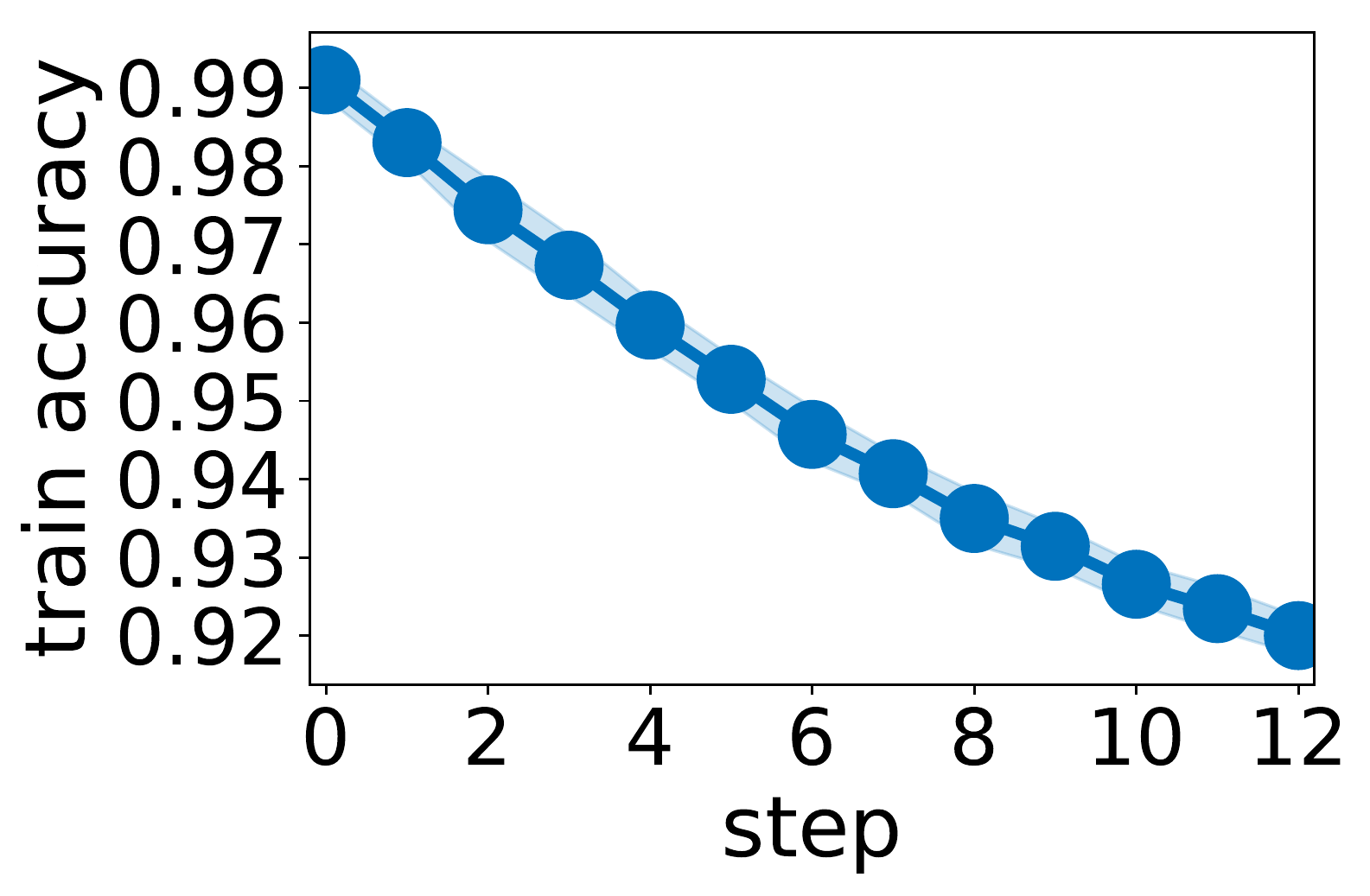} 
    \end{tabular}
    \caption{Self-Distillation results with cross-entropy loss on hard labels for Resnet-50 and CIFAR-10 }
    \label{fig:app-ce-label-resnet50-CIFAR-10}
\end{figure}

\begin{figure}[ht]
    \centering
    \begin{tabular}{c c}
    \includegraphics[width=0.3\textwidth]{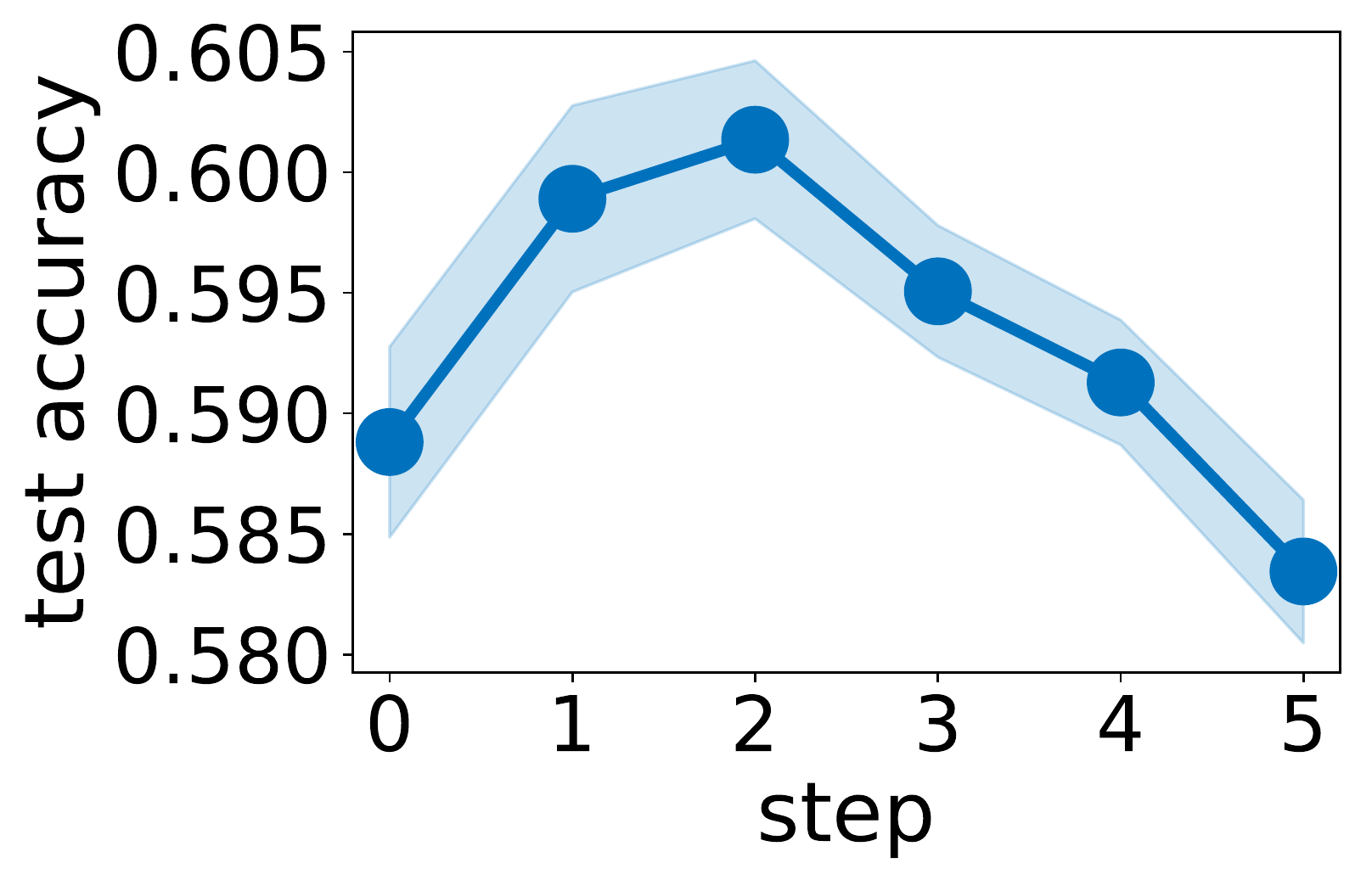} &
    \includegraphics[width=0.3\textwidth]{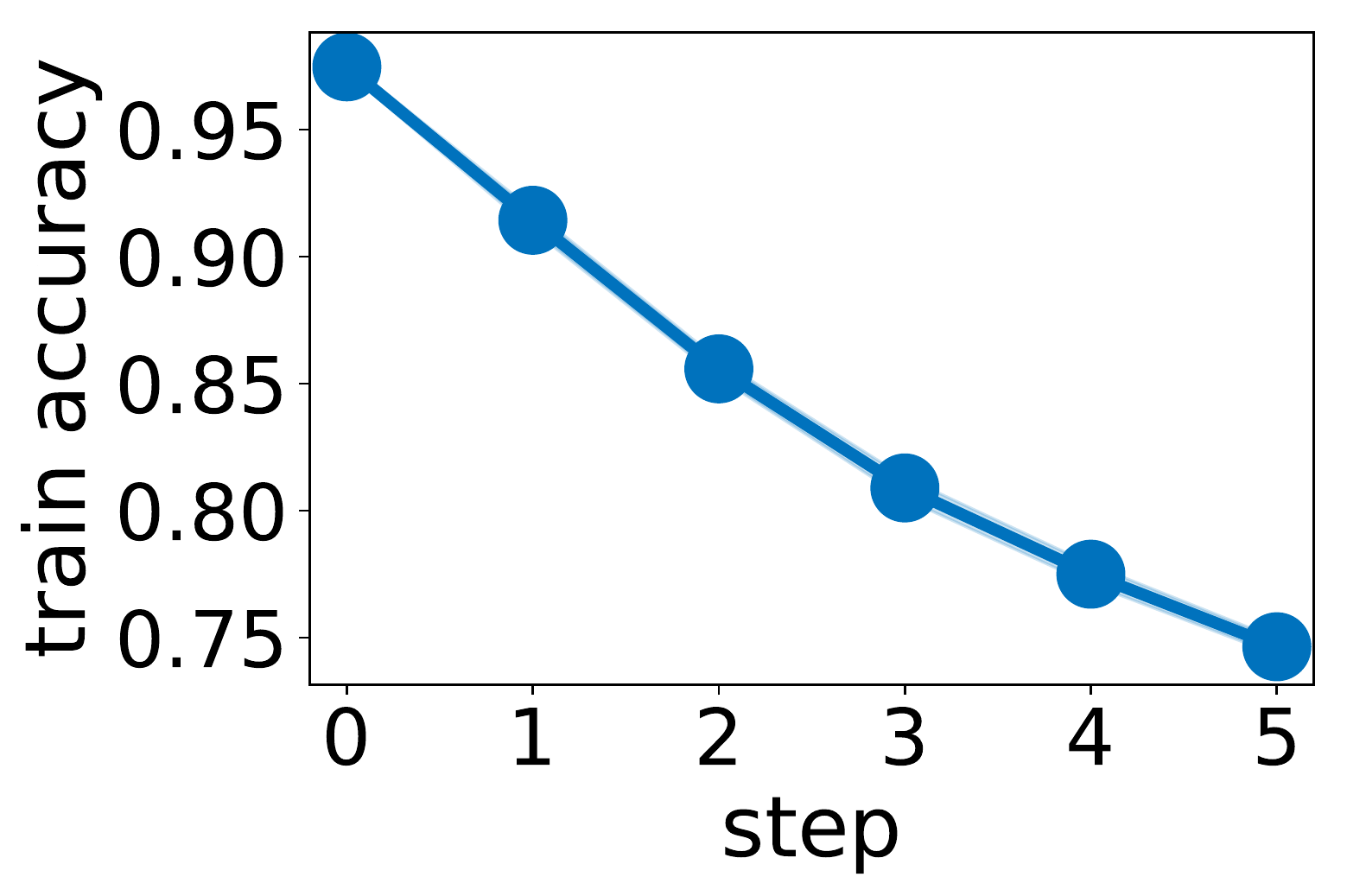} \\
    \end{tabular}
    \caption{Self-Distillation results with cross-entropy loss on hard labels for Resnet-50 and CIFAR-100 }
    \label{fig:app-ce-label-resnet50-CIFAR-100}
\end{figure}

\pagebreak
\section{Mathematica Code To Reproduce Illustrative Example}

\begin{lstlisting}[extendedchars=true,language=Mathematica]

x = (Table[i, {i, -5, 5}]/5 + 1)/2;
y = Sin[x*2*Pi] + 
  RandomVariate[NormalDistribution[0, 0.5], Length[x]]
ListPlot[y]

(* UNCOMMENT IF YOU WISH TO USE EXACT SAME RANDOM SAMPLES IN THE PAPER *)
(* y = {0.38476636465198066`, 
  1.2333967683416893`, 1.33232242218057`, 
  0.6920159488889518`, -0.29756145531871736`, -0.24189291901377769`, \
-0.7964485769175675`, -0.9616480167034174`, -0.49672509509916934`, \
-0.3469066003991437`, 0.5589512650600734`}; *)


(******** PLOT GREEN'S FUNCTION g0(X,T) FOR OPERATOR d^4/dx^4 ********)

g0 = 1/6*Max[{(T - X)^3, 0}] - 1/6*T*(1 - X)*(T^2 - 2*X + X^2);
ContourPlot[g0, {X, 0, 1}, {T, 0, 1}]
Plot3D[g0, {X, 0, 1}, {T, 0, 1}]

(***** COMPUTE g AND G *****)

G = Table[
   g0 /. X -> ((i/5 + 1)/2) /. T -> ((j/5 + 1)/2), {i, -5, 5}, {j, -5,
     5}];
g = Transpose[{Table[g0 /. T -> ((j/5 + 1)/2), {j, -5, 5}]}];

(***** PLOT GROUND-TRUTH FUNCTION (ORANGE) AND OVERFIT FUNCTION \
(BLUE) *****)
FNoReg = (Transpose[g].Inverse[
      G + 0.0000000001*IdentityMatrix[Length[x]]].Transpose[{y}])[[1, 
   1]];
pts = Table[{x[[i]], y[[i]]}, {i, 1, Length[x]}];
Show[{ListPlot[pts], Plot[{FNoReg, Sin[X*2*Pi]}, {X, 0, 1}]}]


(***** PARAMETERS *****)
MaxIter = 10;
eps = 0.045;


(***** SUBROUTINES *****)
Loss[G_, yin_, c_] := Module[
   {t = (G.Inverse[c*IdentityMatrix[Length[yin]] + G] - 
        IdentityMatrix[Length[x]]).yin},
   Total[Flatten[t]^2]/Length[yin]
   ];

FindRootsC[f_, c_] := Module[
   {Sol = Quiet[Solve[f == 0, c]], Sel},
   Sel = Select[
     c /. Sol, (Abs[Im[#]] < 0.00000001) && # > 0.00000001 &]
   ];


(***** MAIN *****)

(* Initialization *)
y0 = Transpose[{y}];
ycur = y0;
B = IdentityMatrix[Length[x]];
FunctionSequence = {};
ASequence = {};
BSequence = {};

(* Self-Distllation Loop *)
For[i = 1;, i < MaxIter, i++,
 Print["Iteration ", i];
 Print["Norm[y]=", Norm[ycur]];
 L = Loss[G, ycur, c];
 RootsC = FindRootsC[L - eps, c];
 Switch [Length[RootsC], 0, (Print["No Root"]; Break[];), 1, 
  Print["Found Unique Root c=", RootsC[[1]] ];];
 (* Now that root is unique *)
 RootC = RootsC[[1]];
 Print["Achieved Loss Value ", Loss[G, ycur, RootC]];
 U = G.Inverse[G + RootC*IdentityMatrix[Length[ycur]]];
 A = DiagonalMatrix[Eigenvalues[U]]; 
 f = (Transpose[g].Inverse[
      G + RootC*IdentityMatrix[Length[ycur]]].ycur)[[1, 1]];
 B = B.A;
 ycur = U.ycur;
 
 FunctionSequence = Append[FunctionSequence, f];
 ASequence = Append[ASequence, Diagonal[A]];
 BSequence = Append[BSequence, Diagonal[B]];
 ]

If[i == MaxIter, Print["Max Iterations Reached!"]]

Plot[FunctionSequence, {X, 0, 1}]
BarChart[ASequence, ChartStyle -> "DarkRainbow", AspectRatio -> 0.2, 
 ImageSize -> Full]
BarChart[BSequence, ChartStyle -> "DarkRainbow", AspectRatio -> 0.2, 
 ImageSize -> Full]


\end{lstlisting}
\pagebreak
\section{Python Implementation}
Implementing self-distillation is quite straight forward provided with merely a customized loss that replaces the ground-truth labels with teacher predictions. Here, we provide a Tensorflow implementation of the self-distillation loss function:

\definecolor{Code}{rgb}{0,0,0}
\definecolor{Decorators}{rgb}{0.5,0.5,0.5}
\definecolor{Numbers}{rgb}{0.5,0,0}
\definecolor{MatchingBrackets}{rgb}{0.25,0.5,0.5}
\definecolor{Keywords}{rgb}{0,0,1}
\definecolor{self}{rgb}{0,0,0}
\definecolor{Strings}{rgb}{0,0.63,0}
\definecolor{Comments}{rgb}{0,0.63,1}
\definecolor{Backquotes}{rgb}{0,0,0}
\definecolor{Classname}{rgb}{0,0,0}
\definecolor{FunctionName}{rgb}{0,0,0}
\definecolor{Operators}{rgb}{0,0,0}
\definecolor{Background}{rgb}{0.98,0.98,0.98}
\lstdefinelanguage{Python}{
numbers=left,
numberstyle=\footnotesize,
numbersep=1em,
xleftmargin=3em,
framextopmargin=2em,
framexbottommargin=3em,
showspaces=false,
showtabs=false,
showstringspaces=false,
frame=l,
tabsize=4,
basicstyle=\ttfamily\small\setstretch{1},
backgroundcolor=\color{Background},
commentstyle=\color{Comments}\slshape,
stringstyle=\color{Strings},
morecomment=[s][\color{Strings}]{"""}{"""},
morecomment=[s][\color{Strings}]{'''}{'''},
morekeywords={import,from,class,def,for,while,if,is,in,elif,else,not,and,or,print,break,continue,return,True,False,None,access,as,,del,except,exec,finally,global,import,lambda,pass,print,raise,try,assert},
keywordstyle={\color{Keywords}\bfseries},
morekeywords={[2]@invariant,pylab,numpy,np,scipy},
keywordstyle={[2]\color{Decorators}\slshape},
emph={self},
emphstyle={\color{self}\slshape},
}

\linespread{1.3}

\begin{lstlisting}[language=Python]
def self_distillation_loss(labels, logits, model, reg_coef, teacher=None, data=None):
  if teacher is None:
    main_loss = tf.reduce_mean(tf.squared_difference(labels, 
                                                     tf.nn.softmax(logits)))
  else:
    main_loss = tf.reduce_mean(tf.squared_difference(tf.nn.softmax(teacher(data)), 
                                                     tf.nn.softmax(logits)))
  reg_loss = reg_coef*tf.add_n([tf.nn.l2_loss(w) for w in model.trainable_weights])
  total_loss = main_loss + reg_loss
  return total_loss
\end{lstlisting}
\bigskip

The following snippet also demonstrates how one can use the above loss function to train a neural network using self-distillation. 

\begin{lstlisting}[language=Python]
def self_distillation_train(model, train_dataset, optimizer, reg_coef=1e-4, 
                            epochs=30, teacher=None):
  for epoch in range(epochs):
    for iter, (x_batch_train, y_batch_train) in enumerate(train_dataset):
      with tf.GradientTape() as tape:
        logits = model(x_batch_train, training=True)
        loss_value = self_distillation_loss(y_batch_train, logits, model, 
                                            reg_coef, teacher, x_batch_train)
      grads = tape.gradient(loss_value, model.trainable_weights)
      optimizer.apply_gradients(zip(grads, model.trainable_weights))
  return model

teacher = None
for step in range(distillation_steps):
  model = get_resnet_model()
  optimizer = keras.optimizers.Adam(learning_rate=learning_rate)
  model = self_distillation_train(model, train_dataset, optimizer, 
                                  reg_coef, epochs, teacher)
  teacher = model
\end{lstlisting}
\bigskip


\end{document}